%% file: main.tex
\documentclass[11pt]{article}
\usepackage[utf8]{inputenc}

\usepackage{mystyle}

\newcommand{\KUCB}{\textrm{KOVI}}
\newcommand{\NUCB}{\textrm{NOVI}}
\newcommand{\act}{\mathtt{act}}

\begin{document}

\title{\huge On Function Approximation in Reinforcement Learning: 
Optimism in the Face of Large State Spaces}

\author{Zhuoran Yang\thanks{Princeton University. Email:\texttt{\{zy6,chij,mendiw\}@princeton.edu}}\qquad\qquad\qquad Chi Jin$^*$\qquad \qquad\qquad Zhaoran Wang\thanks{Northwestern  University. Email: \texttt{zhaoranwang@gmail.com}}\\
 Mengdi Wang$^*$ \qquad\qquad\qquad\quad Michael I. Jordan\thanks{University of California, Berkeley. Email: \texttt{jordan@cs.berkeley.edu}}}

\maketitle

\begin{abstract}
The classical theory of reinforcement learning (RL) has focused on tabular
and linear representations of value functions.  Further progress hinges on
combining RL with modern function approximators such as kernel functions and 
deep neural networks, and indeed there have been many empirical successes 
that have exploited such combinations in large-scale applications.  There
are profound challenges, however, in developing a theory to support this
enterprise, most notably the need to take into consideration the exploration-exploitation 
tradeoff at the core of RL in conjunction with the computational and statistical
tradeoffs that arise in modern function-approximation-based learning systems.
We approach these challenges by studying an optimistic modification of the 
least-squares value iteration algorithm, in the context of the action-value function  
 represented by a kernel function or an overparameterized neural network. 
We establish both polynomial runtime complexity and polynomial sample complexity 
for this algorithm, without additional assumptions on the data-generating model.  
In particular, we prove that the algorithm incurs  an $\tilde{\mathcal{O}}(\delta_{\cF}  
H^2 \sqrt{T})$ regret, where $\delta_{\cF}$ characterizes the intrinsic complexity of the function class $\cF$, $H$ is the length of each episode, and  $T$ is the  total number of episodes.  Our regret bounds are independent of the number of states, 
a result which exhibits clearly the benefit of function approximation in RL.
 

\end{abstract}

\input{intro}

\input{setting}

\input{algorithm}

\input{results}

\input{proof_thm}

\input{conclu}

\section*{Acknowledgements}

We would like to thank the Simons Institute for the Theory of Computing 
in Berkeley, where this project was initiated.  Zhuoran Yang would like to 
thank Jianqing Fan, Csaba Szepsv\'ari, Tuo Zhao, Simon Shaolei Du, Ruosong 
Wang, and Yiping Lu for valuable discussions.  Mengdi Wang gratefully 
acknowledges funding from the U.S. National Science Foundation (NSF) 
grant CMMI1653435, Air Force Office of Scientific Research (AFOSR) grant 
FA9550-19-1-020, and C3.ai DTI.  Michael Jordan gratefully acknowledges 
funding from the Mathematical Data Science program of the Office of Naval 
Research under grant number N00014-18-1-2764.
 
\bibliographystyle{ims}
\bibliography{rl_ref}

\appendix{}

\input{proof_lemmas}

\input{proof_aux}

\end{document}

%% file: intro.tex

\section{Introduction}

Reinforcement learning (RL) algorithms combined with modern function approximators 
such as kernel functions and deep neural networks have produced empirical successes 
in a variety of application problems \citep[e.g.,][]{duan2016benchmarking, silver2016mastering, silver2017mastering, wang2018deep,vinyals2019grandmaster}.
However, theory has lagged, and when these powerful function approximators are employed, 
there is little theoretical guidance regarding the design of RL algorithms that are
efficient computationally or statistically, or regarding whether they even converge.  
In particular, function approximation blends statistical estimation issues with dynamical 
optimization issues, resulting in the need to balance the bias-variance tradeoffs that arise 
in statistical estimation with the exploration-exploitation tradeoffs that are inherent 
in RL.  Accordingly, full theoretical treatments are mostly restricted to the tabular   
setting, where both the state and action spaces are discrete and the value function 
can be  represented as a table~\citep[see, e.g.,][]{jaksch2010near,osband2014generalization,
azar2017minimax, jin2018q, osband2018randomized, russo2019worst}, and there is a disconnect between theory and the most compelling applications.

Provably efficient exploration in the function approximation setting has been addressed
only recently, with most of the existing work considering (generalized) linear 
models~\citep{yang2019reinforcement, yang2019sample, jin2019provably, cai2019provably,
zanette2019frequentist, wang2019optimism}.  These algorithms and their analyses stem from  
classical upper confidence bound (UCB) or Thompson sampling methods for linear contextual 
bandits~\citep{bubeck2012regret, lattimore2018bandit}  and it seems difficult to extend 
them beyond the linear setting.  Unfortunately, the linear assumption is rather rigid 
and rarely satisfied in practice; moreover, when such a model is misspecified, sublinear
regret guarantees can vanish.  There has been some recent work that has presented 
sample-efficient  algorithms with general function approximation. However, these 
methods are either  computationally intractable~\citep{krishnamurthy2016pac, 
jiang2017contextual, dann2018oracle,dong2019sqrt} or hinge on strong assumptions 
on the transition model \citep{wen2017efficient, du2019provably}.  Thus, the 
following question remains open:
 
 \begin{tcolorbox}[colback=blue!5!white,colframe=white]
 \begin{center}
Can we design  RL algorithms that incorporate powerful nonlinear function approximators 
such as neural networks or kernel functions and provably achieve both computational and 
statistical efficiency?
 \end{center} 
\end{tcolorbox}
  
In this work, we provide an affirmative answer to this question.  Focusing on the 
setting of an episodic Markov decision process (MDP) where the value function
is represented by either a kernel function or an overparameterized  neural network,
we propose an RL algorithm with polynomial runtime complexity and  sample complexity, 
without imposing any additional assumptions on the data-generating model.  Our 
algorithm is relatively simple---it is an optimistic modification of the   
least-squares value iteration algorithm (LSVI) \citep{bradtke1996linear}---a 
classical batch RL algorithm---to which we add a UCB bonus term to each iterate.
Specifically, when using a kernel function, each LSVI update becomes a kernel 
ridge regression, and the bonus term is derived from that proposed for kernelized 
contextual bandits~\citep{srinivas2009gaussian, valko2013finite, chowdhury2017kernelized}. 
For the   neural network setting, motivated by the NeuralUCB algorithm  for contextual 
bandits~\citep{zhou2019neural}, we construct a  UCB bonus from the tangent features  
of the neural network and we perform the LSVI updates via projected gradient descent.
In both of these settings, the usage of the UCB bonus ensures that the value functions 
constructed by the algorithm are always optimistic in the sense that they serve as  
uniform upper bounds of the optimal value function.  Furthermore, for both the kernel 
and neural settings, we prove that  the proposed algorithm incurs an 
$\tilde{\mathcal{O}}(\delta_{\cF} H^2 \sqrt{ T})$ regret, where $H$ is the length 
of each episode,  $T$ is the total number of episodes, and  $\delta_{\cF}$ quantifies  
the intrinsic complexity of the function class $\cF$. Specifically, as we will 
show in \S\ref{sec:results}, $\delta_{\cF}$ is determined  by the interplay between 
the $\ell_\infty$-covering number of the function class used to represent the value
function and the effective dimension of function class $\cF$. (See Table \ref{table:summary} 
for a summary.) 

A key feature of our regret bounds  is that they depend on the complexity of the 
state space only through $\delta_{\cF}$ and thus allow the number of states to be 
very large or even divergent.  This clearly exhibits the  benefit of  function 
approximation by tying it directly to sample efficiency.  To the best of our 
knowledge, this is the first provably efficient framework for reinforcement learning 
with kernel and neural network function approximations.

 \begin{table}[htpb]

      \centering
  \begin{tabular}{c c} 
   \hline
   {\color{blue} function class $\cF$} & {\color{blue} regret bound}  \\
      general RKHS $\cH$ & $H^2\cdot  \sqrt{ d_{\textrm{eff}}  \cdot [d_{\textrm{eff}}  + \log N_{\infty} (\epsilon^* )} ] \cdot  \sqrt{T} $ \\
      $\gamma$-finite spectrum & $H^2 \cdot \sqrt{\gamma^3 T} \cdot  \log(\gamma TH) $ \\
      $\gamma$-exponential decay &  $H^2 \cdot \sqrt{ (\log T)^{3/\gamma} \cdot T} \cdot \log(TH)$ \\
     $\gamma$-polynomial decay  & $H^2 \cdot T^{\kappa^* + \xi^* + 1/2} \cdot [ \log(TH) ] ^{3/2}$  \\ 
     overparameterized neural network & $H^2\cdot  \sqrt{ d_{\textrm{eff}}   \cdot [d_{\textrm{eff}} + \log N_{\infty} (\epsilon^* )} ] \cdot  \sqrt{T}  + \textrm{poly}(T, H) \cdot m^{-1/12} $ \\
   \hline
  \end{tabular}
 
   \caption{Summary of the main results. Here $H$ is the length of each episode, $T$ is the number of episodes in total, and $2m$ is the number of neurons of the overparameterized networks in  the neural setting. 
   For an  RKHS $\cH$ in general, $d_{\textrm{eff}}$ denotes the effective dimension of $\cH$
    and $  N_{\infty} (\epsilon^*)$ is the $\ell_\infty$-covering number of the value function class, 
    where $\epsilon ^* = H / T$.
Note that to obtain concrete bounds, we apply the general result to RKHS's with various 
eigenvalue decay conditions. Here $\gamma$ is a positive integer in the case of 
a $\gamma$-finite spectrum and is a positive number in  the subsequent two cases. 
In addition, $\kappa^*$ and $\xi^*$, which are defined in \eqref{eq:define_poly_kappa}, 
are constants that depend on $d$, the input dimension, and the parameter $\gamma$. 
Finally, in the last case we present the regret bound for the neural setting in 
general, where $d_{\textrm{eff}}$ is the effective dimension of the neural tangent 
kernel (NTK) induced by the overparameterized neural network with $2m$ neurons and 
$\textrm{poly}(T,H)$ is a polynomial in $T$ and $H$. Such a general regret bound
can be expressed concretely as a function of the spectrum of the NTK.}
  \label{table:summary}
  \end{table}

 \vskip5pt
\noindent{\bf Related Work.}
There is a vast literature on establishing provably efficient RL methods in the
absence of a generative model or an explorative behavioral policy.  Much of this 
literature has focused on the tabular setting; see \citet{jaksch2010near, 
osband2014generalization, azar2017minimax, dann2017unifying, strehl2006pac, 
jin2018q, russo2019worst} and the references therein.  In particular, 
\cite{azar2017minimax,jin2018q} prove that an RL algorithm necessarily incurs 
a $\Omega(\sqrt{SAT})$ regret under the tabular setting, where $S$ and $A$ 
are the cardinalities of the state and action spaces, respectively. Thus, 
algorithms designed for the tabular setting cannot be directly applied to 
the function approximation setting, where the number of effective states is 
large.  A recent literature has accordingly focused on the function approximation 
setting, specifically the (generalized) linear setting  where the value function 
(or the  transition model) can be represented using a linear transform of a known 
feature mapping~\citep{yang2019sample, yang2019reinforcement, jin2019provably, 
cai2019provably,zanette2019frequentist, wang2019optimism,ayoub2020model, 
zhou2020provably, kakade2020information}.  Among these papers, our work is 
most closely related to \cite{jin2019provably}.  In particular, in our kernel 
setting when the kernel function has a finite rank, both our LSVI algorithm 
and the corresponding regret bound reduce to those established in \cite{jin2019provably}.
However, the sample complexity and regret bounds in \cite{jin2019provably}
diverge when the dimension of the feature mapping goes to infinity and thus 
  cannot be directly applied to the kernel setting. 

Also closely related to our work is  \citet{wang2020provably}, which 
studies a similar optimistic LSVI algorithm for general function approximation. 
This work focuses on value function classes with bounded eluder 
dimension~\citep{russo2013eluder, osband2014model}.  It is unclear 
whether whether this formulation can be extended to the kernel or neural 
network settings.  \citet{yang2019reinforcement} studies a kernelized 
MDP model where the transition model can be directly estimated.  Under 
a slightly more general model, \cite{ayoub2020model} recently propose  
an optimistic model-based algorithm via value-targeted regression, where 
the model class is the set of functions with bounded eluder dimension. 
In other recent work, \cite{kakade2020information} studies a nonlinear 
control formulation in which the transition dynamics belongs to a known 
RKHS and can be directly estimated from the data.  Our work differs from
this work in that we impose an explicit assumption on the transition model 
and our proposed algorithm is model-free. 

Other authors who have presented regret bounds and sample complexities beyond   
the linear setting include \citet{krishnamurthy2016pac, jiang2017contextual, 
dann2018oracle,dong2019sqrt}.  These algorithms generally involve either
high computational costs or require possibly restrictive assumptions on the 
transition model~\citep{wen2013efficient, wen2017efficient, du2019provably}.

Our work is also related to the literature on contextual bandits with either
kernel function classes~\citep{srinivas2009gaussian, krause2011contextual,
srinivas2012information, valko2013finite,chowdhury2017kernelized, durand2018streaming} 
or neural network function classes \citep{zhou2019neural}.  Our construction 
of a bonus function for the RL setting has been adopted from this previous work.
However, while contextual bandits can be viewed formally as special cases of 
our episodic MDP formulation with the episode length equal to one, the   
temporal dependence in the MDP setting raises significant challenges.
In particular, the covering number $N_{\infty} (\epsilon^*)$ in 
Table \ref{table:summary} arises as a consequence of the fundamental 
challenge of performing temporally extended exploration in RL. 
 
Finally, our analysis of the optimistic LSVI algorithm is related to recent   
work on optimization and generalization in overparameterized neural networks 
within the framework of the neural tangent kernel \citep{jacot2018neural}. 
See also \citet{daniely2017sgd, jacot2018neural, wu2018sgd,  du2018gradient1, 
du2018gradient, allen2018convergence, allen2018learning, zou2018stochastic, 
chizat2018note, li2018learning, arora2019fine, cao2019bounds,  cao2019generalization,  
lee2019wide}.  This literature focuses principally on supervised learning, however;
in the RL setting we need an additional bonus term in the least-squares problem 
and thus require a novel analysis.

%% file: setting.tex
 
\section{Background}

In this section, we provide essential background on reinforcement learning, 
reproducing kernel Hilbert space (RKHS), and  overparameterized neural networks. 

\subsection{Episodic Markov Decision Processes} 

We focus on episodic MDPs, denoted $\mathrm{MDP}(\cS, \cA, H, \PP, r)$,
where $\cS$ and $\cA$ are the state and action spaces, respectively, the
integer $H>0$ is the length of each episode, $\PP = \{ \PP_h \}_{h \in [H] }$ 
and $r = \{ r_h \}_{h\in [H]} $  are the  Markov transition kernel  and the reward 
functions,  respectively, where we let $[n]$ denote the set $\{ 1, \ldots, n\}$ 
for integers $n\geq 1$.  We assume that $\cS$ is a   measurable space   
of possibly infinite cardinality while $\cA $ is a finite set.  Finally,
for each $h \in [ H]$, $\PP_h ( \cdot \given x, a) $ denotes the probability 
transition kernel when action $a$ is taken at  state $x \in \cS$ in timestep $h\in [H]$, 
and $r_h \colon \cS \times \cA \to [0,1]$ is the reward function at step $h$ 
which is assumed to be deterministic for simplicity. 

A \emph{policy} $\pi$ of an agent is a set of $H$ functions $\pi  =\{ \pi_h \}_{h \in [H]}$  
such that each $\pi_h(\cdot \given x )$ is a probability distribution over $\cA$.  
Here   $\pi_h (a \given x)$ is the probability of the agent taking action $a$ at 
state $x$ at the $h$-th step in the episode.  

The agent interacts with the environment as follows.  For any $t\geq 1$, at the beginning 
of the $t$-th episode, the agent determines a policy $\pi^t = \{\pi^t_h\}_{h \in [H]}$ 
while an initial state $x_1^t$ is   picked arbitrarily by the environment. 
Then, at each step $h \in [H]$, the agent observes the state $x_h^t \in \cS$, picks 
an action $a_h^t \sim \pi_h^t(\cdot \given x_h^t)$, and receives a reward $r_h(x_h^t, a_h^t)$. 
The environment then transitions into a new state  $x_{h+1}^t$  that is drawn 
from the probability measure $\PP_h(\cdot \given x_h^t, a_h^t)$. The episode terminates 
when the $ H $-th step is reached and $r_H (x_H^t, a_H^t)$ is thus the final reward 
that the agent receives.

%

The performance of the agent is captured by the \emph{value function}. 
For any policy $\pi$, and  $h\in [H]$, we define the value function 
$V_h^{\pi}\colon \cS \to \mathbb{R} $ as 
\begin{equation*}
 V_h^\pi(x) =  \EE_{\pi} \left[\sum_{h' = h}^H r_{h'}(x_{h'},  a_{h'} )  
\bigggiven  x_h = x\right], \qquad \forall x\in \cS, h \in [H],
\end{equation*}
where $\EE_{\pi}[\cdot]$ denotes the expectation with respect to the 
randomness of the trajectory $\{(x_h,a_h)\}_{h=1}^H$ obtained by following 
the  policy $\pi $.  We also define the action-value function  
$Q_h^\pi:\cS \times \cA \to \mathbb{R}$ as follows:
  \begin{equation*}
  Q^{\pi}_h(x,a) =\EE_{\pi} \biggl[ \sum_{h'=h}^H r_{h'} (x_{h'},a_{h'} ) 
\,\Big|\, x_h=x,\,a_h=a  \biggr].
\end{equation*}
Moreover, let $\pi^\star$ denote the optimal policy which by definition
yields the optimal value function, $V^{\star}_{h}(x) = \sup_{\pi} V_h^\pi(x)$,
for all $x\in \cS$ and $h\in [H]$.
To simplify the notation,  we write  
$$
 (\PP_h V  ) (x, a) \defeq \EE_{x' \sim \PP_h(\cdot \given x, a)} [V (x')],
$$
for any measurable function $V \colon \cS \rightarrow [ 0, H] $.
Using this notation, the   Bellman equation associated with a policy $\pi$ becomes 
\begin{align} \label{eq:bellman} 
    \sind{Q}{\pi}{h}(x, a) = (r_h + \PP_h \sind{V}{\pi}{h+1})(x, a) , \qquad   V_{h}^\pi (x) = \la Q_h^{\pi} (x, \cdot ), \pi_{h}(\cdot \given x)  \ra_{\cA} ,     
    \qquad \sind{V}{\pi}{H+1}(x) = 0.   
\end{align}
Here we let  $\la\cdot,\cdot\ra_{\cA}$ denote  the inner product over $\cA$. 
Similarly, the Bellman optimality equation is given by 
\begin{align}\label{eq:opt_bellman}
  \sind{Q}{\star}{h}(x, a) = (r_h + \PP_h \sind{V}{\star}{h+1})(x, a) , \qquad 
 \sind{V}{\star}{h}(x) = \max_{a\in\cA}\sind{Q}{\star}{h}(x, a), \qquad \sind{V}{\star}{H+1}(x) = 0. 
\end{align}
Thus,  the optimal policy $\pi^\star$ is the greedy policy with respect to      $\{ Q^\star _h \}_{h \in [H]}$. 
Moreover, we define the 
Bellman optimality operator $\TT_h^\star$ by letting 
$$( \TT_h^{\star} Q ) (x,a) = r(x,a) + (\PP_h V  )(x,a) 
\qquad \textrm{for all}~~Q \colon \cS\times \cA \rightarrow \RR, 
$$   where $V  (x) = \max_{a \in \cA} Q(x,a)$. 
By definition, the Bellman equation in \eqref{eq:opt_bellman} is equivalent to  $Q_h ^\star = \TT_h^\star Q_{h+1} ^\star$,  $\forall h \in [H]$.  
 The goal of the agent is to 
 learn the optimal policy $\pi^\star$. For any policy  $\pi$, the difference between $V_1^\pi$ and $V_1^\star$ quantifies the sub-optimality of $\pi$. 
 Thus, for a fixed integer $T > 0$,  after playing for $T$ episodes, 
 the total (expected) regret \citep{bubeck2012regret} of the agent is defined as 
\begin{equation*}
\mathrm{Regret}(T) = \sum_{t=1}^T \bigl [\sind{V}{\star}{1} (x^t_1) - \sind{V}{\pi^t}{1} (x^t_1)\bigr],
\end{equation*}
where $\pi^t$ is the policy executed in the $t$-th episode and   $x_1^t$ is  the initial state.

\subsection{Reproducing Kernel Hilbert Space} \label{sec:rkhs}

We will be considering the use of reproducing kernel Hilbert space (RKHS) 
as function spaces to represent optimal value functions $Q_h^\star$.
To this end, hereafter, to simplify the notation, we let $z = (x,a)$ 
denote a state-action pair and denote $\cZ = \cS \times \cA$.  Without 
loss of generality, we regard $\cZ$ as a compact subset of $\RR^d$, 
where the dimension $d$ is assumed fixed.  This can be achieved if there 
exists an embedding $\psi_{\textrm{embed}} \colon \cZ \rightarrow \RR^d$ 
that serves as a preprocessing of the input $(x,a)$.  Let $\cH$ be an RKHS 
defined on $\cZ$ with kernel function, $K \colon \cZ \times \cZ \rightarrow \RR$, 
which contains a family of functions defined on $\cZ$.   
Let $\la \cdot,\cdot \ra_{\mathcal{H}} \colon \cH \times \cH \rightarrow  \RR $ 
and $\| \cdot \|_{\cH} \colon \cH \rightarrow \RR$  denote the inner product and 
RKHS norm on $\cH$, respectively.  Since $\cH$ is an RKHS, there exists 
a \emph{feature mapping}, $\phi \colon \cZ \rightarrow \cH$, such that 
$f(z)=\la f(\cdot),\phi(z) \ra_{\mathcal{H}}$
for all $f\in \cH$ and all $z \in \cZ$.
Moreover,  for any $x, y \in \cZ$, we have 
$K(x,y) = \langle \phi(x) , \phi(y) \rangle_{\cH}$. 
 In this work, we assume that the kernel function $K$ is uniformly bounded 
in the sense that $\sup_{z \in \cZ} K(z,z) < \infty$. 
 Without loss of generality, we assume that 
 $\sup_{z \in \cZ}  K(z,z) \leq 1$, which implies that 
 $
\|  \phi(z) \|_{\cH} \leq 1 
 $
 for all $z\in \cZ$. 
 
Let $\cL^2 (\cZ)$ be the space of square-integrable functions on $\cZ$ with respect to 
Lebesgue measure and let $\la \cdot , \cdot \ra_{\cL^2 }$ be the inner product on $\cL^2 (\cZ)$. 
The kernel function $K$ induces a integral operator $T_K \colon \cL^2 (\cZ) \rightarrow \cL^2 (\cZ)$ defined~as 
\#\label{eq:integral_oper}
T_K f (z) =  \int_{\cZ } K(z, z') \cdot f (z')~\ud z',  \qquad \forall f\in \cL^2 (\cZ). 
\#
By Mercer's Theorem \citep{steinwart2008support}, the integral operator $T_K $ has countable and positive eigenvalues $\{\sigma_i\}_{i\ge 1}$ and the corresponding eigenfunctions $\{ \psi_i\}_{i\ge 1}$ 
form  an orthonormal basis of $\cL^2(\cZ)$. 
Moreover, the kernel function admits a spectral  expansion 
\#\label{eq:kernel_expansion}
K(z, z') = \sum_{i=1}^\infty \sigma_i \cdot \psi_i (z) \cdot  \psi_j (z'). 
\#

The RKHS $\cH$ can thus be written as a subset of $\cL^2 (\cZ)$:
\$
\cH = \biggl \{ f \in \cL^2 (\cZ) \colon \sum_{i=1}^\infty    \frac{ \la f ,\psi_i \ra _{ \cL^2   } ^2 }{\sigma_i } < \infty  \biggr \},
\$
and the inner product of $\cH$ can be written as 
\$
\la f , g\ra _{\cH } = \sum_{i=1}^{\infty} 1/  \sigma_i \cdot \la f ,\psi_i \ra _{\cL^2 } \cdot \la g ,\psi_i \ra _{\cL^2 }, \qquad \text{for all}\quad f, g \in \cH.  
\$
By such a construction, 
the scaled eigenfunctions $\{\sqrt{\sigma_i}\psi_i \}_{i\ge 1} $  form  an orthogonal basis  of RKHS $ \mathcal{H}$ and the feature mapping $\phi(z)  \in \cH$ can be written as 
$
\phi(z)  = \sum_{i=1}^\infty \sigma_i   \psi_i (z) \cdot \psi_i   
$
for any $z\in \cZ$.

\subsection{Overparameterized  Neural Networks} \label{sec:dnn}

We also study a formulation in which value functions are approximated by 
neural networks.  In this section we define the class of overparameterized 
neural networks that will be used in our representation.
 
Recall that we denote $\cZ = \cS \times \cA$ and view it as a subset of $\RR^d$.  
For neural networks, we  further regard $\cZ$ as a subset of the unit sphere in 
$\RR^d$. That is, $\|z \|_2 = 1$ for all $z = (x,a) \in \cZ$.  A two-layer neural 
network  $f (\cdot ; b, W ) \colon \cZ \rightarrow \RR$  with  $2 m$ neurons and  
weights $  (b, W)$  is defined as  
\#\label{eq:two_layer_nn}
f  ( z ;  b, W  ) = \frac{1}{\sqrt{ 2m}}   \sum_{j =1}^{2m}  b_j  \cdot \act(W_{j}^\top z  ) , \qquad \forall z  \in \cZ .
 \#
Here $ \act \colon \RR\rightarrow \RR$ is the \emph{activation function}, 
$b_j \in \RR$ and $ W_j \in \RR^{d  }$ for all $j \in [2m]$, and  
$ b = (b_1, \ldots, b_{2m})^\top \in \RR^{2m}$ and $W = (W_1, \ldots, W_{2m}) \in \RR^{2 dm} $. 
During training, we initialize $(b, W)$ via a symmetric initialization 
scheme  \citep{gao2019convergence, bai2019beyond} as follows. 
For any $ j \in [m]$, we set  $b_j \overset{\text{i.i.d.}}{\sim}  \textrm{Unif}(\{ -1, 1\}) $ and $W_{j } \overset{\text{i.i.d.}}{\sim}  N(0, I_d / d)$, where $I_d
$ is the identity matrix in $\RR^d$.
For any $j \in \{ m+1, \ldots, 2m\}$, we set $b_j = - b_{j-m}$ and $W_j = W_{j-m}$. 
We note that such an initialization implies that  the initial neural network 
is a zero function, which is used only to simply the theoretical analysis.
Moreover, for ease of presentation, during training we  fix    $ b $   
at its initial value  and only optimize over $W$.  Finally, we denote 
$f(z; b, W)$ by $f(z; W)$ to simplify the notation. 
 
We assume that the neural network in is overparameterized in the sense that 
the width $2m$ is much larger than the number of episodes $T$. 
Overparameterization has been shown to be pivotal for neural training 
in both theory and practice \citep{neyshabur2019towards,allen2018learning, arora2019fine}.  
Under such a regime, the dynamics of the training process can be captured by the 
framework of neural tangent kernel (NTK)  \citep{jacot2018neural}. Specifically, 
let $\varphi(\cdot; W) \colon \cZ \rightarrow \RR^{2md}$ be the gradient of 
$f(; W)$ with respect to $W$, which is given by 
\#\label{eq:gradient_feature}
\varphi(z; W) = \nabla _{W} f( z;  W) = \bigl (\nabla _{W_1} f(z; W), \ldots , \nabla _{W_{2m}} f(z; W) ), \qquad \forall z \in \cZ. 
\#
Let $W^{{(0)}}$ be the initial value of $W$. 
Conditioning on the realization of $W^{(0)}$, define a  kernel matrix 
$K_{m} \colon \cZ \rightarrow \cZ $ as 
\#\label{eq:emp_kernel}
K_{m} (z, z') = \bigl \la \varphi(z; W^{(0)}) , \varphi(z'; W^{(0)}) \bigr \ra, \qquad \forall (z, z') \in \cZ \times \cZ. 
\#
When $m$ is sufficiently large, for all $W$ in a neighborhood of $W^{(0)}$, 
it can be shown that $f(\cdot, W)$ is close to its linearization at $W^{(0)}$,
\#\label{eq:linearlization}
f(\cdot; W)\approx \hat f(\cdot; W) = f(\cdot, W^{(0)}) + \bigl \la \phi(\cdot; W^{(0)} ) , W - W^{(0)} \big\ra =\bigl \la \phi(\cdot; W^{(0)}), W - W^{(0)} \big\ra.  
\#
The linearized function $\hat f(\cdot; W)$ belongs to an RKHS with kernel $K_m$. 
Moreover, as $m$ goes to infinity, due to the random initialization, 
$K_m$ converges to a kernel $K_{\mathrm{ntk}} \colon \cZ \times \cZ$, 
referred to as a neural tangent kernel (NTK), which is given by 
\#\label{eq:define_ntk}
K_{\mathrm{ntk}} (z,z') = \EE \bigl [ \act'(w^\top z) \cdot \act'(w^\top z') \cdot z^\top z'\bigr ], \qquad (z,z') \in \cZ\times \cZ,
\#
where $\act'$ is the derivative of the activation function, and the expectation  
in \eqref{eq:define_ntk} is taken with respect to $w \sim N(0, I_d / d)$.

%% file: algorithm.tex

\section{Optimistic Least-Squares Value Iteration Algorithms}

In this section, we introduce the \emph{optimistic  least-squares 
value iteration algorithm} based on the representations of action-value 
functions discussed in the previous section.  The value iteration 
algorithm \citep{puterman2014markov,sutton2018reinforcement},
which computes $\{ Q_h^{\star} \} _{h \in [H] }$ by applying the 
Bellman equation in \eqref{eq:opt_bellman} recursively, is one of 
the classical methods in reinforcement learning.  In more detail, the 
algorithm forms a sequence of action-value functions, $\{ Q_h \}_{h \in [H] }$, 
via the following recursion:
\#\label{eq:value_iter}
Q _{h}(x, a)  \leftarrow  (\TT_h^{\star} Q_{h+1} ) & =   (r_h + \PP_h V_{h+1})(x, a), \\
V_{h+1} (x)   \leftarrow \max_{a '\in\cA} Q _{h+1}(x, a'),&\qquad \forall(x, a) \in \cS\times \cA, \forall h \in [H],   \notag 
 \# 
where $Q_{H+1}$ is set to be the zero function.  To turn this generic 
formulation into a practical solution to real-world RL problems, it
is necessary to address two main problems: (i) the transition kernel 
$\PP_h$ is unknown, and (ii) we can neither iterate over all state-action 
pairs nor store a table of size $|\cS \times \cA|$ when the number of  
states is large.  To address these challenges, the least-squares value 
iteration \citep{bradtke1996linear,osband2014generalization} algorithm  
implements an approximate update to \eqref{eq:value_iter} by solving 
a least-squares regression problem based on historical data consisting 
of  the trajectories generated by the learner in previous episodes.  
Specifically, at the onset of the $t$-th episode, we have observed  
$t-1$ transition tuples, $\{ (x_h^{\tau} , a_h^{\tau},   
x_{h+1}^{\tau}) \}_{\tau\in [t-1]}$, and  LSVI proposes to
estimate $Q_h^\star$ by replacing \eqref{eq:value_iter} with the
following least-squares regression problem:
\#\label{eq:least_squares_vi}
\hat Q_{h } ^t \leftarrow  \minimize_{f \in \cF} \bigg\{  \sum_{\tau =1}^{t-1} \bigl [   r_h ( x_h^\tau , a_h^\tau) +  V_{h+1}^t ( x_{h+1} ^\tau ) - f  ( x_h^{\tau },  a_h^\tau )    \bigr ] ^2  + \mathrm{pen}(f) \biggr \} ,
\#
where $\cF$ is a function class, and $\mathrm{pen} (f)$ is a regularization term. 

Although classical RL methods did not focus on formal methods for exploration,
in our setting it is critical to incorporate such methods.  Following the 
principle of \emph{optimism in the face of uncertainty}~\citep{lattimore2018bandit},
we define a bonus function $b_h^t \colon \cZ \rightarrow \RR$ and write
\#\label{eq:optimistic_VI}
Q_h^t (\cdot, \cdot) = \min  \big \{ \hat Q_h^t(\cdot, \cdot) + \beta \cdot b_h^t 
(\cdot, \cdot ) , H- h +1  \bigr  \}^+, \qquad V_h^t (\cdot) = \max_{a\in \cA} Q_h^t(\cdot, a), 
\# 
where $\beta > 0$ is a parameter  and $\min \{ \cdot ,  H-h+1 \}^{+}  $ 
denotes a truncation to the  interval $[0, H-h-1]$.  Here, we truncate 
the value function $Q_h^t$ to $[0, H-h+1]$ as each reward function is bounded in $[0, 1]$.  
In the $t$-the episode, we then let $\pi^t$ be the greedy policy with 
respect to $\{ Q_h^t \}_{h \in [H]}$ and execute $\pi^t$.  Combining 
\eqref{eq:least_squares_vi} and \eqref{eq:optimistic_VI} yields our 
optimistic least-squares value iteration algorithm, whose details are 
given in Algorithm \ref{algo:lsvi}. 

\begin{algorithm}[ht]
  \caption{Optimistic Least-Squares Value Iteration with Function Approximation}\label{algo:lsvi}
  \begin{algorithmic}[1]
  \STATE{\textbf{Input:} Function class $\cF$, penalty function $\textrm{pen} (\cdot)$, and parameter $\beta$.}
  \FOR{episode $t = 1, \ldots, T$}
  \STATE{Receive the initial state $x^t_1$.}
  \STATE{Set $V_{H+1}^t$ as the zero function.}
  \FOR{step $h = H, \ldots, 1$}
  \STATE {Obtain $Q_h^t$ and $V_h^t$ according to \eqref{eq:least_squares_vi} and \eqref{eq:optimistic_VI}.} 
  \ENDFOR
  \FOR{step $h = 1, \ldots, H$}
  \STATE Take action $a^t_h \gets  \argmax_{a \in \cA } Q_h^t (x^t_h, a)$. 
  \STATE Observe the  reward $r_h(x_h^t, a_h^t)$ and the next state $x^t_{h+1}$. 
  \ENDFOR
  \ENDFOR
  \end{algorithmic}
  \end{algorithm}

 
We note that both the  bonus function $b_{h}^t$ in \eqref{eq:optimistic_VI} 
and  the penalty function in \eqref{eq:least_squares_vi} depend on the 
choice of function class $\cF$, and the optimistic LSVI in Algorithm 
\ref{algo:lsvi} is only implementable when $\cF$ is specified.
For instance, when $\cF$ consists of linear functions, $\theta^\top \phi(z)$, 
where $\phi \colon \cZ \rightarrow \RR^d$ is a known feature mapping and 
$\theta\in \RR^d$ is the parameter,  we choose the ridge penalty 
$\| \theta\|_2^2$ in \eqref{eq:least_squares_vi} and define $b_h^t(z)$ 
as   $ [  \phi (z) ^\top A_h^t \phi(z)] ^{1/2}$ for some invertible matrix $ A_h^t  $.
In this case, Algorithm \ref{algo:lsvi} recovers the LSVI-UCB algorithm 
studied in \cite{jin2019provably}, which further reduces to the tabular 
UCBVI algorithm \citep{azar2017minimax} when $\phi$ is the canonical basis. 

In the rest of this section, we instantiate the optimistic LSVI framework 
in two ways, by setting $\cF$ to be an RKHS and the class of overparameterized 
neural networks.

\subsection{The Kernel Setting} \label{sec:kernel_lsvi} 

In the following, we consider the case where function class $\cF$ is 
an RKHS $\cH$ with kernel $K$.  In this case, by setting $\textrm{pen}(f)$ 
to be the ridge penalty, \eqref{eq:least_squares_vi} reduces to a kernel 
ridge regression problem.  Moreover, we define $b_h^t$ in 
\eqref{eq:optimistic_VI} as the UCB bonus function that also appears in 
kernelized contextual bandit algorithms~\citep{srinivas2009gaussian, 
valko2013finite, chowdhury2017kernelized,durand2018streaming, 
yang2019reinforcement, sessa2019no, calandriello2019gaussian}. 
With these two modifications, we define the \emph{Kernel Optimistic 
Least-Squares Value Iteration} (KOVI) algorithm, which is summarized 
in  Algorithm~\ref{algo:kucb}. 

Specifically, for each $t\in [T]$,  at the beginning of  the $t$-th episode, 
we first obtain  value functions $\{ Q_h^t\}_{h \in [H]}$ by solving a sequence of    kernel ridge regressions with the data obtained from the previous $t-1$ episodes. 
In particular, letting $Q_{H+1}^t $ be a zero function, for any $h \in [H]$, 
we replace \eqref{eq:least_squares_vi} by a kernel ridge regression given by 
\# \label{eq:krr_vi}
\hat Q_{h} ^t\leftarrow  \minimize_{f \in \cH}  \sum_{\tau =1}^{t-1}  \bigl [   r_h ( x_h^\tau , a_h^\tau) + V_{h+1}  ^t( x_{h+1} ^\tau ) - f  ( x_h^{\tau },  a_h^\tau )    \bigr ] ^2 + \lambda \cdot \| f \|_{\cH}^2 , 
\#
where $\lambda > 0 $ is the regularization parameter. 
We then obtain $Q_h^t$ and $V_h^t$ as in \eqref{eq:optimistic_VI}, 
where the bonus function $b_h^t$ will be specified later.  We have:
\#\label{eq:update_Q_func}
Q_h^t (s,a) = \min \bigl  \{\hat Q_h^t (s,a) + \beta \cdot b_h^t (s,a), 
H - h+1   \bigr \}^+ , \qquad V_{h}^t (s) = \max_{a} Q_h^t(s,a),
\#
where $\beta > 0$ is a parameter. 
 
The solution to \eqref{eq:krr_vi} can be written in closed form as follows. 
We define the response vector $y_h^t \in \RR^{t-1}$ by letting its $\tau$-th entry be  
 \#\label{eq:krr_response}
 [ y_h^t ]_{\tau} = r_h (x_h^\tau , a_h^\tau ) + V_{h+1} ^t(x_{h+1}^\tau ) , \qquad \forall \tau \in [t-1]. 
 \#   
 Recall that we denote $z = (x,a)$ and $\cZ = \cS \times \cA$.  
Based on the kernel function $K$ of the RKHS,  we define the Gram 
matrix $K_h^t \in \RR^{(t-1) \times (t-1)}$ and function 
$k_h^t  \colon \cZ \rightarrow \RR^{t-1}$ respectively as 
\#\label{eq:define_gram}
 K_h^t  = [ K(z_h^\tau , z_h^{\tau  '}) ] _{\tau , \tau ' \in [t-1] } \in \RR^{(t-1) \times (t-1)} , \qquad k_h^t (z) =  \bigl  [  K(z_h^1, z) , \ldots K(z_h^{t-1} , z)  \bigr ] ^\top \in \RR^{t-1}.
 \#
 Then $\hat Q_h^t$ in \eqref{eq:krr_vi} can be written as $\hat Q_h^t(z) =  k_h^t(z)    ^\top \alpha_h^t$, where we define $\alpha_h^t =  (K_h^t + \lambda \cdot  I  )^{-1}  y_h^t.$

Using $K_h^t$ and $k_h^t$ defined in \eqref{eq:define_gram}, the bonus 
function is defined as 
\#\label{eq:ucb_bonus}
b_h^t (x,a) = \lambda^{-1/2}\cdot  \bigl [  K(z,z) - k _h^t(z) ^\top (K_h^t + \lambda I  )^{-1} k_h^t(z) \bigr ]  ^{1/2} ,
\#
which can be interpreted as the posterior variance of a Gaussian 
process regression~\citep{rasmussen2003gaussian}.  This bonus term   
reduces to the UCB bonus used for linear  bandits~\citep{bubeck2012regret, 
lattimore2018bandit} when the feature mapping $\phi$ of the RKHS is 
finite-dimensional.  Moreover, in this case, KOVI reduces to the 
LSVI-UCB algorithm proposed in \cite{jin2019provably} for linear 
value functions.
 Furthermore, since both $\alpha_h^t$ and $b_h^t$ admit close-form expressions, the runtime complexity of \KUCB~is a polynomial of $H$ and $T$. 
 
We refer to the bonus defined in \eqref{eq:ucb_bonus} as a UCB bonus 
because, when combined with $Q_h^t$ as defined in \eqref{eq:update_Q_func},
it serves as an upper bound of $Q_h^\star$ for all state-action pairs.  
In more detail, intuitively the target function of the kernel ridge 
regression in \eqref{eq:krr_vi}  is $\TT_{h }^\star Q_{h+1}^t$. 
Due, however, to having limited data, the solution $\hat Q_h^t$ has 
some estimation error, as quantified by $b_h^t$.  When $\beta$ is 
properly chosen, the bonus term captures the uncertainty of estimation, 
yielding $Q_h^t \geq \TT_{h }^\star Q_{h+1}^t$ elementwisely.  Notice 
that $Q_{H+1}^t = Q_{H+1}^\star  = 0$.  The Bellman equation,  
$Q_h ^\star = \TT_h^\star Q_{h+1} ^\star$,  directly implies that 
$Q_h^t $ is an elementwise upper bound of $Q_{h}^\star$ for all $h \in [H]$. 
Our algorithm is thus called ``optimistic value iteration'',  as the 
policy $\pi^t$ is greedy with respect to $\{ Q^t_h \}_{h \in [H]}$, 
which is an upper bound on the optimal value function.  In other words, 
compared with a standard value iteration algorithm, we always 
over-estimate the value function. Such an optimistic approach is 
pivotal for the RL agent to perform efficient, temporally-extended 
exploration.

\begin{algorithm}[ht]
\caption{Kernelized  Optimistic Least-Squares Value Iteration (\KUCB)}\label{algo:kucb}
\begin{algorithmic}[1]
 \STATE{\textbf{Input:} Parameters $\lambda$ and $\beta$.}
  \FOR{episode $t = 1, \ldots, T$}
\STATE{Receive the initial state $x^t_1$.}
\STATE{Set $V_{H+1}^t$ as the zero function.}
\FOR{step $h = H, \ldots, 1$}
\STATE Compute the response $y_h^t \in \RR^{t-1} $,  the Gram matrix $K_h^t \in \RR^{(t-1) \times (t-1) }$, and function $k_h^t$ as in \eqref{eq:krr_response} and \eqref{eq:define_gram}, respectively.
\STATE Compute 
\STATE \qquad $\alpha_h^t  = (K_h^t + \lambda \cdot  I  )^{-1}  y_h^t,$ 
\STATE  \qquad $b_h^t (\cdot, \cdot ) = \lambda^{-1/2}\cdot  \bigl [  K( \cdot, \cdot; \cdot, \cdot ) - k _h^t(\cdot, \cdot) ^\top (K_h^t + \lambda I  )^{-1}  k_h^t(\cdot, \cdot) \bigr ]  ^{1/2} $. \label{line:bonus}
\STATE{Obtain value functions $$Q_h^t(\cdot, \cdot) \leftarrow \min\{ k_h^t (\cdot, \cdot)^\top \alpha_h^t + \beta   \cdot b_h^t(\cdot, \cdot ) , H - h +1 \}^+ , \qquad V_h^t(\cdot) = \max_{a} Q_h^t(\cdot, a).$$} \label{line:ucb}
\ENDFOR
\FOR{step $h = 1, \ldots, H$}
\STATE Take action $a^t_h \gets  \argmax_{a \in \cA } Q_h^t (x^t_h, a)$. 
\STATE Observe the  reward $r_h(x_h^t, a_h^t)$ and the next state $x^t_{h+1}$.  
\ENDFOR
\ENDFOR
\end{algorithmic}
\end{algorithm}

\subsection{The Neural Setting} \label{sec:NN_VI}
 
An alternative is to estimate the value functions for an RL agent using 
overparameterized neural networks.  In particular, we estimate each 
$Q_h^\star $ using a neural network as given in \eqref{eq:two_layer_nn}, 
again using a symmetric initialization scheme \citep{gao2019convergence, bai2019beyond}.
We assume, for simplicity, that all the neural networks share the same 
initial weights, denoted by $(b^{(0)}, W^{(0)})$.  In addition, we fix 
$b = b^{(0)}$ in \eqref{eq:two_layer_nn} and only update the value of 
$W \in \RR^{2md}$.  

In the neural network setting, we replace the least-squares regression 
in \eqref{eq:least_squares_vi} by a nonlinear ridge regression. 
That is, for any $(t, h) \in [T] \times [H]$, we define the   
loss function $L_h^t   \colon \RR^{2md} \rightarrow \RR$ as 
 \#\label{eq:nn_loss} 
  L_{h}^t (W    ) =  \sum_{\tau =1}^{t-1}  \bigl [   r_h ( x_h^\tau , a_h^\tau) + V _{h+1}^t( x_{h+1} ^\tau  ) -  f  ( x_h^{\tau },  a_h^\tau ; W  )    \bigr ] ^2 + \lambda \cdot    \big \|  W - W^{(0)}  \big \|_2  ^2 , 
 \# 
where   $\lambda >0$ is the regularization parameter. 
We then define $\hat Q_h^t$ as 
\#\label{eq:dnn_vi}
\hat Q_h^t (\cdot, \cdot  ) = f\bigl (\cdot, \cdot ; \hat W_h^t \bigr ), \qquad \text{where} \quad \hat W_h^t = \argmin_{W \in \RR^{2md}} L_h^t (W).
\#
Here we assume that there is an optimization oracle that returns the  global minimizer 
of the loss function $L_h^t$. 
It has been shown in a large body of literature that, when $m$ is sufficiently large, with random initialization, simple optimization methods  such as 
 gradient descent provably find  the global minimizer of the empirical loss function at a linear rate  of convergence \citep{du2018gradient, du2018gradient1, arora2019fine}. 
Thus, such an optimization oracle can be realized by gradient descent with sufficiently large number of iterations and the computational cost of realizing such  an oracle  is  a  polynomial in $m$, $T$, and $H$.

It remains to construct the bonus function $b_h^t$. 
Recall that we define $\varphi(\cdot; W) = \nabla_{W} f(\cdot; W)$ in \eqref{eq:gradient_feature}.
We define a matrix 
$\Lambda_h^t \in \RR^{2md\times 2md }$  as 
 \#\label{eq:lambda_mat}
 \Lambda_h^t = \lambda \cdot I_{2md } + \sum_{\tau = 1}^{t-1} 
 \varphi\bigl  ( x_h^\tau, a_h^\tau ;  \hat W_h^{\tau+1} \bigr )   \varphi \big ( x_h^\tau, a_h^\tau ;\hat W_h^{\tau+1} \bigr )  ^\top  ,
 \#
 which can be recursively computed by letting 
 $$
 \Lambda _h^1 = \lambda \cdot I_{2md}, \qquad \Lambda _h^t = \Lambda _h^{t-1} +    \varphi\bigl  ( x_h^{t-1}, a_h^{t-1} ;  \hat W_h^{t} \bigr )   \varphi \big ( x_h^{t-1}, a_h^{t-1} ;\hat W_h^{t} \bigr )  ^\top , \qquad \forall t \geq 2.
 $$
 Then the bonus function  $b_{h}^t$ is defined as follows:
 \#\label{eq:neural_bonus}
 b_h^t (x,a ) = \bigl [ \varphi \big ( x,a ; \hat W_h^t \bigr ) ^\top ( \Lambda_h^t ) ^{-1} \varphi \big ( x, a; \hat W_h^t \big ) \bigr ] ^{1/2}, \qquad \forall (x,a) \in \cS\times \cA. 
 \#
Finally, we obtain the value functions $Q_h^t$ and $V_h^t$ via 
 \eqref{eq:update_Q_func},
 with $\hat Q_h^t$ and $b_h^t$ defined in \eqref{eq:dnn_vi} and \eqref{eq:neural_bonus}, respectively.
Letting $\pi^t$ be the greedy policy with respect to $\{ Q_h^t\}_{h\in [H]}$, 
we have defined our \emph{Neural Optimistic Least-Squares Value Iteration}
(\NUCB) algorithm, the pseudocode for which is presented in
  Algorithm \ref{algo:neural} in  \S\ref{sec:nn_algo}. 
  Since  $b_h^t$ in \eqref{eq:neural_bonus} admits a closed-form expression  and $\hat W_h^t$ defined in  \eqref{eq:dnn_vi} can be obtained in polynomial time, the runtime complexity of \NUCB~is a polynomial of $m$, $T$, and $H$.

An intuition for the bonus term in \eqref{eq:neural_bonus} can be obtained 
via the connection between overparameterized neural networks and neural
tangent kernels.  Specifically, when $m$ is sufficiently large, it can be shown 
that each $\hat W_h^t$ is not far from the initial value $W^{(0)}$. When this is 
the case, suppose we replace the neural tangent features 
$\{ \varphi(\cdot; \hat W_h^\tau) \} _{\tau \in [T]}$ in 
\eqref{eq:lambda_mat} and \eqref{eq:neural_bonus} by $\varphi (\cdot; W^{(0)} )$,
then  $ b_h^t $ recovers the UCB bonus in linear contextual bandits and  linear 
MDPs with the feature mapping $\varphi(\cdot ; W^{(0)})$~\citep{abbasi2011improved, 
jin2019provably,wang2019optimism}. Moreover, when $m$ goes to infinity,  
we obtain the UCB bonus defined in \eqref{eq:ucb_bonus} for the RKHS setting  
with  the kernel being   $K_{\textrm{ntk}}$.  Accordingly, when working with
overparameterized neural networks, the value functions $\{ Q_h^t \}_{h \in [H]}$ 
are approximately elementwise upper bounds for the optimal value functions, 
and we again obtain optimism-based exploration at least approximately.

%% file: results.tex

\section{Main Results} \label{sec:results}

In this section, we prove that both  \KUCB~and \NUCB~achieve  $ \tilde \cO( \delta_{\cF} 
H^2 \sqrt{T} )$ regret, where $\delta_{\cF}$ characterizes the intrinsic complexity 
of the function  class $\cF$  used to approximate $\{ Q_h^\star\}_{h \in [H]}$. 
We first consider the kernel setting and then turn to the neural network setting.
 
\subsection{Regret of \KUCB}

Before presenting the theory, we first lay out a  structural  assumption  for the 
kernel setting, which  postulates that the Bellman operator maps any bounded 
value function to a bounded RKHS-norm ball.  

\begin{assumption}\label{assume:opt_closure}
	Let $R_{Q}> 0$ be a fixed constant. We define $\cQ^\star  =  \{ f\in \cH \colon \| f\|_{\cH} \leq R_{Q} H \}$. 
	We assume that for any $h \in [H]$ and  any $Q \colon \cS \times \cA \rightarrow [0,H] $, we have $\TT_h^\star Q  \in \cQ^\star$. 
	\end{assumption}

Since $Q_h^{\star}$ is bounded by in $[0,H]$ for each all $h \in [H]$, 	
Assumption \ref{assume:opt_closure} ensures that  the optimal value functions 
are contained in the RKHS-norm ball $\cQ^\star$.  Thus, there is no approximation 
bias when using functions in $\cH$ to approximate $\{Q_h^{\star} \}_{h\in [H]}$. 
A sufficient condition for Assumption \ref{assume:opt_closure} to hold  is that 
\#\label{eq:suff_cond_closure}
 r_h (\cdot, \cdot ),~ \PP_h(x' \given \cdot, \cdot)   \in \{f\in \cH \colon \| f \|_{\cH } \leq 1 \}, \qquad   \forall h \in [H],~\forall x' \in \cS.  
\#   
That is, both the reward function and the Markov transition kernel can be represented 
by functions in the unit ball of $\cH$.  When \eqref{eq:suff_cond_closure} holds, 
for any $V \colon \cS \rightarrow [0, H]$, it holds that $r_h + \PP_h V \in \cH$ 
with its RKHS norm bounded by $H+1$. Hence, Assumption \ref{assume:opt_closure} 
holds with $R_{Q} = 2$.  Similar assumptions are made in 
\cite {yang2019sample, yang2019reinforcement, jin2019provably, 
zanette2019frequentist, zanette2020learning,  wang2019optimism} 
for (generalized) linear functions.  See also \cite{du2019good,van2019comments, 
lattimore2019learning} for a discussion of the necessity of such an assumption. 

It is shown in \cite{du2019good} that only assuming $\{Q_h^{\star} \}_{h\in [H]} 
\subseteq \cQ^\star$ is not sufficient for achieving a regret that is polynomial 
in $H$.  Thus, we further assume that $\cQ^\star $ contains the image of the Bellman 
operator.  Given this assumption, the complexity of $\cH$ plays an important 
role in the performance of \KUCB.  To characterize  the intrinsic complexity 
of $\cF$, we consider the maximal information gain \citep{srinivas2009gaussian}:
\#\label{eq:maintext_infogain}
\Gamma_K(T, \lambda) = \sup _{\cD \subseteq \cZ  }\bigl \{   1/2 \cdot  \logdet  ( I + K_\cD / \lambda    )  \bigr \} , 
\#
where the supremum is taken over all $\cD\subseteq \cZ$ with $| \cD | \leq T$.  
Here $K_{\cD}$ is the Gram matrix defined as in \eqref{eq:define_gram}; i.e.,
it is based on $\cD$, $\lambda > 0$ is a parameter, and the subscript $K$ 
in $\Gamma_K$ indicates the kernel $K$.  The  magnitude of $\Gamma_K(T, \lambda)$ 
relies on how fast the the eigenvalues   $\cH$ decay to zero and can be 
viewed as a proxy of the  dimension  of $\cH$ when $\cH$ is infinite-dimensional.  
In the special case where $\cH$ is  finite rank, we have that 
$\Gamma_K(T, \lambda) = \cO( \gamma \cdot \log T )$ where $\gamma$ is 
the rank of    $\cH$. 

Furthermore, for any $h \in [H]$, note that each $Q_h^t$ constructed 
by \KUCB~takes the following form:
 \#\label{eq:value_func_ucb}
 Q(z  )=   \min \bigl \{ Q_0( z ) + \beta \cdot \lambda^{-1/2}   \bigl [  K(z,z) - k_{\cD}(z) ^\top (K_{\cD} + \lambda I  )^{-1} k_{\cD} (z) \bigr ]  ^{1/2}, H - h  + 1 \bigr \}^+,
 \#
where $Q_0 \in \cH$, an analog of $\hat Q_h^t$ in \eqref{eq:krr_vi},  
is the solution to a kernel ridge regression problem and 
$\cD \subseteq \cZ$ is a discrete subset of $\cZ$ containing no more 
than $T$ state-action pairs. Moreover, $K_{\cD}$ and $k_{\cD}$ are 
defined analogously to \eqref{eq:define_gram} based on data in $\cD$. 
Then, for any $R, B > 0$,  we define a function class $\cQ_{\UCB}( h, R, B)$  as 
 \#\label{eq:main_text_func_class}
 \cQ_{\UCB} (h, R, B)= \bigl\{    Q    \colon Q\textrm{~takes the form of \eqref{eq:value_func_ucb} with~} \| Q_0 \|_{\cH } \leq R, \beta \in [0, B],  | \cD| \leq T\bigr \}.
 \#
 As we  will show in Lemma \ref{lemma:thetahat_estimate}, 
 we have $\| \hat Q_h^t \|_{\cH} \leq R_T$ for all $(t, h) \in [T] \times [H]$, where  $R_T = 2 H \sqrt{  \Gamma_K(T, \lambda)}$.
Thus,  when $B$  exceeds the value  $\beta$ in \eqref{eq:update_Q_func}, 
each $Q_h^t$  is contained in $\cQ_{\UCB}(h, R_T, B)$. 
 
Moreover, since $r_h + \PP_h V_{h+1}^t = \TT_h^\star Q_{h+1}^t $ is the 
population ground truth  of the ridge regression  in  \eqref{eq:krr_vi}, 
the complexity  of $\cQ_{\UCB} ( h+1 , R_T, B)$ naturally appears when 
quantifying the uncertainty of $\hat Q_h^t$.  To this end, for any 
$\epsilon > 0$, let $N_{\infty} ( \epsilon; h, B)$ be the $\epsilon$-covering 
number of $\cQ_{\UCB} ( h, R_T, B )$ with respect  to  the 
$\ell_{\infty}$-norm on $\cZ$, which characterizes the complexity 
of the value functions constructed by KOVI and which is determined 
by the spectral structure of $\cH$.
    
We are now ready to present a regret bound for \KUCB. 
\begin{theorem} \label{thm:main}
  	Assume that there exists $B_T > 0$ satisfying 
  	\#\label{eq:equation_B_T}
  	 8  \cdot \Gamma_K(T, 1 + 1/T)     + 8    \cdot    \log N_{\infty }(  \epsilon^*; h , B_T)    + 16   \cdot \log (   2 T H  )  + 22  + 2 R_Q^2     \leq (B_T / H)^2  
  	\#
  	for all $h \in [H]$, where $\epsilon^* = H/ T$.
  	We set $\lambda = 1 + 1/ T$ and $\beta = B_T $
  	  in Algorithm \ref{algo:kucb}.
  	  Then, under Assumption \ref{assume:opt_closure}, 
  	  with probability at least $1 - (T^2H^2)^{-1}$, we have 
  	\#\label{eq:kucb_regret}
  	\mathrm{Regret}(T) \leq  5 \beta H \cdot \sqrt{T\cdot \Gamma_K(T, \lambda)} .
  	\#
  \end{theorem}
 
As shown in \eqref{eq:kucb_regret}, the regret can be written as 
 $\cO( H^2 \cdot  \delta_{\cH}\cdot    \sqrt{ T} ) $, where 
$\delta_{\cH} = B_T/ H  \cdot \sqrt{\Gamma_K(T, \lambda)}$ 
reflects   the  complexity  of $\cH$ and $B_T$ satisfies \eqref{eq:equation_B_T}. 
Specifically, $\delta_\cH$ involves (i) the $\ell_{\infty}$-covering number 
$  N_\infty(\epsilon^*, h,B_T)$ of $\cQ_{\UCB} (h, R_T, B_T)$, and 
(ii) the effective dimension $\Gamma_K(T, \lambda)$.
Moreover, when neglecting the constant and logarithmic 
terms in \eqref{eq:equation_B_T}, it suffices to choose $B_T$ satisfying 
$$
B_T / H \asymp \sqrt{ \Gamma_K (T, \lambda)} + \max_{h\in [H]}  
\sqrt{ \log N_{\infty} (\epsilon^*, h, B_T)},
$$
which reduces the  regret bound in \eqref{eq:kucb_regret}  to the following:
   \begin{tcolorbox}[width=\textwidth, colback=blue!5!white,colframe=white]
  \#\label{eq:kernel_regret_simple}
  \mathrm{Regret}(T) = \tilde \cO \Bigl ( H^2 \cdot \underbrace{\Bigl [\Gamma_{K} (T, \lambda)  + \max_{h\in [H]} \sqrt{\Gamma_{K} (T, \lambda ) \cdot \log N_{\infty} (\epsilon^*, h, B_T) } \Bigr ] }_{\dr \delta_\cH\colon\textrm{complexity of function class}~\cH}\cdot \sqrt{T} \Bigr  ).
  \#
    \end{tcolorbox}

To obtain some intuition for  \eqref{eq:kernel_regret_simple}, 
let us  consider the tabular case where $\cQ^\star $ consists of 
all measurable functions   defined on $\cS\times \cA$ with range $[0, H]$. 
In this case, the value function class $ \cQ_{\UCB} (h, R_T, B_T)$ 
can be set to $\cQ^\star$, whose $\ell_{\infty}$-covering number 
$N_{\infty}(\epsilon^*, h, B_T) \leq |\cS \times \cA| \cdot  \log T$. 
It can be shown that the effective dimension is also $  \cO(| \cS \times 
\cA|\cdot \log T)$.  Thus, ignoring the logarithmic terms, Theorem 
\ref{thm:main} implies that  by choosing $\beta \asymp  H \cdot | \cS \times \cA|$, 
optimistic least-squares value iteration achieves an  $\tilde \cO( H^2 \cdot | \cS \times \cA| \cdot \sqrt{T} )$ regret. 
    
Furthermore, we  remark that the  regret bound in \eqref{eq:kucb_regret} 
holds in general for any RKHS.  The result hinges on (i) Assumption 
\ref{assume:opt_closure}, which postulates that the RKHS-norm ball 
$\{ f\in \cH \colon \| f\|_{\cH} \leq R_Q H \}$ contains the image of 
the Bellman operator, and (ii) the inequality in \eqref{eq:equation_B_T} 
admits a solution $B_T$, which is set to be $\beta $ in Algorithm \ref{algo:kucb}. 
Here we set  $\beta $  to be sufficiently large  so as to dominate the 
uncertainty of $\hat Q_h^t$,  whereas to quantify such uncertainty, 
we utilize the  uniform concentration over the  value function class  
$ \cQ_{\UCB} (h+1, R_T, \beta)$ whose complexity metric, the  
$\ell_\infty$-covering number,   in turn depends on $\beta$.  
Such an intricate desideratum  leads to \eqref{eq:equation_B_T} 
which determines $\beta$ implicitly.
 
 It is worth noting that the uniform concentration is unnecessary when $H = 1$. In this case, it suffices to choose $\beta =  \tilde \cO (  \sqrt{ \Gamma_K(T, \lambda)} )$  and 
\KUCB~incurs a $\tilde \cO(  \Gamma_K(T, \lambda) \cdot \sqrt{T})$ regret, which matches 
 the regret bounds of UCB algorithms for  kernelized  contextual bandits in 
 \cite{srinivas2009gaussian} and \cite{chowdhury2017kernelized}. 
Here $\tilde O(\cdot )$ omits logarithmic terms. 
Thus, the covering number in \eqref{eq:kernel_regret_simple} is specific 
for MDPs and  arises due to the  temporal dependence within an episode.
 
To obtain a concrete regret bound from \eqref{eq:kucb_regret}, it remains 
to further characterize $\Gamma_K(T, \lambda)$ and $\log N_{\infty} 
(\epsilon^*, h, B_T)$ using characteristics of $\cH$.  To this end, 
we specify the eigenvalue decay property of $\cH$.
    
\begin{assumption} [Eigenvalue Decay of $\cH$] \label{assume:decay}
	Recall that the integral operator $T_K$ defined in \eqref{eq:integral_oper} has eigenvalues $\{ \sigma_j \}_{j \geq 1}$ and eigenfunctions $\{ \psi_j \}_{j\geq 1}$. 
	We assume that  $\{ \sigma_j \}_{j \geq 1}$ satisfies 
	one of the following three eigenvalue decay conditions for some constant $\gamma>0$:  
\begin{itemize}
	\item [(i)] $\gamma$-finite spectrum: we have $\sigma_j = 0$ for all $j > \gamma$, where $\gamma$ is a positive integer. 
	\item [(ii)] $\gamma$-exponential decay: there exist  absolute constants $C_1$ and $C_2$ such that $\sigma_j \leq C_1 \cdot \exp( - C_2 \cdot j^\gamma)$ for all $j \geq 1$.  
	\item [(iii)] $\gamma$-polynomial decay: there exists a constant $C _1$ such that $\sigma_j \geq C_1 \cdot j^{-\gamma}$ for all $j \geq 1$, where $\gamma > 1$. 
\end{itemize}
For cases (ii) and (iii), we further assume that there exist constants 
$\tau \in [ 0,1/2)$  $C_{\psi} >0 $ such that $ \sup_{z \in \cZ } \sigma_j^\tau   \cdot | \psi_j  (z) | \leq C_{\psi}$ for all $j  \geq 1$.

\end{assumption}
  
Case (i) implies that $\cH$  is a $\gamma$-dimensional RKHS. 
When this is the case, under 
Assumption \ref{assume:opt_closure},
  there exists a feature mapping $\phi \colon \cZ \rightarrow \RR^\gamma$ such that, for any $V \colon \cS\rightarrow [0, H] $, $r_h + \PP_h V$ is a linear function of $\phi$. 
Such a property is satisfied by the linear MDP model studied in 
\cite{yang2019sample}, \cite{yang2019reinforcement}, \cite{jin2019provably}, 
and \cite{zanette2019frequentist}.  Moreover, when $\cH$ satisfies 
case (i), \KUCB~reduces to the LSVI-UCB algorithm studied in \cite{jin2019provably}. 
In addition, cases (ii) and (iii) postulate that the eigenvalues of 
$T_K$ decay exponentially and polynomially  fast, respectively, 
where $\gamma$ is a constant that might depend on the input dimension $d$, 
which is assumed fixed throughout this paper. For example,   the squared
exponential kernel belongs to case (ii) with  $\gamma = 1/d$, whereas 
the Mat\'ern  kernel with parameter $\nu > 2$ belongs to case 
(iii) with $\gamma = 1+ 2\nu /d$ \citep{srinivas2009gaussian}. 
Moreover,  we assume that there exists $\tau \in [0,1/2) $ such that 
$\sigma_j^\tau \cdot  \| \psi_j \|_{\infty} $ is universally bounded. 
Since $K(z,z) \leq 1$, this condition is naturally satisfied for $\tau = 1/2$. 
However, here we assume that $\tau \in (0,1/2)$, which is satisfied 
when the magnitudes of the eigenvectors do not grow too fast compared 
to the decay of the eigenvalues.  Such a condition is significantly 
weaker than assuming $\| \psi_j \|_{\infty}$ is universally bounded, 
which is also commonly made in the   literature on nonparametric 
statistics \citep{lafferty2005diffusion,shang2013local,  zhang2015divide, 
lu2016nonparametric,yang2017frequentist}.  It can be shown that the 
squared exponential kernel on the unit sphere in $\RR^d$ satisfies 
this condition for any $\tau >0$ --- see Appendix~\S\ref{sec:ntk_examples} --- and
the Mat\'ern kernel on $[0,1]$ satisfy this property with $\tau = 0$ 
\citep{yang2017frequentist}.  We refer to \cite{mendelson2010regularization} 
for a more detailed discussion.  

We now present regret bounds for the three eigenvalue decay conditions.
 
\begin{corollary} \label{cor:main}
	Under Assumptions \ref{assume:opt_closure}  and \ref{assume:decay},
   we set $\lambda = 1 + 1/T$ and $\beta = B_T$ in Algorithm \ref{algo:kucb}, where $B_T$ is defined as 
	\#\label{eq:set_BT}
	B_T = \begin{cases} 
		C_b \cdot \gamma H \cdot \sqrt{ \log (\gamma \cdot  TH)} &\qquad \textrm{$\gamma$-finite spectrum}, \\
			C_b  \cdot H \sqrt{ \log (TH)}
			\cdot (\log T)^{1/\gamma}&\qquad \textrm{$\gamma$-exponential decay}, \\
				C_b \cdot H  \log (TH)   \cdot  T^{\kappa^*}    &\qquad \textrm{$\gamma$-polynomial decay}.
				\end{cases}
	\#
	Here   $C_b$ is an absolute constant that does not depend on $T$ or $H$, and $\kappa^*$ is  given by 
	\#\label{eq:define_poly_kappa}
	\kappa^* = \max \bigg\{ \xi^*, \frac{2d+\gamma +1 }{(d+\gamma)\cdot [\gamma (1-2\tau )-1]}, \frac{2}{\gamma (1- 2\tau ) -3 }  \biggr \}, \qquad \xi^* = \frac{d+1}{2(\gamma+d)}. 
	\#
 For  the third case, we further assume that $\gamma $ is sufficiently large such that $\kappa^* + \xi^* < 1/2$. 
 Then, there exists an absolute constant $C_r$ such that, with probability at least $1 - (T^2H^2)^{-1}$, we have 
 \#\label{eq:kucb_regret}
 \mathrm{Regret}(T) \leq  \begin{cases}
	C_r \cdot H^2 \cdot \sqrt{ \gamma^3 T} \cdot \log(\gamma TH) & \textrm{$\gamma$-finite spectrum},    \\
	C_r \cdot H^2 \cdot \sqrt{ (\log T)^{3/\gamma} \cdot  T} \cdot    \log( TH)  & \textrm{$\gamma$-exponential decay},
	\\
	C_r \cdot H^2 \cdot T^{\kappa^* +\xi^* + 1/ 2 }\cdot \bigl[ \log(TH) \bigr ]^{3/2}  & \textrm{$\gamma$-polynomial decay} .
 \end{cases} 
 \#
 \end{corollary}

Corollary  \ref{cor:main} asserts that when  $\beta$ is chosen  properly according to the eigenvalue decay property of $\cH$, \KUCB~incurs  a sublinear regret
under all of the three cases specified in Assumption \ref{assume:decay}.
Note that 
the   linear MDP \citep{jin2019provably}  satisfies the $\gamma$-finite spectrum condition and \KUCB~recovers the LSVI-UCB algorithm studied in \cite{jin2019provably} when restricted to this setting. 
Moreover, 
our $\tilde \cO(H^2 \cdot \sqrt{\gamma^3 T})$ also matches the regret bound in \cite{jin2019provably}. 
In addition, under the  $\gamma$-exponential eigenvalue decay condition, 
as we will show in Appendix~\S\ref{sec:aux_lem}, the log-covering number and the effective dimension are bounded by 
 $(\log T)^{1+2/\gamma}$ and $(\log T )^{1+ 1/\gamma}$, respectively. 
Plugging these facts into \eqref{eq:kernel_regret_simple}, we obtain the sublinear regret in \eqref{eq:kucb_regret}. 
As a concrete example, for the squared exponential kernel, we obtain an $\tilde\cO( H^2 \cdot (\log T)^{1+ 1.5 d} \cdot \sqrt{T})$ regret, where $d$ is the input dimension.
This such a regret   is $(\log T)^{d/2}$ worse than that in \cite{srinivas2009gaussian} for kernel contextual bandits, which is due to bounding the  log-covering number. 
Furthermore, 
for the case of  $\gamma$-polynomial decay, 
since the eigenvalues decay to zero rather slowly, we fail to obtain a $\sqrt{T}$-regret  and only obtain a   sublinear  regret in \eqref{eq:kucb_regret}. 
As we will show in the proof, the log-covering number and the  effective dimension are $\tilde \cO(T^{2\kappa^*})$ and  $\tilde \cO( T^{2\xi^*})$, respectively, which,  combined with \eqref{eq:kernel_regret_simple}, yield   the regret bound in \eqref{eq:kucb_regret}. 
 As a concrete example, consider  the Mat\'ern kernel with parameter $\nu > 0$ where we have $\gamma = 1 + 2 \nu  / d$ and $\tau = 0$.  
 In this case, when $\nu$ is sufficiently large such that $T^{2\xi^* -1/2} = o(1) $, we have 
 $$
 \xi^* = \frac{d(d+1) }{2[ 2\nu + d(d+1) ] }, \qquad \kappa^* =\max \Bigl \{ \xi^*,  \frac{3}{d-1}, \frac{2}{d-1} \Bigr \},
 $$
 which implies that \KUCB~achieves an  $\tilde \cO( H^2 \cdot T^{2\xi^* + 1/2} )$ regret when $d$ is large.
 Such a regret bound matches that in \cite{srinivas2009gaussian} for the bandit setting. 
See Appendix~\S\ref{sec:proof_corollary} for details.
 
 Furthermore, similarly to the discussion in Section 3.1 of \cite{jin2018q}, the regret bound in  \eqref{eq:kucb_regret} directly translates to an upper bound on the sample complexity   as follows. 
 When the initial state is fixed for all episodes, 
 for any fixed $\epsilon > 0$,
 with at least a constant probability, 
  \KUCB~returns a policy $\pi$ satisfying $V_1^\star (x_1) - V_1^\pi (x_1) \leq \epsilon $ using $\cO( H^4 B_T^2 \cdot \Gamma_K(T, \lambda)/ \epsilon^2)$ samples. 
  Specifically, for the three cases considered in Assumption \ref{assume:decay}, such a sample complexity guarantee reduces to 
  \$
   \tilde \cO\bigl ( H^4 \cdot \gamma^3   / \epsilon^2 \bigr ), 
   \qquad \tilde \cO \big ( H^4\cdot (\log T)^{2+3/\gamma } / \epsilon ^2 \big  ), \qquad \tilde \cO \bigl ( H^4 \cdot T^{2(\kappa^* + \xi^* )} / \epsilon^2 \bigr ), 
  \$
  respectively. 
Moreover, similar to \cite{jin2019provably}, our analysis can also be extended to the 
misspecified setting where $\inf_{f \in \cQ^\star} \| f - \cT_h^\star Q \|_{ \infty}  
\leq \texttt{err}_{\mathrm{mis} } $ for all $Q \colon \cZ \rightarrow [0, H]$. 
Here $\texttt{err}_{\mathrm{mis} } $ is the model misspecification error.
Under this setting, \KUCB~will suffer from  an extra  $\texttt{err}_{\mathrm{mis} } 
\cdot T H $ regret.  The analysis for the misspecified setting is similar to that 
for the neural network setting that will be  presented in  \S\ref{sec:nn_ntk_thoery}. 

\subsection{Regret of \NUCB} \label{sec:nn_ntk_thoery}

In this section, we establish the  regret of \NUCB.  Throughout this subsection, 
we  let $\cH$ denote the RKHS whose kernel function is $K_{\textrm{ntk}}$ defined 
in \eqref{eq:define_ntk}.  Recall that we regard $\cZ = \cS \times \cA$ as a 
subset of the unit sphere $\SSS^{d-1} =  \{ z \in \RR^d \colon \| z \|_2 =  1\}$.

Let $(b^{(0)}, W^{(0)})$ be the initial value of the network weights obtained 
via the symmetric initialization scheme introduced in \S\ref{sec:dnn}. 
Conditioning on the randomness of the initialization, we define a finite-rank 
kernel, $K_m \colon \cZ \times \cZ \rightarrow \RR$, by letting 
$K_{m}(z, z') = \la \nabla_{W} f(z; b^{(0)}, W^{(0)}) , \nabla_{W} 
f(z' ; b^{(0)}, W^{(0)})  \ra$.  Notice that the rank of $K_m$ is $md$, 
where $m$ is much larger than $T$ and $H$ and is allowed to increase to infinity. 
Moreover, with a slight abuse of notation, we  define 
\#  \label{eq:random_feature_class} 
\cQ^\star  = \biggl \{  f_{\alpha} (z) =     \int _{\RR^d}  \act' ( w^\top z)\cdot z ^\top \alpha (w)  ~ \ud p_0(w)    \colon  \alpha\colon  \RR^d \rightarrow \RR^d,    \| \alpha  \|_{2, \infty}   \leq R_{Q} H  /\sqrt{d}    \biggr \},
\#   
where $R_{Q} $ is a positive number,  $p_0$ is the density of $N(0, I_d /d)$, and  we define $\| \alpha\|_{2, \infty} = \sup_{w } \| \alpha(w) \|_2$. 
That is, 
$\cQ^\star$ consists of functions that can be expressed as infinite number of random features. 
As shown in Lemma C.1 of \cite{gao2019convergence}, $\cQ^\star $ is a dense  subset of the RKHS $\cH$. 
Thus, when $R_{Q} $ is sufficiently large, $\cQ^\star$ in \eqref{eq:random_feature_class} is an expressive function class.  We impose the following condition on~$\cQ^\star$.

\begin{assumption} \label{assume:opt_closure2}
	We assume that for any $h \in [H ]$ and  any $Q \colon \cS \times \cA \rightarrow [0,H] $, we have $\TT_h^\star Q  \in \cQ^\star$. 
	 \end{assumption}
Assumption \ref{assume:opt_closure2} is in the same vein as 
Assumption \ref{assume:opt_closure}. Here we focus on $\cQ^*$ instead 
of an RKHS norm ball of NTK solely for technical convenience.  Since  
functions of the form in \eqref{eq:random_feature_class} are dense in 
$\cH$, Assumptions  \ref{assume:opt_closure2} and \ref{assume:opt_closure} 
are very similar. 

To characterize the value function class associated with \NUCB,  
for any discrete set $\cD \subseteq \cZ $, in a manner akin to  
\eqref{eq:lambda_mat}, we define 
\$ 
\overline \Lambda_{\cD} = \lambda \cdot I_{2md} + \sum_{z \in \cD} \varphi(z; W^{(0) }) \varphi(z; W^{(0)}) ^\top, 
\$
where $\varphi (\cdot; W^{(0)} )$ is the neural tangent feature 
defined in \eqref{eq:gradient_feature}.  With a slight abuse of 
notation, for any $R, B> 0$, we let $\cQ_{\UCB} (h, R, B )$ denote 
that class of functions that takes the following form:
 \#\label{eq:define_ntk_q_func}
 Q (z)  = &\min \bigl \{ \bigl \la \varphi(z; W^{(0) }) , W  \ra  
  + \beta \cdot \bigl   [   \varphi(z; W^{(0)}) ^\top (\overline \Lambda_{\cD} )^{-1}  \varphi(z; W^{(0)})  \bigr ]^{1/2}, H - h + 1 \bigr \}^+,
 \#
where $W \in \RR^{2md}$ satisfies $\| W   \|_2 \leq R$, $\beta \in [0, B]$, 
and $\cD$ has cardinality no more than $T$.  Intuitively, when both 
$R$ and $B$ are   sufficiently large,  $\cQ_{\UCB}(h, R, B)$ contains 
the counterpart of the neural-network-based value function $Q_h^t$ 
that is based on neural tangent features.  Moreover, when $m$ is 
sufficiently large, it is expected that $Q_h^t$ is well approximated 
by functions in  $\cQ_{\UCB}(h, R, B)$ where the approximation error 
decays with $m$.  It is worth noting the class of linear functions of 
$\varphi (\cdot ; W^{(0)})$ forms an RKHS with kernel $K_m$ in \eqref{eq:emp_kernel}. 
Any function $g$ in this class can be written as $g (\cdot ) =
\la \varphi (\cdot ; W^{(0)}), W_g \ra $ for some $W_g \in \RR^{2md}$. 
Moreover, the RKHS norm of $f$ is given by  $\| W_g\|_2$.  Thus, 
  $\cQ_{\UCB} (h, R, B)$ defined above coincides with   the counterpart  defined in \eqref{eq:main_text_func_class} with  the kernel function being  $K_m$. 
  We set $R_T = H  \sqrt{2 T / \lambda} $ and let $N_{\infty} (\epsilon; h, B)$ denote the $\epsilon$-covering number of $\cQ_{\UCB} (h, R_T, B)$ with respect to the $\ell_{\infty}$-norm on $\cZ$.

We can now present a general regret bound for  \NUCB. 
\begin{theorem}
	\label{thm:neural}
		Under Assumptions \ref{assume:opt_closure2}, 
		We also assume that $m$ is sufficiently large such that 
		$
		m = \Omega ( T^{13 } H^{14} \cdot (\log m)^3).
		$
In   Algorithm \ref{algo:neural}, we let  $\lambda  $ be a sufficiently large constant and let $\beta = B_T$  which satisfies  inequality 
\#\label{eq:equation_B_T_nn}
16  \cdot \Gamma _{K_m } (T,  \lambda ) + 16   \cdot \log N_{\infty} 
(\epsilon^*, h+1 , B_T) + 32  \cdot \log (2TH) 
+  4   R_{Q} ^2   \cdot  (1 +  \lambda / d)    \leq (B_T / H)^2,
\#
for all $h \in [H]$.  Here $\epsilon^* = H/ T$  and $\Gamma_{K_m}(T, \lambda)$ 
is the maximal information gain defined for kernel $K_m$.  In addition, 
for the neural network in   \eqref{eq:two_layer_nn},  we assume that the 
activation function $\act $ is  $C_{\act}$-smooth; i.e., its derivative 
$\act'$  is $C_{\act}$-Lipschitz, and $m$ is sufficiently large such that 
\#\label{eq:m_is_large}
m =  \Omega \big ( \beta ^{12} \cdot T^{13} \cdot H^{14} \cdot (\log m)^3 \bigr ) .
\#
Then, with probability at least $1 - (T^2H^2)^{-1} $, we have 
\#\label{eq:regret_neural_final}
\textrm{Regret}(T) =  5 \beta H \cdot \sqrt{T\cdot \Gamma_{K_m}(T, \lambda)} 
+ 10  \beta T H \cdot \iota,
\#
where we define $\logt =T^{7/12} \cdot H^{1/6} \cdot      
m^{-1/12} \cdot (\log m)^{1/4}.$
\end{theorem}

This theorem shows that, when $m$ is sufficiently large, \NUCB~enjoys a 
similar regret bound as \KUCB.  Specifically, the choice of $\beta$ in 
\eqref{eq:equation_B_T_nn} is similar to that in  \eqref{eq:equation_B_T} 
for the kernel $K_m$.  Here we set $\lambda$ to be  an absolute  constant 
as  $\sup_{z} K_m(z, z) \leq 1 $ no longer holds.  In addition, here we 
assume that $\act' $ is $C_{\textrm{act}}$-Lipschitz on $\RR$, which can 
be relaxed to only assuming $\act'$ is Lipschitz continous on  a bounded 
interval of $\RR$ that contains $w^\top z$ with high probability, where 
$w$ is drawn from the initial distribution of $W_j$, $j\in [m]$. 

Comparing \eqref{eq:kucb_regret} and \eqref{eq:regret_neural_final}, 
we observe that, when $m$ is sufficiently large,  \NUCB~can be viewed 
as a misspecified version of \KUCB~for the RKHS with kernel $K_m$, 
where the model misspecification error is   $\texttt{err}_{\mathrm{mis}} 
= 10\beta \cdot \logt$.  Specifically, the first term in 
\eqref{eq:regret_neural_final} is the same as that in \eqref{eq:kucb_regret}, 
where the choice of $\beta $ and $\Gamma_{K_m}(T, \lambda)$ 
reflects  the intrinsic complexity of $K_{m}$. 
The second term is equal to $\texttt{err}_{\mathrm{mis}} \cdot TH$, which 
arises from the approximation of neural-network-based  value functions by 
functions in $\cQ_{\UCB} (h, R_T, B_T)$, which are constructed using  the 
neural tangent feature  $\varphi(\cdot; W^{(0)})$.
Moreover, when $\beta$ is bounded by a polynomial of $TH$, to make 
$\texttt{err}_{\mathrm{mis}} \cdot T H $ be negligible, it suffices 
to let $m$ be a polynomial of $TH$. That is, when the network width 
is a polynomial of the total number of steps, \NUCB~achieves the 
same performance as \KUCB. 

Neglecting the constants and logarithmic terms in \eqref{eq:equation_B_T_nn}, 
we can simplify the regret bound in \eqref{eq:regret_neural_final} as follows: 
\$
\mathrm{Regret}(T) =\tilde  \cO \Bigl ( H^2 \cdot  \Bigl [\Gamma_{K_m} 
(T, \lambda)  + \max_{h\in [H]} \sqrt{\Gamma_{K_m} (T, \lambda ) \cdot 
\log N_{\infty} (\epsilon^*, h, B_T) } \Bigr ]  \cdot \sqrt{T} + 
\texttt{err}_{\mathrm{mis}} \cdot T \Bigr ),
\$
which  depends on the intrinsic complexity of $K_m$ through both the 
effective dimension $\Gamma_{K_m} (T, \lambda)$ and the log-covering 
number $\log N_{\infty} (\epsilon^*, h, B_T) $.  To obtain a more 
concrete regret bound, in the following, we impose an assumption on 
the spectral structure of $K_m$.
 
\begin{assumption} [Eigenvalue Decay of  the Empirical NTK]
\label{assume:neural_decay} 
Conditioning on the randomness of $(b^{(0)}, W^{(0)})$, 
let $K_m$ be the kernel induced by the neural tangent feature  
$\varphi (\cdot; W^{(0)} )$ .
Let $T_{K_m}$ be the integral operator induced by $K_m$ and 
Lebesgue measure on $\cZ$ and   let  $\{ \sigma_j \}_{j \geq 1}$ and   $\{ \psi_j \}_{j\geq 1}$ be  its  eigenvalues and eigenvectors, respectively. 
	We assume that  $\{ \sigma_j \}_{j \geq 1}$ and   $\{ \psi_j \}_{j\geq 1}$ satisfy either one of the three  decay conditions specified in Assumption \ref{assume:decay}.
	Here we assume the constants 
 $C_{1}, C_2, C_{\psi}$, $\gamma$, and $\tau$ do  not depend on $m$.
\end{assumption} 
Here we essentially assume  that $K_m$ satisfies  Assumption \ref{assume:decay}. 
Since $K_m$ depends on the initial network weights, which are random, 
this assumption should be better understood in a limiting sense. 
Specifically, as $m$ goes to infinity, $K_m$ converges to $K_{\textrm{ntk}}$, which is determined by both the activation function and the distribution of the initial network weights. 
Thus, if the RKHS with kernel $K_{\textrm{ntk}}$ satisfies 
Assumption \ref{assume:decay}, when $m$ is sufficiently large, 
it is reasonable to expect that such a condition also holds for $K_m$. 
Due to space limitations, we defer concrete examples of 
$K_{\textrm{ntk}}$ satisfying Assumption \ref{assume:decay} to 
Appendix~\S\ref{sec:ntk_examples}.

We now characterize the performance of \NUCB~for each  case separately. 

\begin{corollary}
 \label{col:neural}
	Under Assumptions \ref{assume:opt_closure2} and \ref{assume:neural_decay}, we assume the activation function is $C_{\act}$-smooth and the number of neurons of the neural network satisfies \eqref{eq:m_is_large}. 
	Besides, in Algorithm \ref{algo:neural} we  let  $\lambda  $ be a sufficiently large constant and set  $\beta = B_T$ 
	 as in \eqref{eq:set_BT}.
	 Then 
 exists an absolute constant $C_r$ such that, with probability at least $1 - (T^2H^2)^{-1}$, we have 
	\#\label{eq:kucb_regret}
	\mathrm{Regret}(T) \leq  \begin{cases}
		C_r \cdot H^2 \cdot \sqrt{ \gamma^3 T} \cdot \log(\gamma TH) + 10 \beta TH \cdot \logt    & \textrm{$\gamma$-finite spectrum},    \\
		C_r \cdot H^2 \cdot \sqrt{ (\log T)^{3/\gamma} \cdot  T} \cdot    \log( TH) + 10 \beta T H \cdot \logt  & \textrm{$\gamma$-exponential decay},
		\\
		C_r \cdot H^2 \cdot T^{\kappa^* +\xi^* + 1/ 2 }\cdot \bigl[ \log(TH) \bigr ]^{3/2}  + 10 \beta T H \cdot \logt  & \textrm{$\gamma$-polynomial decay} ,
		\end{cases}
	\#
	where we define $  \logt =T^{7/12} \cdot H^{1/6} \cdot      m^{-1/12} \cdot (\log m)^{1/4}.$
\end{corollary}

Corollary \ref{col:neural} is parallel to Corollary \ref{cor:main}, with an additional misspecification error $10\beta TH \cdot \logt$. 
It remains to see whether there exist concrete neural networks that induce NTKs satisfying each eigenvalue decay condition. 
As we will show in Appendix~\S\ref{sec:ntk_examples}, a neural network with  
a quadratic activation function induces an NTK with a finite spectrum,  
while the sine activation function and the polynomials of ReLU activation
functions induce NTKs that satisfy the exponential and polynomial eigenvalue 
decay conditions, respectively.  Corollary \ref{col:neural} can be directly 
applied to  these concrete examples to obtain sublinear regret bounds.

%% file: proof_thm.tex
\section{Proofs of the Main Results} 
 
In this section, we provide the proofs of Theorems \ref{thm:main} and \ref{thm:neural}. 
The proofs of the supporting lemmas and auxiliary results are deferred to the appendix. 

\subsection{Proof of Theorem \ref{thm:main}} 
\begin{proof}
For simplicity of presentation, we define the temporal-difference (TD) error as 
\#\label{eq:td_error}
 \delta^t_h(x,a) =r_h (x,a) + (\mathbb{P}_{h}V^{t}_{h+1})(x,a)  - Q^{t}_h (x,a) ,\qquad \forall (x,a) \in \cS \times \cA. 
\#
Here $\delta_h^t$ is a function on $\cS\times \cA$ for all $h \in [H]$ and $t \in [T]$.
Note that $V_{h}^t(\cdot) = \max_{a \in \cA} Q_h^t(\cdot, a)$. 
Intuitively, $\{\delta_h^t \}_{h\in [H]}$ quantifies the how far  
the $\{ Q_h^t\}_{h\in [H]}$ are from satisfying the Bellman optimality 
equation in \eqref{eq:opt_bellman}.  Next, recall that $\pi^t$ is the 
policy executed in the $t$-th episode, which generates a trajectory 
$\{(x_h^t, a_h^t) \}_{h \in [H]}$.  For any $h \in [H]$ and $t \in [T]$, 
we  further define $\zeta_{t,h}^1$,  $\zeta_{t,h}^2 \in \RR$~as 
\#
\zeta_{t,h}^1 & = \bigl [ V_h^t(x_h^t) - V_{h}^{\pi^t} (x_h^t) \bigr ]  - \bigl [ Q_h^t(x_h^t , a_h^t) - Q_{h}^{\pi^t} (x_h^t, a_h^t) \bigr ], \label{eq:define_mtg1} \\
\zeta_{t,h}^2 & = \bigl [ ( \PP_h V_{h+1}^t ) (x_h^t , a_h^t) - (\PP_h V_{h+1}^{\pi^t} ) (x_h^t, a_h^t) \bigr ] - \bigl [ V_{h+1}^t(x_{h+1}^t) - V_{h+1}^{\pi^t} (x_{h+1}^t) \bigr ] \label{eq:define_mtg2}. 
\#
By definition, $\zeta_{t,h}^1$ and $\zeta_{t,h}^2$ capture two sources 
of randomness---the randomness of choosing an action $a_h^t \sim 
\pi_h^t(\cdot \given x_h^t)$ and that of drawing the next state 
$x_{h+1}^t$ from $\PP_h(\cdot \given x_h^t, a_h^t)$, respectively.
As we will see in Appendix~\S\ref{proof:lemma_bound_mtg}, 
$\{\zeta_{t,h}^1, \zeta_{t,h}^2 \}$ form a bounded martingale 
difference sequence with respect to a properly chosen filtration, 
which enables us to bound their total sum via the Azuma-Hoeffding 
inequality \citep{azuma1967weighted}. 

To establish an upper bound on the regret, the following lemma first 
decomposes the regret into three parts using the notation defined above. 
Similar regret decomposition results also appear in \cite{cai2019provably, 
efroni2020optimistic}.   

\begin{lemma}[Regret Decomposition]\label{lemma:regret_decomp}
The temporal-difference error is the mapping $\delta_{h}^t \colon \cS 
\times \cA \rightarrow$ defined in \eqref{eq:td_error} for all 
$(t,h) \in [T]\times [H]$.  We can thus write the regret as 
\# \label{eq:regret_decomp}
\text{Regret}(T) &=  \underbrace{ \sum_{t=1}^T \sum_{h=1}^H\bigl 
[ \EE_{\pi^{\star}}[ \delta^t_h(x_h,a_h)\,|\, x_1 = x^t_1] 
-  \delta^t_h(x^t_h,a^t_h)\bigr]}_{\dr (i)} + \underbrace{  
\sum_{t=1}^T \sum_{h=1}^H ( \zeta_{t,h} ^1 + \zeta_{t, h }^2 )}_{\dr (ii)}   \notag \\
& \qquad 	\underbrace{\sum_{t=1}^T\sum_{h=1}^H \EE_{\pi^{\star}} \bigl[  \bigl \la Q^{t}_h(x_h,\cdot), \pi^{\star}_h(\cdot\,|\,x_h) - \pi^t_h(\cdot\,|\,x_h)  \bigr \ra _{\cA} \,\big|\, x_1 = x^t_1 \bigr]}_{\dr (iii)},
\#
where   $\zeta_{t,h}^1$ and $\zeta_{t,h}^2$ are defined in \eqref{eq:define_mtg1} 
and \eqref{eq:define_mtg2}, respectively. 
\end{lemma}
\begin{proof}
	See Appendix~\S\ref{sec:proof_regret_decomp} for a detailed proof.
\end{proof}

Returning to the main proof, notice that $\pi_h ^t $ 
is the greedy policy with respect to $Q_h^t$ for all $(t, h) \in [T]\times [H]$.
We have 
\$
\bigl \la Q_h^t (x_h, \cdot ), \pi_h^\star (\cdot \given x_h) - \pi_h ^t  (\cdot \given x_h) \bigr \ra _{\cA} =  \bigl \la Q_h^t (x_h, \cdot ), \pi_h^\star (\cdot \given x_h)  \bigr \ra_{\cA} - \max_{a \in \cA} Q_h^t(x_h, a) \leq 0,
\$
for all $x_h \in \cS$.  Thus, $\textrm{Term~(iii)}$ in \eqref{eq:regret_decomp}  
is non-positive.  Then, by   Lemma \ref{lemma:regret_decomp}, we can upper bound 
the regret by
\# \label{eq:main11}
	\mathrm{Regret}(T) &\leq \underbrace{ \biggl \{  \sum_{t=1}^T \sum_{h=1}^H\bigl [  \EE_{\pi^{\star}}[ \delta^t_h(x_h,a_h)\,|\, x_1 = x^t_1] -  \delta^t_h(x^t_h,a^t_h)\bigr ]   \bigg\}}_{\dr\textrm{  (i)}} +   \underbrace{ \biggl[  \sum_{t=1}^T \sum_{h=1}^H ( \zeta_{t,h} ^1 + \zeta_{t, h }^2 ) \bigg] }_{ \dr (ii)}.
 \#
For Term (i), since we do not observe trajectories from $\pi^*$, which is 
unknown, it appears that  $\EE_{\pi^*} [ \delta_h^t(x_h, a_h) \given x_1 = x_1^t]$ 
cannot be estimated.  Fortunately, however, by adding the bonus term in 
Algorithm~\ref{algo:kucb}, we ensure that the temporal-difference error 
$\delta_h^t$ is a non-positive function, as shown in the following lemma.  
  
 \begin{lemma}[Optimism]  \label{lemma:ucb_kernel1}
Let $\lambda=1 + 1/ T$ and $\beta =B_T$   in   Algorithm \ref{algo:kucb}, 
where $B_T $  satisfies \eqref{eq:equation_B_T}.
Under Assumptions \ref{assume:opt_closure},  with probability 
at least $1-   ( 2 T^2 H^2)^{-1} $, we have that the following
holds for all $(t,h)\in[T]\times[H]$ and $(x,a)\in\cS\times\cA$:
 	\$
 	-2\beta \cdot  b_h^t(x, a)   \le \delta^{t}_h(x,a) \le 0. 
 	\$
 	\end{lemma}
\begin{proof}
See Appendix~\S\ref{sec:proof_ucb_kernel1} for a detailed proof. 
\end{proof}

Applying    Lemma \ref{lemma:ucb_kernel1} to Term (i) in \eqref{eq:main11}, 
we obtain that 
\#\label{eq:main13}
\textrm{Term~(i)} \leq \biggl [   \sum_{t=1}^T \sum_{h=1}^H  -  \delta^t_h(x^t_h,a^t_h)  \bigg] \leq 2 \beta \cdot \biggl [   \sum_{t=1}^T \sum_{h=1}^H     b^t_h(x^t_h,a^t_h)  \bigg] 
\#
holds with probability at least $1 -  ( 2 T^2 H^2)^{-1}$,
where $\beta$ is equal to $B_T$ as specified in \eqref{eq:equation_B_T}. 

Finally, it remains to bound 
the sum of bonus terms in \eqref{eq:main13}. 
As we show in \eqref{eq:new_bonus}, using the feature representation  
of $\cH$,  we can write each $b^t_h(x^t_h,a^t_h) $ as 
\$
b^t_h(x^t_h,a^t_h)   = \bigl [ \phi (x^t_h,a^t_h ) ^\top (  \Lambda _h^t )^{-1}   \phi (x^t_h,a^t_h ) \bigr ]^{1/2}, 
\$
where $\Lambda_h^t = \lambda \cdot I_{\cH} + \sum_{\tau =1}^{t-1}  
\phi( x_h^t, a_h^t) \phi(x_h^t, a_h^t)^\top  $ is a self-adjoint 
and positive-definite operator on $\cH$   and $\cI_{\cH}$ is the 
identity mapping on $\cH$.  Thus, combining the Cauchy-Schwarz 
inequality and Lemma \ref{lemma:telescope}, we have, for any 
$h \in [H]$, with probability at least $1 -  ( 2 T^2 H^2)^{-1}$
the following:
\#\label{eq:main15}
\textrm{Term~(i)} &    \leq 2 \beta \cdot \sqrt{T} \cdot \sum_{h=1}^H  \bigg[   \sum_{t=1}^T \phi (x^t_h,a^t_h ) ^\top (  \Lambda _h^t )^{-1}   \phi (x^t_h,a^t_h ) \biggr ] ^{1/2 } \notag\\
& \leq  2 \beta \cdot  \sum_{h=1}^H \bigl [ 2 T   \cdot \logdet ( I + K_h^T / \lambda  )   \bigr ]^{1/2 }  =  4 \beta H \cdot  \sqrt{ T  \cdot \Gamma_K(T, \lambda )},
\#
where $\Gamma_K(T, \lambda)$ is the maximal information gain  defined in \eqref{eq:maintext_infogain} with parameter $\lambda$.

It remains to bound Term (ii) in \eqref{eq:main11}, which is 
the purpose of the following lemma.  
\begin{lemma}\label{lemma:bound_mtg}
	For $ \zeta_{t,h}^1$ and $\zeta_{t,h}^2$  defined respectively in \eqref{eq:define_mtg1} and 
	 \eqref{eq:define_mtg2} and for any $\zeta  \in (0,1)$, with probability at least  $1- \zeta  $, we have 
		\$
		\sum_{t =1}^T \sum_{h=1}^H ( \zeta_{t,h}^1 +\zeta_{t,h}^2)  \leq \sqrt{ 16T H^3  \cdot \log (2/ \zeta )} .
		\$
		\end{lemma}

\begin{proof}
See Appendix~\S\ref{proof:lemma_bound_mtg} for a detailed proof. 
\end{proof}

Setting $\zeta = (2 T^2 H^2)^{-1}$ in Lemma \ref{lemma:bound_mtg} we obtain that 
\#\label{eq:main14}
 \textrm{Term (ii)} = \sum_{t=1}^T \sum_{h=1}^H ( \zeta_{t,h} ^1 + \zeta_{t, h }^2 )   \leq   \sqrt{16 TH^3\cdot\log( 4T^2 H^2 )} =    \sqrt{ 32    TH^3 \cdot \log( 2TH) }  
\# 
holds with probability at least $1 - (2TH)^{-1}$. 

Therefore, combining \eqref{eq:equation_B_T},
    \eqref{eq:main11},
    and  \eqref{eq:main14},    we conclude that, with probability at least $1-(T^2H^2)^{-1}$, the regret is bounded by  
\$ 
\mathrm{Regret}(T)  & \leq 4 \beta H \cdot  \sqrt{ T  \cdot \Gamma_K(T, \lambda )} +    \sqrt{ 32  TH^3  \cdot \log( 2TH) }  \leq 5 \beta H \cdot  \sqrt{ T  \cdot \Gamma_K(T, \lambda )},
\$
where the last inequality follows from the choice of $\beta = B_T$, which implies that 
\$
\beta \geq H \cdot \sqrt{ 16 \log ( TH)} \geq \sqrt{32  H\cdot \log (2TH)}.
\$
This concludes the proof  of Theorem \ref{thm:main}.
\end{proof}

\subsection{Proof of Theorem \ref{thm:neural}}

\begin{proof}
The proof of Theorem \ref{thm:neural} is similar to that of Theorem \ref{thm:main}. 
Recall that we let $\cZ$ denote $\cS\times \cA$ for simplicity. 
Recall also that for all $(t,h) \in [T]\times [H]$,
we define 
the temporal-difference (TD) error 
$\delta_h^t \colon \cZ \rightarrow \RR$ 
in \eqref{eq:td_error}  and define random variables $\zeta_{t,h}^1$ and $\zeta_{t,h}^2$ in \eqref{eq:define_mtg1} and \eqref{eq:define_mtg2}, respectively. 

Then, combining Lemma \ref{lemma:regret_decomp} and the fact that $\pi^t$  is the  greedy 
policy with respect to $\{ Q_{h}^t\}_{h\in [H]}$,
we bound the regret by 
\# \label{eq:neural11}
  \mathrm{Regret}(T) &\leq \underbrace{ \biggl \{  \sum_{t=1}^T \sum_{h=1}^H\bigl [  \EE_{\pi^{\star}}[ \delta^t_h(x_h,a_h)\,|\, x_1 = x^t_1] -  \delta^t_h(x^t_h,a^t_h)\bigr ]   \bigg\}}_{\dr\textrm{(i)}} +   \underbrace{ \biggl[  \sum_{t=1}^T \sum_{h=1}^H ( \zeta_{t,h} ^1 + \zeta_{t, h }^2 ) \bigg] }_{\dr\textrm{(ii)}}.
\#
Here, Term (ii) is a sum of a martingale difference sequence. 
By setting  
$\zeta  = (4T^2H^2)^{-1}$ in Lemma \ref{lemma:bound_mtg}, 
with probability at least 
$1 - (4T^2H^2)^{-1}$, we have  
\#\label{eq:neural12}
\textrm{Term (ii)} = \sum_{t=1}^T \sum_{h=1}^H ( \zeta_{t,h} ^1 + \zeta_{t, h }^2 )   \leq   \sqrt{16 TH^3\cdot\log( 8T^2H^2)} \leq H\cdot   \sqrt{ 32  T H\log( 2TH) }  .
\#

It remains to bound  Term (i) in \eqref{eq:neural11}.
To this end, we aim to establish a counterpart of Lemma \ref{lemma:ucb_kernel1} for neural  value functions, which shows that, by adding a bonus term $\beta \cdot b_h^t$, the TD error $\delta_h^t$ is always a non-positive function approximately. 
This implies that bounding Term (i) in \eqref{eq:neural11} reduces to
controlling $\sum_{t=1}^T \sum_{h = 1}^H  b_h^t(x_h^t, a_h^t)$.

Note that the bonus functions $b_h^t$ are constructed based on the 
neural tangent features $\varphi(\cdot; \hat W_h^t)$ and the matrix $\Lambda_h^t$.  
In order to relate $\sum_{t=1}^T \sum_{h = 1}^H  b_h^t(x_h^t, a_h^t)$ 
to the maximal information gain of the empirical NTK $K_m$, 
we define $\overline \Lambda_h^t$ and $\overline b_h^t$,
by analogy with $ \Lambda_h^t$ and $b_h^t$, as follows:
\$
\overline \Lambda_h^t = \lambda \cdot I_{2md} + \sum_{\tau = 1}^{t-1}  \varphi   ( x_h^\tau , a_h^\tau ; W^{(0)}  ) \varphi   ( x_h^\tau , a_h^\tau ; W^{(0)}  ) ^\top, \qquad \overline b_h^t(z) = \bigl [\varphi(z; W^{(0)} ) ^\top (\overline \Lambda_h^t ) ^{-1} \varphi(z; W^{(0)} )\bigr ] ^{1/2}  .
\$
 
 In the following lemma, we bound the TD error $\delta_h^t$ using $\overline b_h^t$ and show that $b_h^t$ and $\overline b_h^t$ are close in the   $\ell_\infty$-norm on $\cZ$ when $m$ is sufficiently large.

\begin{lemma}[Optimism]  \label{lemma:ucb_nn1}
  Let $\lambda $ be an absolute constant and let $\beta =  B_T$   in   
Algorithm \ref{algo:neural}, where $B_T $  satisfies \eqref{eq:equation_B_T_nn}. 
Under the  assumptions made    in Theorem \ref{thm:neural},   with probability at least $1-  (2T^2H^2)^{-1} - m^2 $, it holds for all $(t,h)\in[T]\times[H]$ and $(x,a)\in\cS\times\cA$ that  
 \#\label{eq:nn_optimism}  
 & - 5\beta \cdot \logt -2\beta \cdot \overline  b_h^t(x, a)   \le \delta^{t}_h(x,a) 
  \le  5 \beta \cdot   \logt, \qquad  
    \sup_{(x,a)\in \cZ} \bigl | b_h^t(x,a) - \overline b_h^t(x,a) \bigr | \leq 2  \logt,  
 \#
 where we define $\logt = T^{7/12} \cdot H^{1/12} \cdot m^{-1/12} \cdot (\log m)^{1/4}$. 
 \end{lemma}
\begin{proof}
See Appendix~\S\ref{sec:proof_ucb_nn1} for a detailed proof. 
\end{proof}

Applying    Lemma \ref{lemma:ucb_kernel1} to Term (i) in \eqref{eq:main11}, 
we obtain that 
\#\label{eq:neural13}
\textrm{Term~(i)} \leq \biggl [   \sum_{t=1}^T \sum_{h=1}^H  -  \delta^t_h(x^t_h,a^t_h)  \bigg]  +  5TH \cdot \logt \leq  2 \beta   \cdot  \biggl [   \sum_{t=1}^T \sum_{h=1}^H  \overline b^t_h(x^t_h,a^t_h)  \bigg] + 10 \beta  TH \cdot     \logt  
\#
holds with probability at least $1 - ( 2T^2H^2)^{-1} - m^{-2}$,
where $\beta = B_T$. 
Moreover, combining  the Cauchy-Schwarz inequality and Lemma \ref{lemma:telescope},  
we have 
\#\label{eq:neural14}
 \sum_{t=1}^T \sum_{h=1}^H  \overline b^t_h(x^t_h,a^t_h)  
&  \leq  \sqrt{T} \cdot \sum_{h=1}^H  \bigg[   \sum_{t=1}^T \varphi (x^t_h,a^t_h; W^{(0)}  ) ^\top (  \overline \Lambda _h^t )^{-1}   \varphi (x^t_h,a^t_h; W^{(0)}) \biggr ] ^{1/2 } \notag \\
& \leq    2  H \cdot  \sqrt{ T  \cdot \Gamma_{K_m} (T, \lambda )},
\#
where $\Gamma_K(T, \lambda)$ is the maximal information gain  defined in \eqref{eq:maintext_infogain}  for kernel $K_m$.

Notice that $  (2T^2H^2) ^{-1} +  m^{-2} + (4T^2H^2) ^{-1} \leq (T^2H^2) ^{-1} $.
Thus, combining \eqref{eq:neural11}, \eqref{eq:neural12},   \eqref{eq:neural13}, and \eqref{eq:neural14},  we obtain that 
\$
\mathrm{Regret}(T) & \leq 4 \beta H \cdot \sqrt{T \cdot \Gamma_{K_m} (T, \lambda )} +10 \beta  TH \cdot \logt+ H \cdot \sqrt{32 TH \log(2TH) } \notag \\
& \leq 5 \beta H \cdot \sqrt{T \cdot \Gamma_{K_m} (T, \lambda )} + 10 \beta  TH \cdot \logt
\$
holds with probability at least $1 - (2T^2H^2)^{-1}$.
Here the last inequality follows from the fact that 
\$
\beta \geq  H \cdot \sqrt{32 \log (TH) } \geq \sqrt{32 H \log (2TH)}.
\$
This concludes the proof of Theorem \ref{thm:neural}. 
 \end{proof}

%% file: conclu.tex

\section{Conclusions} \label{sec:discussion}

In this paper, we have presented an algorithmic framework for reinforcement 
learning with general function approximation.  Such a framework is based on
an optimistic least-squares value iteration algorithm that incorporates
an additional bonus term in the solution to a least-squares value estimation 
problem.  The bonus term promotes exploration.  When deploying this 
framework in the settings of kernel function  and overparameterized 
neural networks, respectively, we obtain two algorithms \KUCB~and \NUCB.
Both algorithms are provably efficient, both computationally and in terms 
of the number of samples.  Specifically, under the kernel and neural network 
settings respectively, \KUCB~ and \NUCB~ both achieve sublinear  regret,
$\tilde O(\delta_{\cF} H^2 \sqrt{T})$,  where $\delta _{\cF}$ is a  quantity that characterizes the intrinsic complexity of the function 
class $\cF$.  To the best of our knowledge, this is the first provably 
efficient reinforcement learning algorithm in the general settings of 
kernel and neural function approximations.

%% file: proof_lemmas.tex

\section{Neural Optimistic Least-Squares Value Iteration} \label{sec:nn_algo}

In this section, we provide the pseudocode for \NUCB, which was omitted 
in the main text for brevity.  We remark that the loss function $L_h^t$ 
in Line \ref{line:nn_opt} is given in \eqref{eq:nn_loss} and its global 
minimizer $\hat W_h^t$ can be efficiently obtained by first-order optimization 
methods. 

\begin{algorithm}[ht]
	\caption{Neural   Optimistic Least-Squares Value Iteration (\NUCB)}\label{algo:neural}
	\begin{algorithmic}[1]
		\STATE{\textbf{Input:} Parameters $\lambda$ and $\beta$.}
		\STATE{Initialize the network weights $(b^{(0)}, W^{(0)})$ via the symmetric initialization scheme.}
		\FOR{episode $t = 1, \ldots, T$}
		\STATE{Receive the initial state $x^t_1$.}
		\STATE{Set $V_{H+1}^t$ as the zero function.}
		\FOR{step $h = H, \ldots, 1$}
		\STATE Solve the neural network optimization problem $\hat W_h^t = \argmin_{W} L_h^t(W)$. \label{line:nn_opt} 
		\STATE Update $\Lambda_h^t = \Lambda_h^{t-1} + \varphi   ( x_h^{t-1}, a_h^{t-1} ;\hat W_h^{t}   )\varphi  ( x_h^{t-1}, a_h^{t-1} ;\hat W_h^{t}   )  ^\top  $.
		\STATE Obtain the bonus function $b_h^t$ defined in \eqref{eq:neural_bonus}.
		\STATE{Obtain value functions $$Q_h^t(\cdot, \cdot) \leftarrow \min \bigl \{ f\bigl (\cdot, \cdot; \hat W_h^t \bigr ) + \beta   \cdot b_h^t(\cdot, \cdot ) , H - h +1 \bigr \}^+ , \qquad V_h^t(\cdot) = \max_{a} Q_h^t(\cdot, a).$$} \label{line:ucb}
		\ENDFOR
		\FOR{step $h = 1, \ldots, H$}
		\STATE Take action $a^t_h \gets  \argmax_{a \in \cA } Q_h^t (x^t_h, a)$. 
		\STATE Observe the  reward $r_h(x_h^t, a_h^t)$ and the next state $x^t_{h+1}$.  
		\ENDFOR
		\ENDFOR
	\end{algorithmic}
\end{algorithm}

\section{Proofs of the Corollaries} \label{sec:proof_corollaries}

In this section, we prove Corollaries \ref{cor:main} and \ref{col:neural}, 
which establish the regret for \KUCB~and \NUCB~under each specific eigenvalue 
decay condition.  in Appendix~\S\ref{sec:ntk_examples}  we provide 
concrete examples of neural tangent kernels that satisfy  
Assumption \ref{assume:decay} and show how to apply 
Corollaries \ref{cor:main} and \ref{col:neural} to these examples.

\subsection{Proof of Corollary \ref{cor:main} } \label{sec:proof_corollary}
\begin{proof}
	To prove this corollary, it suffices to verify that for each eigenvalue decay condition specified in Assumption \ref{assume:decay}, $B_T$ defined in \eqref{eq:set_BT} satisfies the condition in \eqref{eq:equation_B_T}. 
	Recall that we set $\lambda = 1 + 1/ T$ in Algorithm \ref{algo:kucb} and denote $R_T = 2 H  \sqrt{   \Gamma_K(T, \lambda)}$, $\epsilon^* = H / T$. 
	Also recall that we let $N_{\infty}(\epsilon, h , B)$ denote the $\epsilon$-covering number of $\cQ_{\UCB} (h, R_T, B)$ with respect to the $\ell_{\infty}$-norm.
In the sequel, we  consider the three cases separately. 
	
	\vspace{5pt} 
{\noindent \bf Case (i): $\gamma$-Finite Spectrum.} 
When $\cH$   has at most  $\gamma$ nonzero eigenvalues, 
by Lemma \ref{lemma:effective_dim}, we have $\Gamma_K(T, \lambda) \leq C_K  \cdot  \gamma\log T$,
where $C_K$ is an absolute constant. 
Moreover, by Lemma \ref{lemma:covering_number_V}, for  any $h \in [H]$, we have 
\#\label{eq:bound_BT_linear1}
\log N_{\infty} (\epsilon^*, h, B_T) & \leq C _N \cdot \gamma \cdot \bigl \{ 1 + \log \bigl[   2 \sqrt{\Gamma(T, \lambda)} \cdot T \bigr ]  \bigr \} + C _N \cdot \gamma ^2 \cdot \bigl [ 1+ \log ( B_T\cdot  T / H )\big ] \notag \\
& \leq 2 C _N \cdot \gamma^2 + C' \cdot \gamma \cdot \log ( \gamma T) + C_N \cdot \gamma ^2 \cdot \log (B_T \cdot T / H), 
\# 
where $C _N > 0$ is the  absolute constant given in Lemma \ref{lemma:covering_number_V} and $C'$ is an absolute constant that depends on $C_N $ and $C_K$.  
Thus, setting $B_T = C_b \cdot \gamma H \cdot \sqrt{ \log (d TH)}$ in  \eqref{eq:bound_BT_linear1}, the left-hand side (LHS) of \eqref{eq:equation_B_T} is bounded by 
\#\label{eq:bound_BT_linear2}
\mathrm{LHS~of~\eqref{eq:equation_B_T}} & \leq 
 8 C_K \cdot \gamma \log T +       16  C_N  \cdot \gamma ^2 +  8C' \cdot  \gamma \cdot \log (\gamma T) + \notag \\
 & \qquad  8 C_N \cdot \gamma^2 \cdot \log ( C_b \cdot \gamma T \cdot \sqrt{ \log (dTH )})  +  16   \cdot \log (    T H  )  + 22  + 2 R_Q ^2 \notag \\
 &\leq \gamma^2 \cdot \bigl [ \overline C_1  \cdot \log (\gamma TH) +  8 C _N \cdot \log (C_b  ) \big ], 
\# 
where $\overline  C_1$ is an absolute constant that depends on $C'$, $C_N$, $C_K$, and $R_{Q}$. 
Thus, setting $C_b$ as a sufficiently large constant, 
by \eqref{eq:bound_BT_linear2}, 
we have 
$$
\mathrm{LHS~of~\eqref{eq:equation_B_T}}  \leq C_b ^2 \cdot \gamma^2 \cdot \log (d TH) = ( B_T / H)^2   ,  
$$
which establishes \eqref{eq:equation_B_T} for the first case. 
Thus, applying Theorem \ref{thm:main} we obtain that 
\$
\mathrm{Regret} (T) \leq 8 B_T \cdot H \cdot \sqrt{T \cdot \Gamma _K(T, \lambda) } \leq C_{r,1} \cdot  H^2 \cdot \sqrt{ \gamma^3 T} \cdot \log (\gamma TH) = \tilde \cO(H^2 \sqrt{\gamma^3 T}) 
\$
holds with probability at least $1 - (T^2 H^2 )^{-1}$, 
where $C_{r,1}$ is an absolute constant and $\tilde \cO (\cdot )$ omits the logarithmic factor. 
Therefore, we conclude the first case.

\vspace{5pt}
	{\noindent \bf Case (ii): $\gamma$-Exponential Decay.} 
	For the second case, by Lemma \ref{lemma:effective_dim}  we have 
	\#\label{eq:bound_BT_exp1}
	\Gamma_K (T, \lambda ) \leq C_K \cdot ( \log T)^{1 + 1/ \gamma},
	\# where $C_K$ is an absolute constant. 
Thus, by the choice of $B_T$ in \eqref{eq:set_BT}, when $C_b $ is sufficiently large, it holds that 
$
R_T = 2 H   \sqrt{ \Gamma_K(T, \lambda) } \leq B_T. 
$
Then by Lemma \ref{lemma:covering_number_V} we have 
 \$
	\log N_{\infty} ( h, \epsilon^*, B_T) &  \leq C _N \cdot   \bigl [ 1 +  \log (  R_T/\epsilon^* )    \bigr ]^{1+ 1/ \gamma} + C _N   \cdot \bigl [  1+ \log ( B_T / \epsilon^* ) \bigr ] ^{1 + 2/ \gamma}   \\
	& \leq 2 C_N  \cdot  \bigl [1 +  \log ( B_T / \epsilon^* ) \bigr  ] ^{1 + 2/ \gamma}  = 2 C _N \cdot \big \{   1 +  \log \bigl [  C_b  T \cdot \sqrt{\log (TH)} \cdot (\log T)^{1/ \gamma }  \big] \bigr \}^{1 + 2/ \gamma } , \notag
	\$ 
	where the absolute constant $C_N $ is given by Lemma \ref{lemma:covering_number_V}. 
	By direct computation, there exists an absolute constant $\overline C_2$ such that 
	\# \label{eq:bound_BT_exp2} 
	\log N_{\infty} ( h, \epsilon^*, B_T)  \leq 2 C _N \cdot  \bigl  [ 1 + \log (C_b) + \overline C_2 \cdot \log T + 1/ 2 \cdot \log\log H \bigr ]^{1 + 2/ \gamma}.
	\#
	Thus, combining \eqref{eq:bound_BT_exp1} and \eqref{eq:bound_BT_exp2}, the left-hand side of \eqref{eq:equation_B_T} is bounded by 
	\#\label{eq:bound_BT_exp3}
	\mathrm{LHS~of~\eqref{eq:equation_B_T}} & \leq 8 C_K \cdot (\log T)^{1 + 1/ \gamma } + 16 C \cdot  \bigl  [ 1 + \log (C_b) + \overline C_2 \cdot \log T + 1/ 2 \cdot \log\log H \bigr ]^{1 + 2/ \gamma} \notag \\
	& \qquad + 16 \cdot \log (TH) + 22 + 2R_{Q} ^2 \notag \\
	& \leq \overline C_3 \cdot \bigl [  (\log T)^{1+2/\gamma } + (\log\log H)^{1+2/\gamma} + \log (C_b) \bigr ] ,
	\#
where $\overline C_3$ is an absolute constant that does not depend on $C_b$.
Thus, when $C_b$ is sufficiently large, 
\eqref{eq:bound_BT_exp3} implies that 
\$
\mathrm{LHS~of~\eqref{eq:equation_B_T}} & \leq \overline C_3 \cdot \bigl [  (\log T)^{1+2/\gamma } + (\log\log H)^{1+2/\gamma} + \log (C_b) \bigr ] \leq C_b^2 \cdot (\log T)^{2/ \gamma} \cdot \log(TH) = (B_T / H )^2.
\$
Thus, for the case of $\gamma$-exponential eigenvalue decay,  \eqref{eq:equation_B_T} holds true for $B_T$ defined in \eqref{eq:set_BT}. 

Finally, applying Theorem \ref{thm:main} and combining \eqref{eq:set_BT} and \eqref{eq:bound_BT_exp1}, we obtain that 
\$
\mathrm{Regret}(T) \leq C_{r,2} \cdot H^2 \cdot \log(TH) \cdot \sqrt{ (\log T)^{3/\gamma} \cdot T},
\$
where $C_{r,2}$ is an absolute constant.
Thus we conclude the second case.

\vspace{5pt}
{\noindent \bf Case (iii): $\gamma$-Polynomial Decay.} 
Finally, it remains to consider the last case where the eigenvalues satisfy the $\gamma$-polynomial decay condition. 
By Lemma \ref{lemma:effective_dim}, we have 
\#\label{eq:bound_BT_poly1}
\Gamma_K ( T, \lambda ) \leq  C_K \cdot T^{(d+1) / (\gamma + d)} \cdot \log T,
\#
where $C_K$ is an absolute constant. 
By direct computation, we have that 
\#\label{eq:bound_BT_poly2}
R_T / \epsilon^*  = 2 T \sqrt{ \Gamma (T, \lambda )} \leq 2 \sqrt{C_K} \cdot T ^{ (2d + \gamma + 1)/ (\gamma + d)} \cdot \log T.
\#
To simplify the notation, in the sequel, we let $\tilde \gamma $ denote $\gamma (1 -2 \tau)$. 
Moreover, we write  $\kappa^*$ defined in \eqref{eq:define_poly_kappa} equivalently as   $\kappa^* = \max \{ \kappa_1, \kappa_2, \kappa_3\}$ for notational simplicity, where  
\# \label{eq:bound_BT_kappa}
\kappa_1 = \frac{d+1}{2 (\gamma+d)} = \xi^*, \qquad \kappa_2 =  \frac{2d+ \gamma + 1}{(d+\gamma)\cdot(\tilde \gamma -1)},  \qquad \kappa_3 = \frac{2}{\tilde \gamma - 3}.
\#
With $B_T = C_b \cdot H   \log (TH ) \cdot T^{\kappa^*} $, it holds that 
\#\label{eq:bound_BT_poly4}
B_T / \epsilon^* = C_b \cdot  \log (TH)  \cdot T^{1+ \kappa^*}. 
\#
 Meanwhile, Lemma \ref{lemma:covering_number_V} implies that  
\#\label{eq:bound_BT_poly3}
&\log N_{\infty} ( h, \epsilon^*, B_T) \\
&\qquad   \leq C _N \cdot  ( 2 \sqrt{C_K}  ) ^{ 2 / (\tilde \gamma -1 )} \cdot \Bigl (T^{(2d + \gamma + 1)/ (\gamma + d) }\cdot \log T \Bigr )^{2/ (\tilde \gamma -1 )} \cdot \bigl [ 1 + \log ( R_T / \epsilon^*)] \notag \\
& \qquad\qquad  +   C _N   \cdot C_b ^{4/ (\tilde \gamma -1 )} \cdot T^{4(1+\kappa^*)/ (\tilde \gamma -1 )}\cdot \bigl [\log (TH)\bigr ]^{4/(\tilde \gamma -1 )}\cdot \bigl [  1+ \log ( B_T / \epsilon^* ) \bigr ] \notag \\
&\qquad  \leq \overline  C_4 \cdot T^{2 \kappa_2 } \cdot  (\log T )^{ 1 + 2 / (\tilde \gamma -1)} + \overline  C_4 \cdot C_b ^{4/ (\tilde \gamma -1 )}\cdot T^{4(1+\kappa^*) / (\tilde \gamma -1 )} \cdot \log (C_b\cdot TH) \cdot   \bigl [ \log (TH)\bigl ]^{ 4 / (\tilde \gamma - 1)  } , \notag 
\#
where $\overline C_4$ is an absolute constant that only depends on $C_N$ and $C_K$.
Here in the first inequality of \eqref{eq:bound_BT_poly3} we plug in \eqref{eq:bound_BT_poly2} and \eqref{eq:bound_BT_poly4}  while in the second  inequality  we utilize the definition of $\kappa_2$ in \eqref{eq:bound_BT_kappa} and  
\$
\log ( R_T / \epsilon^*) \asymp \log T, \qquad \log (B_T / \epsilon ^*) \asymp \log (C_b) + \log T + \log\log H \leq \log (C_b \cdot TH).  
\$
Moreover, since $\kappa < 1/2$, by the last equality in \eqref{eq:bound_BT_kappa},  it holds that $2/ (\tilde \gamma -1) <1/3$.
Since $\kappa^* \geq \kappa_3$, we have $4 (1 + \kappa^*) / (\tilde \gamma - 1) \leq 2 \kappa^*$. 
Moreover, when $C_b$ is sufficiently large,  we have that   $\log (C_b) < C_b^{1/3}$,
which implies that 
$
C_b^{4/ (\tilde \gamma - 1)} \cdot \log (C_b ) \leq C_b. 
$
Thus, \eqref{eq:bound_BT_poly3} can be further simplified into 
\#\label{eq:bound_BT_poly5}  
\log N_{\infty} ( h, \epsilon^*, B_T)  \leq \overline  C_4 \cdot T^{2 \kappa_2 } \cdot  (\log T )^{ 2} + \overline  C_4 \cdot C_b  \cdot T^{2\kappa^* } \cdot  \bigl [ \log (TH)\bigl ]^{ 2 } , 
\#
Finally, combining \eqref{eq:bound_BT_poly1} and \eqref{eq:bound_BT_poly5}, the left-hand side of \eqref{eq:equation_B_T} is bounded by 
\$
\mathrm{LHS~of~\eqref{eq:equation_B_T}} & \leq \overline C_5 \cdot T^{2\kappa_1} \cdot \log T +  \overline C_5 \cdot T^{2\kappa_2 }  \cdot  (\log T )^{ 2}  + \overline C_5 \cdot C_b \cdot T^{ 2\kappa^* } \cdot  \bigl [ \log (TH)\bigl ]^{ 2 }  \\
& \leq  C_b^2 \cdot \big [ \log (TH)\bigr ] ^2 \cdot T^{2\kappa^*} = (B_T/ H)^2 ,
\$
where $\overline C_5$ is an absolute constant and the last inequality holds when $C_b$ is sufficiently large. 
Thus, we establish \eqref{eq:equation_B_T}.
Combining \eqref{eq:kucb_regret} in Theorem \ref{thm:main}, \eqref{eq:set_BT},  and \eqref{eq:bound_BT_poly1}, the regret of \KUCB~under this case is bounded by 
\$
\mathrm{Regret} (T) \leq C_{r, 3} \cdot H^2 \cdot T^{\kappa ^* + \xi^* + 1/2} \cdot \big [ \log (TH)\bigr ]^{3/2} = \tilde \cO \bigl ( H^2 \cdot T^{\kappa ^* + \xi^* + 1/2} \bigr ),
\$  
where $C_{r,3}$ is an absolute constant and $\tilde \cO(\cdot)$ omits the logarithmic factor. Thus, we establish the last inequality in \eqref{eq:kucb_regret}. 
Therefore, we conclude the proof of Corollary \ref{cor:main}. 
\end{proof}

\subsection{Proof of Corollary \ref{col:neural}}

\begin{proof}
By Theorem \ref{thm:neural}, we have 
\#\label{eq:apply_thm_neural}
\mathrm{Regret}(T) = 5 \beta H \cdot \sqrt{T\cdot \Gamma_{K_m} (T, \lambda) } 
+ 10 \beta TH \cdot \logt,
\#
where $\beta = B_T$ satisfies \eqref{eq:equation_B_T_nn} and  $\logt =T^{7/12} \cdot H^{1/6} \cdot      m^{-1/12} \cdot (\log m)^{1/4}.$
When Assumption \ref{assume:neural_decay} holds, 
thanks to the similarity between \eqref{eq:equation_B_T} and \eqref{eq:equation_B_T_nn}, 
it can be similarly shown that 
$B_T$ defined in \eqref{eq:set_BT}  satisfies the inequality in  \eqref{eq:equation_B_T_nn} when $C_b$ is sufficiently large. 
Moreover, Lemma \ref{lemma:effective_dim} provides upper bounds on $\Gamma_{K_m} (T, \lambda)$ for all the three eigenvalue decay conditions. 
Finally, combining \eqref{eq:set_BT}, \eqref{eq:apply_thm_neural}, and Lemma \ref{lemma:effective_dim}, we conclude the proof of Corollary \ref{col:neural}.
\end{proof}

\subsection{Examples of Kernels Satisfying Assumption \ref{assume:decay}} \label{sec:ntk_examples} 

In the following, we introduce concrete  kernels  and neural tangent kernels  that satisfy Assumption \ref{assume:decay}. 
We consider each eigenvalue decay condition separately. 

\vspace{5pt} 
{\noindent \bf Case (i): $\gamma$-Finite Spectrum.} 
Consider the polynomial kernel $K(z,z') = (1 + \la z, z' \ra  )^n$ defined on the unit ball $\{ z\in \RR^d \colon \| z \|_2 \leq 1\}$, where $n$ is a fixed number. 
By direct computation, the kernel function can be written as 
\$
K(z,z') = \sum_{\alpha \colon  \| \alpha\|_1 \leq n} z^\alpha\cdot {z'}^\alpha,		
\$
where $\alpha = (\alpha_1, \ldots, \alpha_d) \in \NN^d$ is a multi-index and  $z^\alpha$ is a monomial with degree $\alpha$. 
It can be shown that all monomials in $\RR^d$ with degree no more than $n$ are linearly independent. 
Thus, the  dimension of such an RKHS is  ${n+d\choose d}$; i.e., it satisfies 
the $\gamma$-finite spectrum condition with $\gamma = {n+d \choose d}$.

Furthermore, for a finite-dimensional  NTK, we consider the quadratic 
activation function $\act(u) = u^2$. 
Note that we assume $\cZ= \SSS^{d-1}$ for the neural  network setting.
Moreover, in \eqref{eq:two_layer_nn}, instead of  sampling $W_j \sim N(0, I_d/d)$ for all $j \in [d]$, we draw $W_j$ uniformly over the unit sphere $\SSS^{d-1}$. 
Then it holds that $|W_j^\top z| \leq 1$ for all $j \in [2m] $ and $z\in \SSS^{d-1}$. 
Here we let the distribution be  $ \textrm{Unif}(\SSS^{d-1})$ in order to  ensure that the $\act'$ is Lipschitz continuous  
on $\{ W_j^\top z \colon z \in \SSS^{d-1}\} \subseteq [-1, 1]$ for any $W_j$ sampled from the initial distribution, which is required when utilizing  Proposition C.1 in \cite{gao2019convergence} in the proof of Lemma \ref{lemma:ucb_nn1}. 
Note that the covariance of $W_j$ is still $I_d/ d$.  Then  by \eqref{eq:define_ntk}, the NTK is given by 
\#\label{eq:ntk_sphere_poly}
K_{\textrm{ntk} } (z, z' ) = \EE_{w \sim \textrm{Unif}(\SSS^{d-1})
} [ 2 ( w ^\top z) \cdot 2 ( w^\top z' ) \cdot (z^\top z') ] = 4/d \cdot (z^\top z') ^2 , \qquad \forall  z,z' \in \SSS^{d-1}.
 \#
 Thus, $K_{\textrm{ntk}} (z, z')$ can be written as a univariate function of the inner product $\la z, z'\ra$. 
To characterize the spectral property $K_{\textrm{ntk}} $, we first introduce  
some background on spherical harmonic functions on $\SSS^{d-1}$,
which are closely related to inner product kernels on $\SSS^{d-1}\times \SSS^{d-1}$.

Let $\mu$ be the uniform  measure on $\SSS^{d-1}$. 
For any $j \geq 0$, let $\cY_j(d)$ be the    set of all homogeneous harmonics of degree $j$  on $\SSS^{d-1}$, which is a 
finite-dimensional subspace of $\cL^2_\mu(\SSS^{d-1})$, the space of square-integrable functions on   $\SSS^{d-1}$ with
respect to $\mu$. 
 It can be shown that the dimensionality  of  $\cY_j(d) $ is given by  $N(d, j)$, 
 which is defined as 
 \#\label{eq:define_ndj_number}
 N(d, j ) = \frac{(2j+ d-2) (d+j -3) !}{j! (d-2)! }.
 \#
 In addition, let $\{ Y_{j,\ell}\}_{ \ell \in [ N(d, j)]} $ 
be an orthonormal basis of $\cY_j (d)$, then $\{ Y_{j, \ell} \}_{ \ell \in [N(d,j)],  j \in \NN}$ form an orthonormal basis of $\cL^2_{\mu} ( \SSS^{d-1})$.
In the next lemma, we present 
 the  Funk-Hecke formula \citep[page 30]{muller2012analysis}, which relates spherical harmonics to inner product kernels.

\begin{lemma} [Funk-Hecke formula]
	\label{lemma:funk} 
	 Let $k\colon [-1, 1]\rightarrow \RR$ be a continuous function, which 
	gives rise to an inner product kernel $K(z,z') = k( \la z, z'\ra)$ on 
	$\SSS^{d-1} \times \SSS^{d-1}$.
	For any $\ell \geq 2$, let $|\SSS^{\ell -1}|$ be the Lebesgue measure of $\SSS^{\ell -1}$, which is given by $| \SSS^{\ell -1}| = 2 \pi^{\ell /2} / \Gamma(\ell /2)$, where $\Gamma (\cdot)$ is the Gamma function. 
	Moreover, for any $j \geq 0$, 
 let $Y_j \colon  \SSS^{d-1} \rightarrow \RR$ be any function in $\cY_j (d) $. Then for any $z \in \SSS ^{d-1}$, we have 
\#
\int_{\SSS^{d-1}} K(z, z') Y_j (z') ~\ud \mu( z') =    \bigg[  \frac{| \SSS^{d-2} | }{| \SSS^{d-1} | } \cdot \int_{ -1} ^1 k( u ) \cdot P_j (u; d)  \cdot (1-u^2) ^{(d-3)/2}~\ud u \biggr ]\cdot Y_j (z)  , 
\#
where $P_j (\cdot; d)$ is the $j$-th Legendre polynomial in dimension $d$, which is given by 
$$
P_j(u ; d) = \frac{ (-1/2)^{j} \cdot \Gamma(\frac{ d-1}{2} ) } { \Gamma (\frac{2j+ d-1}{2}  )} \cdot (1- u^2)^{(3-d) / 2}  \cdot \biggl ( \frac{\ud }{\ud u} \bigg)^{j} \bigl [  ( 1-u^2) ^{j+ (d-3)/2 }\big]. 
$$
	\end{lemma} 
 
Thus, by the Funk-Hecke formula, for any inner product kernel $K$, its integral operator $T_{K}\colon \cL^2_\mu(\SSS^{d-1} ) \rightarrow \cL^2_{\mu} (\SSS^{d-1} )$ has eigenvalues  
\#\label{eq:eigenvalue_inner_prod_ker} 
\varrho_j = \frac{| \SSS^{d-2} | }{| \SSS^{d-1} | } \cdot \int_{ -1} ^1 k(u) \cdot P_j (u; d)  \cdot (1-u^2) ^{(d-3)/2}~\ud u, \qquad \forall j \geq 1, 
\#
each with multiplicity $N(d, j)$. Moreover, for each eigenvalue $\varrho_j$, the corresponding eigenfunctions are spherical harmonics $\{ Y_{j,\ell}\}_{\ell \in [N(d, j)]}$. Furthermore, to compute the eigenvalues in \eqref{eq:eigenvalue_inner_prod_ker}, 
we can use Rodrigues' rule \citep[page 23]{muller2012analysis}, as follows. 

\begin{lemma}[Rodrigues' Rule] \label{lemma:rodrigues} 
	For any $j \geq 0$, 
	let $f \colon [-1, 1]\rightarrow \RR$ be any  $j$-th  continuously differentiable function. 
	Then we have 
	\$
	\int_{-1}^1 f(t) \cdot P_j (u; d) \cdot (1 - u^2 )^{(d-3) /2 } ~\ud u = R_j(d) \cdot \int_{-1}^1 f^{(j)} (u) \cdot (1 - u^2 ) ^{(2j + d- 3) / 2 } ~\ud t,
	\$
	where $f^{(j)}$ is the $j$-th order derivative of $f$ and $R_j(d) = 2^{-j} \cdot \Gamma( (d-1) / 2)  \cdot [  \Gamma( ( 2j + d -1) / 2) ] ^{-1}$ is the $j$-th Rodrigues constant. 
	\end{lemma} 

Now we consider the NTK given in \eqref{eq:ntk_sphere_poly}, which is the  inner product kernel induced by the univariate function $k_1(u) = 4/d \cdot u^2$. 
Note that $k_1^{(3)} $ is a zero function. 
Combining Lemma \ref{lemma:rodrigues} and \eqref{eq:eigenvalue_inner_prod_ker}, we observe that 
$\varrho_j = 0$ for all $j \geq 3$.  
In addition, by direct computation, we have that 
\$
\varrho_1 = R_1 (d) \cdot (8 /d ) \cdot \int_{-1}^1 u 
\cdot (1 - u^2 ) ^{( d -1 ) / 2 } ~\ud u = 0,  
\$
and $\varrho_0 , \varrho_2 > 0$. 
Thus, $K_{\textrm{ntk}}$ given in \eqref{eq:ntk_sphere_poly} has $N(d, 0) + N(d, 2) = d( d+1) / 2$ nonzero eigenvalues, each with value $\varrho_2$. 
This implies that the NTK induced by   neural networks with quadratic activation 
satisfies the $\gamma$-finite spectrum condition with 
$
\gamma = d( d+1) / 2.
$
For such a class of neural networks, Corollary \ref{col:neural} asserts that 
the regret of 
\NUCB~is 
$ \tilde \cO ( H^2  d^3 \cdot \sqrt{T}    +  \beta TH \cdot \logt   )    .
$

\vspace{5pt} 

{\noindent \bf Case (ii): $\gamma$-exponential Decay.} 
Now we consider the squared exponential kernel 
\# \label{eq:rbf_kernel}
K(z,z') = \exp( - \| z - z' \|_2^2 \cdot \sigma^{-2})   = k_2(\la z, z' \ra ),  \qquad \forall z, z' \in \SSS^{d-1}, 
\#
where $\sigma> 0$ is an absolute constant and we define $k_2 (u) = \exp[ - 2 \sigma^{-2} \cdot (1 - u)]$.
Note  that $d$ is regarded as a fixed number. 
Applying Lemmas \ref{lemma:funk} and \ref{lemma:rodrigues}, we obtain the following lemma that bounds the eigenvalues of $T_K$.

\begin{lemma} [Theorem 2 in \cite{minh2006mercer}]\label{lemma:rbf_eigen} 
	For the squared quadratic kernel in \eqref{eq:rbf_kernel}, 
	the corresponding integral operator has eigenvalues $\{ \rho_j\}_{j  \geq 0}$, where each $\rho_j$ is defined in \eqref{eq:eigenvalue_inner_prod_ker} with $k$ replaced by  $k_2$. 
	Moreover, each $\varrho_j$ has multiplicity $N(d, j)$ and the corresponding eigenfunctions are $\{ Y_{j, \ell}\}_{\ell \in [ N(d, j)]}$. 
	Finally,  when  $\sigma$ in \eqref{eq:rbf_kernel} satisfy 
 $\sigma^2 \geq  2 /d $, $\{ \varrho_{j }\}_{ j \geq 0}$ form a decreasing sequence that satisfy 
 \#\label{eq:rbf_eigen_bound}
A_1 \cdot  (   2e / \sigma^2  )^{j} \cdot  (2j + d-2)^{- (2j +d-1) / 2} <  \varrho_ j < A_2 \cdot  (   2e / \sigma^2  )^{j} \cdot  (2j + d-2)^{- (2j +d-1) / 2} 
  \#
  for all $j \geq 0$, 
  where $A_1, A_2$ are absolute constants that only depend on $d$ and $\sigma$. 
	\end{lemma} 

The $\ell_\infty$-norm of each  eigenfunction $ Y_{j, \ell}$ is given by the following lemma. 

\begin{lemma} [Lemma 3 in \cite{minh2006mercer}] \label{lemma:spherical_harm_sup}
	For any $d \geq 2$, $j \geq 0$, and any $\ell \in [ N(d, j)]$, 
	we have 
	\$
	\| Y_{j, \ell} \|_{\infty} = \sup_{z \in \SSS^{d-1} } | Y_{j, \ell} (z) | \leq \sqrt{ N(d, j) / | \SSS^{d-1} | } .
	\$  
	\end{lemma}
Now, let $\tau  > 0$ be a sufficiently small constant. 
Combining Lemmas \ref{lemma:rbf_eigen} and \ref{lemma:spherical_harm_sup}, we have 
\#\label{eq:eigen_supnorm_rbg}
\varrho_j ^\tau \cdot 	\| Y_{j, \ell} \|_{\infty}  \leq C \cdot \Bigl ( \frac{2e }{\sigma^2\cdot  ( 2j + d-2)} \Bigr ) ^{-j \cdot \tau } \cdot \sqrt {N(d, j) \cdot (2j+d - 2) ^{-(d-1)\cdot \tau} }, 
\#
where $C$ is a constant depending on $d$ and $\sigma$. 
By the definition of $N(d, j) $ in \eqref{eq:define_ndj_number}, 
when $j$ is sufficiently large, it holds that 
\#\label{eq:asymp_ndj}
  N(d, j)  \asymp \frac{ (2j+d-2) \cdot \sqrt{ d+j-3} \cdot [ (d+j -3) / e]^{d + j -3} }{ \sqrt{j } \cdot ( j / e) ^j }   \asymp j^{d-2},
\#
where we utilize the Stirling's formula and neglect constants involving $d$.
Then, combining \eqref{eq:eigen_supnorm_rbg} and \eqref{eq:asymp_ndj}, we have 
\#\label{eq:rbf_supnorm_eigen}
\sup_{j \geq 0} \sup_{ \ell \in [N(d, j) ]} \varrho_j ^\tau \cdot 	\| Y_{j, \ell} \|_{\infty}  \leq  C _{\varrho},
\#
for some absolute constant $C_\varrho > 0$.
Renaming the eigenvalues and eigenvectors as $\{ \sigma_j, \psi_j \}_{j \geq 1}$ in the descending order of the eigenvalues,  \eqref{eq:rbf_supnorm_eigen} equivalently states that $\sup_{j\geq 1} \sigma^{\tau}_j \cdot \| \psi_j \|_{\infty} \leq C_\varrho. $

Furthermore, to show that the squared exponential kernel satisfy the $\gamma$-exponential decay condition, 
we notice that 
\#\label{eq:order_eigenvalues_rbf}
\sigma_j = \varrho_{t } \qquad \textrm{for}~~ \sum_{ i = 1}^{t-1} N(d, i ) \leq j < \sum _{ i=1}^t N(d, i).
\#
Then  by \eqref{eq:asymp_ndj}, this implies that $\sigma_j \asymp \rho_t $ for $ (t-1) ^{d-1} \leq j \leq t^{d-1}$
when $j$ is sufficiently large. 
Thus, by Lemma \ref{lemma:rbf_eigen} we further obtain that 
\$
\sigma_j  & \asymp ( 2e / \sigma^2 ) ^{j ^{ \frac{1}{d-1} }} \cdot ( 2 j ^{ \frac{1}{d-1} }  + d - 2) ^{- j^{ \frac{1}{d-1} } - (d-1)/2}   \asymp 
 \exp \bigl ( c_1  \cdot j  ^{ \frac{1}{d-1} }
\bigr ) \cdot \exp \bigl ( c_2 -  j^{ \frac{1}{ d-1}  }  \cdot \log j  \big )  \leq \exp( - c \cdot j^{ 1/ d}) ,
\$
where $c$, $c_1$, and $ c_2$ are constants depending on $d$. Therefore, we have shown that the squared exponential kernel satisfies the $\gamma$-exponential decay condition with $\gamma = 1/d$. Combining this with \eqref{eq:rbf_supnorm_eigen}, we conclude that it satisfies Assumption \ref{assume:decay}.

In the sequel, we construct an NTK that satisfies Assumption \ref{assume:decay}.  Specifically, 
we adopt the sine activation function and  slightly modify the neural network in \eqref{eq:two_layer_nn} by employing an intercept for each neuron. That is, 
\$
f(z; b, W, \theta  ) = \frac{1}{\sqrt{ m}} \sum_{j = 1}^{ m} b_j \cdot  \sin ( W_j^\top z + \theta_ j ).
\$
To initialize the network weights $(b, W, \theta )$, we set $b_j = - b_{j-m}$, $W_j = W_{j-m}$, and $\theta_j = \theta_{j - m}$
 for any $j \in \{ m +1, \ldots, 2m\}. $ 
 For any $j\in[m]$, we independently sample  $b_j \sim \textrm{Unif} ( \{ -1, 1\})$, $W_j \sim N(0, I_d)$, and $\theta_{j} \sim \textrm{Unif}( [0, 2\pi])$. Only $W$ is updated during training. 
 
 For such a neural network, the corresponding NTK is given by 
 \#\label{eq:rbf_ntk}
 K_{\textrm{ntk}} (z,z') &  = 2 \EE \bigl [ (z^\top z' )\cdot \cos ( w^\top z + \theta ) \cdot \cos( w^\top z' + \theta ) \bigr ]  \notag \\
 & = (z^\top z' )\cdot \exp( - \| z - z' \|_2 ^2 / 2)  =    (z^\top z' ) \cdot \exp [  (z^\top z') - 1 ]  = k_3 (\la z, z' \ra ),
 \#
 where we define $k_3(u) = u\cdot \exp( u-1)$. Here the second equality follows from \cite{rahimi2008random}. 
 By construction, such an NTK is closely related to the squared quadratic kernel in \eqref{eq:rbf_kernel}. 
To see that it satisfy the $\gamma$-exponential decay condition, let $\{ \varrho_j\}_{j\geq 0}$ and $\{ \tilde \varrho_j\}_{j\geq 0}$ denote the eigenvalues of the NTK in \eqref{eq:rbf_ntk} and the inner product kernel induced by $k_2(u ) = \exp( u-1)$, respectively. 
By Lemma \ref{lemma:funk}, we have 
 \#\label{eq:ntk_rbf_eigen}
 \rho_ j & =  C _1 \cdot \int_{ -1} ^1 k_3 (u) \cdot P_j (u; d)  \cdot (1-u^2) ^{(d-3)/2}~\ud u =  C_1 \cdot \int_{ -1} ^1 k_2  (u) \cdot u \cdot P_j (u; d)  \cdot (1-u^2) ^{(d-3)/2}~\ud u \notag \\
 & =  C _2 \cdot j / ( 2j+d-2) \cdot \tilde \varrho_{j-1} + C_2 \cdot  ( j + d-2) / (2 j+d - 2) \cdot \tilde \varrho_{j+1} \leq C_2 ( \tilde \rho_{j-1} + \tilde \rho_{j+1})  , 
 \#
 where $C_1$ and $C_2$ are constants and in the second equality, we utilize the following recurrence relation of Legendre polynomials:
 $$
 u \cdot P_j(u; d) = j / ( 2j+d-2) \cdot P_{j-1} (u; d) + ( j + d-2) / (2 j+d - 2) \cdot P_{j+1} (u; d).
 $$
 Notice that $\{ \tilde \varrho_j\}_{j\geq 0}$ satisfy \eqref{eq:rbf_eigen_bound}. 
 Thus, combining \eqref{eq:rbf_eigen_bound} and \eqref{eq:ntk_rbf_eigen}, we obtain \eqref{eq:rbf_supnorm_eigen}. 
 Moreover, when ordering all the eigenvalues of $K_{\textrm{ntk}}$ in the descending order and renaming them as $\{ \sigma_j \}_{j\geq 1}$, similar to \eqref{eq:order_eigenvalues_rbf}, 
 we have 
 \# \label{eq:order_eigenvalues_ntk_rbf}
\sigma_j \leq C_2 \cdot (\tilde \rho_{t-1} + \tilde \rho_{t+1} ) \qquad  \textrm{for}~~ \sum_{ i = 1}^{t-1} N(d, i ) \leq j < \sum _{i=1}^t N(d, i).
\#
 Using a similar analysis, we can show that $\{ \sigma_j\}_{j\geq1 }$ satisfy the $\gamma$-exponential eigenvalue decay condition with $\gamma = 1/ d$. Therefore, we have shown that the NTK given in \eqref{eq:rbf_ntk} satisfy Assumption~\ref{assume:decay}. 
 
\vspace{5pt} 
{\noindent \bf Case (iii): $\gamma$-Polynomial Decay.} 
Finally, for the last case, 
it is stated in \cite{srinivas2009gaussian} that the Mat\'ern  kernel on $[0, 1]^d$ with parameter $\nu > 0$ satisfies the 
$\gamma$-polynomial decay condition with $\gamma = 1 + 2\nu /d$ . 
Moreover, the eigenfunctions are sinusoidal functions and thus are bounded. 
Hence, the Mat\'ern kernel satisfies the last case of  Assumption \ref{assume:decay}
with $\gamma = 1 + 2\nu /d$ and $\tau = 0$. 
As a result, in this case,  \eqref{eq:bound_BT_kappa} reduces to  
\$
\kappa_1 = \xi^* = \frac{d(d+1) }{2[ 2\nu + d(d+1) ] }, \qquad \kappa_2 =\frac{d(d+1) + \nu}{(d+1+ 2\nu/d) \nu} , \qquad \kappa_3 = \frac{d}{\nu-d}.
\$
We further assume   $\nu$ is sufficiently large such that $\Gamma_K (T, \lambda ) \cdot \sqrt{T} = o(T)$,
which implies that $\kappa_1 < 1/4$ and that $2 \nu > d(d+1)$.  
In this case, we have 
\$
\kappa_2 = \frac{1 + d(d+1) / \nu}{ d+1+ 2\nu /d} < \frac{3}{d+1}, \qquad \kappa_2 < \frac{2}{d-1}, \qquad \kappa^* = \max \Bigl \{ \xi^*, \frac{ 3}{d+1}, \frac{2}{d-1} \Big \}.
\$
Thus,   \eqref{eq:kucb_regret} reduces to 
an $\tilde \cO(H^2 \cdot T^{\kappa^*+\xi^* + 1/2})$ regret.

Finally, we construct an NTK that satisfies Assumption \ref{assume:decay}. Similar to the NTK  given in \eqref{eq:ntk_sphere_poly}, 
with $\cZ= \SSS^{d-1}$,  for the   neural network in \eqref{eq:two_layer_nn}, we let  the  activation function be  $\act(u) = (u)_{+}^ {s+1}$ and set $W_j \overset{\text{i.i.d.}}{\sim} \textrm{Unif}(\SSS^{d-1})$. 
Here    $s $ is a positive integer and $(u)_{+} = u \cdot \ind\{ u \geq 0\}$ is the ReLU activation function.
By direct computation, the induced NTK is 
\#\label{eq:ntk_relu}
K_{\textrm{ntk}} (z, z' ) = K_s (z, z')\cdot (z^\top z' ), \qquad K_s(z,z') = (s+1)^2 \cdot  \EE _{w\sim \textrm{Unif} (\SSS^{d-1} ) }  \bigl  [ ( w^\top z )_{+}^s  ( w^\top z'  )_{+}^s \bigr ] .
\#
Utilizing the  rotational invariance of $\textrm{Unif} (\SSS^{d-1} ) $, it can be shown that  $K_{\textrm{ntk}}  $  and $K_s$ in \eqref{eq:ntk_relu} are both  inner-product kernels; i.e., there exist  univariate functions $k_4, k_5 \colon [-1, 1] \rightarrow \RR$  such that 
\$
K_{\textrm{ntk}} (z, z' ) = k_4 (\la z, z' \ra), \qquad K_s(z,z') =    k_5 (\la z, z' \ra),  \qquad k_4 (u) = u\cdot k_5(u).
\$
By Lemma \ref{lemma:funk}, for any $j \geq 0$.  $K_{\textrm{ntk}}$ and $K_s$ both have spherical  harmonics $\cY_j(d)$ as their eigenvectors. Let the corresponding eigenvalues be $\varrho_j$ and $\tilde \varrho_j$,  respectively. 
Similar to \eqref{eq:ntk_rbf_eigen}, we have 
$
\varrho_j \leq C_2 ( \tilde \varrho_{j-1} + \tilde \rho_{j+1} ), 
$
where $C_2$ is an absolute constant depending on $d$. 
Furthermore, as shown in Appendix \S D in \cite{bach2017breaking}, when $j$ is sufficiently large, it holds that $\tilde \varrho_j \asymp j^{-( d+2s)  }
$. 
This further implies that $\varrho_j = \cO(j^{-( d+2s) } )$ where $\cO(\cdot )$ hides constants depending on $d$. 
Thus, for  $\tau = (d-2) / ( 2d+4s)$, combining Lemma \ref{lemma:spherical_harm_sup} and \eqref{eq:asymp_ndj}, we have 
\#\label{eq:relu_eigenvec_sup}
\varrho_j ^\tau \cdot \| Y_{j, \ell} \|_{\infty} \leq C \cdot j^{-( d+2s) \cdot \tau   } \cdot j^{(d-2) /2} = C_{\varrho} \cdot j^{  - [ (2d+4s) \cdot \tau - (d-2) ]/ 2 } \leq C_{\varrho}  
\#
when $j$ is sufficiently large, 
where $C_{\varrho}  $ is an absolute  constant depending on $d$. 
Moreover, we have $\tau \in [0, 1/2)$ and we can make $\tau $ be sufficiently small 
by increasing  $s$. 

Finally, let $\{ \sigma_j \}_{j\geq 1}$ be all the eigenvalues of $K_{\textrm{ntk} } $ in descending order. By   \eqref{eq:order_eigenvalues_ntk_rbf}  and \eqref{eq:asymp_ndj}, when $j$ is sufficiently large,  we have 
\$
\sigma_j = \cO\bigl (  ( j^{\frac{1}{d-1} } ) ^{ - (  d+ 2s)}  \bigr ) = \cO \bigl ( j^{ -   ( 1+ \frac{1+2s}{d-1}  ) }   \bigr ),
\$
where $\cO(\cdot)$ hides constants depending on $d$. 
Therefore, $K_{\textrm{ntk}}$   in \eqref{eq:ntk_relu} satisfies the $\gamma$-polynomial decay condition with $\gamma = 1+ (1+2s) / (d-1)$. Combining this with \eqref{eq:relu_eigenvec_sup}, we conclude that $K_{\textrm{ntk}}$ satisfy the last case of Assumption \ref{assume:decay}.

\section{Proofs of the Supporting Lemmas} \label{sec:proof_lemmas}

\subsection{Proof of Lemma \ref{lemma:regret_decomp}} \label{sec:proof_regret_decomp}
\begin{proof}
For ease of presentation, before presenting the proof,  we first define two operators $\mathbb{J}_h^{\star} $ and $\mathbb{J}_{t,h}$ respectively by letting
\#\label{eq:def_j_oper}
(\mathbb{J}_h^{\star} f)(x) = \la f(x,\cdot), \pi^{\star} _h(\cdot\,|\,x) \ra_{\cA} , \quad
(\mathbb{J}_{t,h} f)(x) = \la f(x,\cdot), \pi^t_h(\cdot\,|\,x) \ra_{\cA},
\#
for any $(t, h) \in [T] \times [H]$ and any function $f: \cS\times\cA\rightarrow \RR$. 
Moreover, for any $(t,h)\in[T]\times[H]$ and  any state $x\in\cS$,
 we define
\#\label{eq:def_xi_hat}
\xi^t_h(x) = (\mathbb{J}_{h} Q^{t}_h)(x) - (\mathbb{J}_{t,h} Q^{t}_h)(x) = \la Q^{t}_h(x,\cdot), \pi^\star _h(\cdot\,|\,x) - \pi^t_h(\cdot\,|\,x) \ra_{\cA}.
\#
 After introducing this notation, to prove \eqref{eq:regret_decomp} we  decompose the instantaneous regret at the $t$-th episode into the following two terms,
\#\label{eq:decomp1}
V^{\star}_1(x_1^t) - V^{\pi^t}_1(x_1^t) = \underbrace{V^{\star}_1(x_1^t) - V^{t}_1(x_1^t)}_{\dr (i)} + \underbrace{V^{t}_1(x_1^t) - V^{\pi^t}_1(x_1^t)}_{\dr (ii)}. 
\#
In the sequel, we consider the two terms  in \eqref{eq:decomp1} separately.

\vspace{4pt}
\noindent
{\bf Term (i).} By the definitions of the value function $V^{\star}_h$ in \eqref{eq:opt_bellman} and the operator $\JJ_h^{\star}$ in \eqref{eq:def_j_oper}, we have $ V^{\star}_h = \JJ_h^{\star} Q_h^{\star}$. Similarly, for all the algorithms, we have $V^t_h (x) = \la Q_h^t(x, \cdot )  , \pi_h^t (\cdot  \given x) \ra $ for all $x \in \cS$. Thus, by the definition of $\JJ_{t,h}$ in \eqref{eq:def_j_oper}, we have 
$V_h^t = \JJ_{t,h} Q_h^t$.  Thus, using  $\xi^t_h$ defined  in \eqref{eq:def_xi_hat},  for any $(t,h) \in [T] \times [H]$, we have 
\#\label{eq:decomp2} 
V^{\star}_h - V^{t}_h &=  \mathbb{J}_h^{\star} Q^{\star}_h - \mathbb{J}_{t,h} Q^{t}_h =   \bigl ( \mathbb{J}_h^{\star} Q^{\star}_h - \mathbb{J}_h^{\star} Q^{t}_h\bigr )   + \bigl ( \mathbb{J}_h^{\star} Q^{t}_h - \mathbb{J}_{t,h} Q^{t}_h \bigr )   \notag\\
&=  \mathbb{J}_h^{\star} ( Q^{\star}_h - Q^{t}_h) + \xi^t_h,
\#
 where the last equality follows from the definition of $\xi^t_h$ in  \eqref{eq:def_xi_hat} and the fact that $\mathbb{J}_h^{\star} $ is a linear operator.  Moreover, by the definition of the temporal-difference error $\delta_h^t$ in \eqref{eq:td_error} and the Bellman optimality condition, we have 
 \#\label{eq:decomp3} 
 Q^{\star}_h - Q^{t}_h = \bigl  ( r_h + \mathbb{P}_h V^{\star}_{h+1} \bigr )  - \bigl ( r _h + \mathbb{P}_h V^{t}_{h+1} -  \delta^t_h \bigr )  = \mathbb{P}_h (V^{\star}_{h+1} - V^{t}_{h+1})+  \delta^t_h. 
 \#
Thus, combining \eqref{eq:decomp2} and \eqref{eq:decomp3}, we obtain that 
\#\label{eq:decomp4}
V^{\star}_h - V^{t}_h =
\mathbb{J}_h^{\star}\mathbb{P}_h (V^{\star}_{h+1} - V^{t}_{h+1}) + \mathbb{J}_h^{\star}  \delta^t_h + \xi^t_h, \qquad \forall (t, h ) \in [T] \times [H]. 
\#
Equivalently, for all $x \in \cS  $, and all $(t, h) \in [T] \times [ H]$, we have
\$
V^{\star}_h (x)  - V^{t}_h  (x) = &  \EE_{a\sim \pi_h^\star (\cdot \given x)} \bigl \{   \EE \bigl [ V_{h+1}^\star (x_{h+1} ) - V_{h+1}^t (x_{h+1} )  \biggiven x_h = x, a_h = a \bigr ] \bigr \}  \\
& \qquad + \EE_{a\sim \pi_h^\star (\cdot \given x)} \bigl [ \delta_{h}^t  (x,a)\bigr ] + \xi_h^t (x). 
\$
Then, by recursively applying \eqref{eq:decomp4} for all $h \in [H]$,  we have 
\#\label{eq:decomp5}
&V^{\star}_1 - V^t_1  =
\Bigl(\prod_{h=1}^H \mathbb{J}_h^{\star}\mathbb{P}_h \Bigr) (V^{\star}_{H+1} - V^{k}_{H+1})
+ \sum_{h=1}^H \Bigl(\prod_{i=1}^{h-1} \mathbb{J}_i ^\star \mathbb{P}_i \Bigr)
\mathbb{J}_h^{\star}  \delta^t_h +
\sum_{h=1}^H \Bigl(\prod_{i=1}^{h-1} \mathbb{J}_i ^\star \mathbb{P}_i \Bigr) \xi^t_h.
\#
Furthermore, notice that we have $V^{\star}_{H+1} = V^{k}_{H+1}=\zero$. Thus,   
 \eqref{eq:decomp5}  can  be equivalently written as 
 \$
 V^{\star}_1 (x) - V^{t}_1 (x)= &   \EE_{\pi^\star } \bigg[ \sum_{h=1}^H   \bigl \la Q^{t}_h(x_h,\cdot), \pi^{\star}_h(\cdot\,|\,x_h) - \pi^t_h(\cdot\,|\,x_h)\bigr  \ra_{\cA} + \delta^t_h(x_h,a_h) \bigggiven x_1 = x \biggr ]   , 
 \$
 where we utilize the definition of $\xi _h^t$ given in \eqref{eq:def_xi_hat}. 
 Thus, we can write Term (i)  on the right-hand side of \eqref{eq:decomp1} as 
 \# \label{eq:decomp_first}
 V^{\star}_1(x_t^t) - V^t_1(x_t^t) =& \sum_{h=1}^H \EE_{\pi^{\star}} \bigl[\la Q^{t}_h(x_h,\cdot), \pi^{\star}_h(\cdot\,|\,x_h) - \pi^t_h(\cdot\,|\,x_h) \ra _{\cA}\,\big|\, x_1 = x_t^t\bigr]  \notag \\
 & \qquad   + \sum_{h=1}^H \EE_{\pi^{\star}}[ \delta^t_h(x_h,a_h)\,|\, x_1 = x_t^t], \qquad \forall t\in [T]. 
 \#

\vspace{4pt}
\noindent
{\bf Term (ii).} 
It remains to bound the second term on the right-hand side of \eqref{eq:decomp1}.
By the definition of the temporal-difference error $ \delta^t_h$ in \eqref{eq:td_error}, for any $(t,h) \in [T] \times [H]$,  we have
\#\label{eq:decomp6}
\delta^t_h(x^t_h,a^t_h)&= r_h(x^t_h,a^t_h) + (\mathbb{P}_h V^{t}_{h+1})(x^t_h,a^t_h) - Q^{t}_h(x^t_h,a^t_h)  \notag\\
&=\bigl [  r_h(x^t_h,a^t_h) + (\mathbb{P}_h V^{t}_{h+1})(x^t_h,a^t_h) - Q^{\pi^t} _h(x^t_h,a^t_h) \bigr ]  + \bigl [  Q^{\pi^t} _h(x^t_h,a^t_h) - Q^{t}_h(x^t_h,a^t_h) \bigr ]  \notag\\
&=\bigl(\mathbb{P}_h V^{t}_{h+1}- \PP_h V^{\pi^t}_{h+1}\bigr)(x^t_h,a^t_h) + ( Q^{\pi^t} _h- Q^{t}_h)(x^t_h,a^t_h),
\#
where the last equality follows from the Bellman equation \eqref{eq:bellman}. 
Morerover, recall that we define $\zeta_{t,h}^1$ and $\zeta_{t,h}^2$ in \eqref{eq:define_mtg1} and \eqref{eq:define_mtg2}, respectively. Thus, from \eqref{eq:decomp6} we obtain that 
\#\label{eq:decomp8}
& V^{t}_h(x^t_h) - V^{\pi^t}_h(x^t_h) \\
&\qquad =V^{t}_h(x^t_h) - V^{\pi^t}_h(x^t_h)  + ( Q^{\pi^t} _h- Q^{t}_h)(x^t_h,a^t_h)  + \bigl(\mathbb{P}_h (V^{t}_{h+1}-V^{\pi^t}_{h+1})\bigr)(x^t_h,a^t_h) -   \delta^t_h(x^t_h,a^t_h),\notag \\
& \qquad =  \bigl(V^{t}_h  - V^{\pi^t}_h  \bigr)(x^t_h) - ( Q^{t}_h-Q^{\pi^t} _h)(x^t_h,a^t_h)   \notag \\
& \qquad \qquad   +  \bigl(\mathbb{P}_h (V^{t}_{h+1}-V^{\pi^t}_{h+1})\bigr)(x^t_h,a^t_h) - (V^{t}_{h+1}-V^{\pi^t}_{h+1})(x^t_{h+1})  + (V^{t}_{h+1}-V^{\pi^t}_{h+1})(x^t_{h+1})-   \delta^t_h(x^t_h,a^t_h)\notag \\
& \qquad = \bigl [  V_{h+1}^t (x_{h+1}^t) - V_{h+1}^{\pi^t} (x_{h+1}^t)\bigr ] + \zeta_{t,h}^1 + \zeta_{t, h}^2  - \delta_{h}^t (x_h^t , a_h^t).  \notag 
\#
Thus, recursively applying \eqref{eq:decomp8} for all $h \in [H]$, we obtain that 
\#\label{eq:decomp_second}
V^t_1(x^t_1) - V^{\pi^t}_1(x^t_1)  & = 
V^{t}_{H+1}(x^k_{H+1}) - V^{\pi^k,k}_{H+1}(x^t_{H+1}) + \sum_{h=1}^H ( \zeta_{t,h}^1 +\zeta_{t,h}^2)- \sum_{h=1}^H  \delta^t_h(x^t_h,a^t_h)  \notag  \\
& = \sum_{h=1}^H ( \zeta_{t,h}^1 +\zeta_{t,h}^2)- \sum_{h=1}^H  \delta^t_h(x^t_h,a^t_h) , \qquad \forall t\in [T],
\#
where the last equality follows from the fact that   $V^{t}_{H+1}(x^t_{H+1}) = V^{\pi^t}_{H+1}(x^t_{H+1})=0$.
Thus, we have simplified Term  (ii) defined in \eqref{eq:decomp1}. 

Thus, combining \eqref{eq:decomp1}, \eqref{eq:decomp_first}, and \eqref{eq:decomp_second}, we obtain that
\$
\text{Regret}(T) &= \sum_{t=1}^T \bigl [ V^{\star }_1(x^t_1) - V^{\pi^t}_1(x^t_1)\bigr ]  \notag \\
& =  \sum_{t=1}^T \sum_{h=1}^H \EE_{\pi^{\star}}[ \delta^t_h(x_h,a_h)\,|\, x_1 = x_t^t]  + \sum_{t =1}^T \sum_{h=1}^H ( \zeta_{t,h}^1 +\zeta_{t,h}^2) - \sum_{t=1}^T\sum_{h=1}^H  \delta^t_h(x^t_h,a^t_h)  \notag \\
&\qquad +
\sum_{t=1}^T  \sum_{h=1}^H \EE_{\pi^{\star}} \bigl[ \bigl \la Q^{t}_h(x_h,\cdot), \pi^{\star}_h(\cdot\,|\,x_h) - \pi^t_h(\cdot\,|\,x_h)  \bigr \ra_{\cA} \,\big|\, x_1 = x_t^t\bigr] .
\$
Therefore, we conclude the proof of this lemma. 
\end{proof}

\subsection{Proof of Lemma \ref{lemma:ucb_kernel1} } \label{sec:proof_ucb_kernel1}
\begin{proof} 
For ease of presentation, we utilize the feature  representation induced by the kernel $K$. Let $\phi \colon \cZ \rightarrow \cH$ be the feature mapping such that $K(z, z') = \la \phi (z), \phi(z') \ra_{\cH}$. 
For simplicity, we formally view $\phi (z)$ as a vector and write  $\la \phi (z), \phi(z') \ra_{\cH}= \phi(z)^\top \phi(z')$. 
Then, any function $f \colon \cZ \rightarrow   \RR$ in the RKHS satisfies $f(z) = \la \phi(z) , f\ra_{\cH} = f^\top \phi(z).$
Using the feature representation, we can rewrite the kernel ridge regression in 
\eqref{eq:krr_vi}   as 
 \#\label{eq:KRR_rewrite}
   \minimize_{\theta \in \cH } L(\theta) & =   \sum_{\tau =1}^{t-1}  \bigl [   r_h ( x_h^\tau , a_h^\tau) + V_{h+1}^t ( x_{h+1} ^\tau) -  \la \phi ( x_h^{\tau },  a_h^\tau ) , \theta\ra _{\cH}    \bigr ] ^2 + \lambda \cdot \| \theta \|_{\cH} ^2 .
\#
We define the feature matrix  $\Phi_h^t\colon \cH \rightarrow \RR^{t-1} $ and ``covariance matrix" $\Lambda_h^t \colon \cH \rightarrow \cH$ respectively as 
\#\label{eq:define_feature_mat}
\Phi_h^t = \big [ \phi(z_h^1)^\top , \ldots, \phi (z_h^{t-1}) ^\top   \bigr ] ^\top, \qquad  \Lambda _h^t = \sum_{\tau = 1}^{t-1}  \phi(z_h^\tau) \phi(z_h^\tau)^\top + \lambda \cdot I_{\cH} =  \lambda \cdot I_{\cH} +(\Phi_h^t )^\top \Phi_h^t  ,
\#
where $I_{\cH}$ is the identity mapping on $\cH$. 
Thus, the Gram matrix $K_h^t$ in \eqref{eq:define_gram} is equal to $ \Phi_h^t  (\Phi_h^t )^\top $. 
More specifically, 
 here $\Lambda_h^t$ is a self-adjoint and positive-definite operator. 
For any $f_1 , f_2 \in \cH$, we denote
\$
\Lambda_h^t f_1 = \lambda\cdot f_1 + \sum_{\tau = 1}^{t-1} \phi(z_h^\tau) \cdot f_1 (x_h^\tau) \in \cH, \qquad f_1^\top \Lambda_h^t f_2 = \la f_1, \Lambda_h^t f \ra _{\cH} . 
\$
It is not hard to see that all the eigenvalues of $\Lambda _h^t$ are positive and at least $\lambda$. 
Thus, the inverse operator of $\Lambda_h^t$, denoted by $(\Lambda_h^t)^{-1}$, is well-defined, which is also a self-adjoint and positive-definite operator on $\cH$.
Similarly, for any $f_1, f_2 \in \cH$, we let $f_1^\top (\Lambda_h^t)^{-1} f_2 $ denote $\la  f_1, (\Lambda_h^t)^{-1} f_2 \ra_{\cH}$. 
The eigenvalues of $(\Lambda_h^t)^{-1}$ are all bounded in interval $[0, 1/ \lambda]$. 
 
In addition, using the feature matrix $\Phi_h^t$ defined in \eqref{eq:define_feature_mat} and $y_h^t$ defined in \eqref{eq:krr_response}, we can write \eqref{eq:KRR_rewrite}~as 
\$
 \minimize_{\theta \in \cH } L(\theta) & =      \|  y_h^t - \Phi_h^t \theta   \|_2^2    + \lambda \cdot  \theta ^\top \theta,
\$
whose solution is given by 
$
\hat \theta_h^t = (\Lambda_h^t )^{-1} (\Phi_h^t )^\top y_h^t.
$
and $\hat Q_h^t$ in  \eqref{eq:krr_vi} satisfies $\hat Q_h^t (z) = \phi(z) ^\top \hat \theta_h^t$.

In the sequel, to further simplify the notation, we let $\Phi$ denote $\Phi_h^t$ when its meaning is clear from the context. 
 Since both  $(\Phi  \Phi^\top  + \lambda \cdot  I)$ and $(\Phi ^\top\Phi  + \lambda \cdot I_{\cH} )$ are strictly positive definite and 
 \$
(\Phi^\top \Phi  + \lambda \cdot  I _{\cH}) \Phi^\top = \Phi^\top (\Phi \Phi ^\top + \lambda \cdot I  ) ,
\$
  which implies that 
\#\label{eq:dim_change}
(\Lambda_h^t )^{-1} \Phi^\top  = (\Phi \Phi ^\top + \lambda \cdot I_{\cH} )^{-1} \Phi^\top = \Phi^\top (\Phi \Phi  ^\top + \lambda \cdot  I)^{-1} =\Phi^\top  (K_h^t + \lambda \cdot I)^{-1} .
\#
Here $I$ is the identity matrix in $\RR^{(t-1) \times (t-1)}$.
Thus, by \eqref{eq:dim_change} we have 
\#\label{eq:theta_alpha}
\hat \theta_h^t = (\Lambda_h^t)^{-1} \Phi^\top y _h^t = \Phi^\top  (K_h^t + \lambda \cdot I)^{-1} y _h^t = \Phi^\top \alpha_h^t. 
\#
Moreover, $k_h^t$ defined in  \eqref{eq:define_gram} can be written   as 
$ k_h^t (z)  =  \Phi \phi(z)$, which, combined with \eqref{eq:dim_change}, implies  
   \#\label{eq:decompose_phi}
\phi(z) & = (\Lambda_h^t ) ^{-1} \Lambda_h^t  \phi(z) =  (\Lambda_h^t ) ^{-1}  (\Phi^\top\Phi +\lambda \cdot  I_{\cH} )\phi(z ) \notag \\
& =(\Lambda_h^t ) ^{-1}  (\Phi^\top\Phi )\phi(z ) + \lambda \cdot (\Lambda_h^t ) ^{-1}  \phi(z) \notag \\
& =   \Phi^\top  (K_h^t + \lambda \cdot I)^{-1} k_h^t(z)+\lambda \cdot (\Lambda_h^t ) ^{-1} \phi(z) . 
\#
 Thus, we can write $\| \phi(z) \|_{\cH}^2 = \phi(z)^\top \phi(z)$ as 
 \$
 \| \phi (z) \|_{\cH}^2  & = \phi(z)^\top  \cdot \bigl[  \Phi^\top  (K_h^t + \lambda \cdot I)^{-1}  k_h^t(z)+\lambda \cdot (\Lambda_h^t ) ^{-1} \phi(z) \bigr ] \\
 & =  k_h^t(z) ^\top (K_h^t + \lambda \cdot I)^{-1} k_h^t(z) + \lambda \cdot  \phi(z)(\Lambda_h^t ) ^{-1} \phi(z),
 \$
 which implies that we can equivalently write the bonus $b_h^t$ defined in 
 \eqref{eq:ucb_bonus} as 
 \#\label{eq:new_bonus} 
 b_h^t (x,a)  = \bigl [\phi(x,a)^\top (\Lambda_h^t ) ^{-1} \phi(x,a )\bigr ] ^{1/2} =   \| \phi(x,a) \|_{(\Lambda_h^t)^{-1} } .
 \#
 Combining \eqref{eq:theta_alpha} and \eqref{eq:new_bonus}, 
 we equivalently write $Q_h^t$ in \eqref{eq:update_Q_func} as 
 \#\label{eq:rewrite_Q_func}
 Q_h^t (x,a ) & = \min \bigl  \{\hat Q_h^t (x,a ) + \beta \cdot b_h^t (x,a) , H - h +1    \bigr \}^+ \notag \\
 & = \min \bigl  \{\phi (x,a)^\top \hat \theta_h^t  + \beta \cdot     \| \phi(x,a) \|_{(\Lambda_h^t)^{-1} } , H - h +1   \bigr \}^+ . 
 \#

 Now we are ready to bound the temporal-difference error $\xi_h^t $ defined in \eqref{eq:td_error}. 
 Noticing that 
 $V_{h}^t (x) = \max_{a} Q_{h}^t(x,a)$ for all $(t, h) \in [T]\times [H]$, 
 we have 
 $$
 \delta_h^t   = r_h + \PP_h V_{h+1}^ t - Q_h^t =  \TT_h ^{\star} Q_{h+1} ^t - Q_h^t ,$$
 where $\TT_h^\star $ is the Bellman optimality operator. 
 Under the  Assumption \ref{assume:opt_closure}, for all $(t,h) \in [T] \times [H]$, 
 since $Q_{h+1}^t \in [0, H]$, 
 we have  $ \TT_h ^{\star} Q_{h+1} ^t \in \cQ^\star $.
Using the feature representation of RKHS, there exists  
  $\overline \theta_h^t \in  \cQ^\star$ such that  $ ( \TT_h ^{\star} Q_{h+1} ^t)(z) =   \phi(z) ^\top  \overline  \theta_h^t    $ for all $z \in \cZ$.

 In the sequel, we consider the difference between $\phi(z)^\top \hat \theta_h^t$ and $\phi(z)^\top  \overline \theta_h^t$. 
   To begin with, using \eqref{eq:decompose_phi}, we can write  $\phi(z)^\top  \overline \theta_h^t $ as 
   \#\label{eq:optim11}
 \phi(z)^\top  \overline \theta_h^t =    k_h^t(z)^\top (K_h^t + \lambda \cdot I)^{-1} \Phi \overline \theta_h^t +\lambda \cdot \phi(z)^\top (\Lambda_h^t ) ^{-1} \overline \theta_h^t . 
   \#
   Hence, combining \eqref{eq:theta_alpha} and \eqref{eq:optim11}, we have 
   \#\label{eq:optim12} 
   \phi(z)^\top \hat \theta_h^t -  \phi(z)^\top  \overline \theta_h^t =\underbrace{  k_h^t(z)^\top (K_h^t + \lambda \cdot I)^{-1} \bigl ( y_h^t - \Phi \overline \theta_h^t  \bigr ) }_{\dr (i)} -  \underbrace{\lambda\cdot
\phi (z) ^\top (\Lambda^{t}_h)^{-1} \overline \theta_h^t }_{\dr (ii)}.
   \#

  We bound Term (i) and Term (ii) on  the right-hand side of \eqref{eq:optim12} separately. 
For Term (ii), by the Cauchy-Schwarz inequality, we have 
\#\label{eq:optim13} 
  \bigl | \lambda \cdot   \phi(z)^\top ( \Lambda^{t}_h)^{-1} \overline \theta _h^t \big |  
&\le   \bigl \| \lambda  \cdot( \Lambda^{t}_h)^{-1}\phi(x) \bigr \|_{\cH} \cdot  \| \overline \theta _h^t \| _{\cH} \leq R_Q  H \cdot \bigl \| \lambda  \cdot ( \Lambda^{t}_h)^{-1}\phi(x) \bigr \|_{\cH}   \\
&= R_Q H \cdot   \sqrt{\lambda \cdot \phi(z)^\top  ( \Lambda^{t}_h)^{-1} \cdot \lambda \cdot I_{\cH}  \cdot ( \Lambda^{t}_h)^{-1} \phi(x)} \notag \\
&\le  R_Q H \cdot  \sqrt{\lambda\cdot  \phi(z)( \Lambda^{t}_h)^{-1} \cdot \Lambda_{h}^t \cdot ( \Lambda^{t}_h)^{-1}\phi(z)} =\sqrt{\lambda}  R_Q H \cdot b_h^t (z).  \notag 
\#
Here  the first inequality follows from the Cauchy-Schwarz inequality and the second inequality follows from the fact that 
$\overline \theta_h^t \in \cQ^\star$, which implies that 
$ \| \overline \theta _h^t \| _{\cH} \leq R_Q H $.
Moreover, the last inequality follows from the fact that $ \Lambda_{h}^t - \lambda \cdot  I_{\cH} $ is a self-adjoint and positive-semidefinite operator, which means that 
$f^\top ( \Lambda_{h}^t - \lambda \cdot  I_{\cH}) f \geq 0$ for all $f \in \cH$, and the last equality follows from 
\eqref{eq:new_bonus}.
 
 Furthermore, for Term (i), by the Bellman equation in \eqref{eq:opt_bellman} and the definition of $y_h^t$ in \eqref{eq:krr_response}, for any $\tau \in [t-1]$, the $\tau$-th entry of  $( y_h^t - \Phi \overline \theta_h^t  \bigr )$ can be written as 
 \#\label{eq:optim14}
 [y_h^t]_{\tau} - [  \Phi \overline \theta_h^t ]_{\tau}  &   = r_h (x_h^\tau , a_h^\tau ) + V_{h+1} ^t(x_{h+1}^\tau ) - \phi(x_h^\tau, a_h^\tau)^\top \overline \theta_h^t = r_h (x_h^\tau , a_h^\tau ) + V_{h+1} ^t(x_{h+1}^\tau ) - (\TT_h^\star Q_{h+1}^t ) (x_h^\tau , a_h^\tau ) \notag \\
 & = V_{h+1} ^t(x_{h+1}^\tau ) -  (\PP_h  V_{h+1}^t ) (x_h^\tau , a_h^\tau ).
 \#
 Thus, combining \eqref{eq:dim_change},  \eqref{eq:optim12}, and \eqref{eq:optim14}
 we have 
 \# \label{eq:optim15}
&\bigl|  k_h^t(z)^\top (K_h^t + \lambda \cdot I)^{-1} \bigl ( y_h^t - \Phi \overline \theta_h^t  \bigr )  \bigr |  \notag \\
&\qquad= \biggl | \phi(z)^\top(\Lambda^{t}_h)^{-1} \bigg\{ \sum_{\tau=1}^{t-1} \phi(x_h^\tau,a_h^\tau)
\cdot \bigl[ V^{k}_{h+1}(x_{h+1}^\tau) - (\mathbb{P}_hV^{k}_{h+1})(x_h^\tau,a_h^\tau)\bigr ] \bigg\}   \bigg| \notag \\
&\qquad  \leq \| \phi(z) \|_{(\Lambda_h^t) ^{-1} } \cdot \biggl \| \sum_{\tau = 1}^{t-1} \phi(x_h^\tau,a_h^\tau)
\cdot \bigl[ V^{t}_{h+1}(x_{h+1}^\tau) - (\mathbb{P}_hV^{t}_{h+1})(x_h^\tau,a_h^\tau)\bigr ]\biggr \| _{(\Lambda^t_h)^{-1}},
\# 
where the last inequality follows from the Cauchy-Schwarz inequality. 
In the following, we aim to bound \eqref{eq:optim15}
by the concentration of self-normalized stochastic processes in the RKHS. 
However, here $V_{h+1}^t$ depends on the historical data in the first $(t-1)$ episodes and is  thus not independent of $\{ (x_h^\tau, a_h^\tau, x_{h+1} ^\tau ) \}_{\tau \in [t-1]}$. 
To bypass this challenge, in the sequel, 
we combine the concentration of self-normalized processes and uniform convergence over the function classes
that  contain each  $V_{h+1}^t$.

Specifically, recall that we define function classes $\cQ_{\UCB} (h, R, B)$ 
in \eqref{eq:main_text_func_class} for any $h \in [H]$, and any $R, B > 0$. 
We define $\cV_{\UCB} ( h, R, B)$ as 
\#\label{eq:main_text_v_class}
\cV_{\UCB}
 ( h , R, B) =\bigl  \{V \colon  V (\cdot  )    = \max_{a\in \cA} Q (\cdot,a)  ~~\textrm{for some~~}   Q \in \cQ_{\UCB} (h, R, B) \bigr  \}.
\#
In the following, we find a  parameter  $R_T  $  such that   $V_h^t \in \cV_{\UCB}(h, R_T, B_T)$ holds for all $h \in [H]$ and $t \in [T]$, where $B_T$ is specified in \eqref{eq:equation_B_T}. 
Here both $R_T$ and $B_T$ depend on $T$. 
 By \eqref{eq:main_text_func_class} and  \eqref{eq:rewrite_Q_func},   it suffices to set $R_T$ as an upper bound of $\| \hat \theta_h^t \|_{\cH} $ for all $(t, h ) \in [T] \times [H]$.   
 In the following lemma,  we bound the RKHS norm of each $\hat \theta_h^t$. 
 
\begin{lemma}[RKHS Norm of $\hat \theta_h^t$]\label{lemma:thetahat_estimate}
	When $\lambda \geq 1$, for any $(t, h) \in [T] \times [H]$,  $\hat \theta_h^t$ defined  in 
	\eqref{eq:theta_alpha} satisfies 
	\begin{equation*}
	\bigl \| \hat \theta_h^t \bigr  \| _{\cH}  \le   H \sqrt{2  /\lambda \cdot \logdet ( I+     K_h^t  / \lambda  )   }  \leq   2 H \sqrt{   \Gamma_{K}(T, \lambda)},
	\end{equation*}
	where $K_h^t$ is defined in \eqref{eq:define_gram} and   $\Gamma_K(T, \lambda)$ is defined in   \eqref{eq:max_infogain}. 
	\end{lemma}

	\begin{proof}
		See \S\ref{proof:thetahat_estimate} for a detailed proof. 
	\end{proof}

By this lemma, in the sequel, we set $R_T =   2 H \sqrt{  \Gamma_K(T, \lambda)}$. 
To conclude the proof, we show that the sum of the  two terms in  
  \eqref{eq:optim12} is bounded by $\beta \cdot \| \phi(z) \|_{(\Lambda_h^t)^{-1}}$, where we set  $\beta = B_T  $. 
To this end, 
for any two value functions $V, V' \colon \cS\rightarrow \RR$, we define their 
distance as $\textrm{dist}(V, V') = \sup_{x\in \cS} | V(x) - V'(x) | $. 
For any $\epsilon \in (0, 1/e)$, any $B > 0$, and any $h \in [H]$, we let $N_{\textrm{dist} } (  \epsilon; h, B)$ be the $\epsilon$-covering number of $\cV_{\UCB}( h, R_T, B)$ with respect to distance $\textrm{dist}(\cdot, \cdot)$. 
Recall that we define $N_{\infty} (\epsilon;  h, B)$ as the $\epsilon$-covering number of $\cQ_{\UCB}(h, R_T, B)$ with respect to the $\ell_{\infty}$-norm on $\cZ$. 
Note that for any $Q, Q' \colon \cZ \rightarrow\RR$, 
we have 
$$
\sup_{x \in \cS} \Bigl | \max_{a\in \cA} Q(x, a) - \max_{a \in \cA} Q'(x,a ) \Big | \leq  \sup_{(x,a) \in \cS\times \cA} |Q(x,a) - Q'(x, a) |  = \| Q - Q' \|_{\infty}.
$$ 
By \eqref{eq:main_text_v_class} 
we have $N_{\mathrm{dist}} ( \epsilon; h, B) \leq N_{\infty} ( \epsilon ; h, B)  .$   
Then, by applying  Lemma \ref{lem:self_norm_covering} with $\delta = (2T^2H^3  )^{-1}$ and taking a union bound over $h \in [H]$,
we obtain that 
\# \label{eq:apply_uniform}
& \biggl \| \sum_{\tau = 1}^{t-1} \phi(x_h^\tau,a_h^\tau)
\cdot \bigl[ V^{t}_{h+1}(x_{h+1}^\tau) - (\mathbb{P}_hV^{t}_{h+1})(x_h^\tau,a_h^\tau)\bigr ]\biggr \| _{(\Lambda^t_h)^{-1}}^2   \notag \\
& \qquad \leq\sup_{V \in \cV_{\UCB} (h+1, R_T, B_T) }\biggl \| \sum_{\tau = 1}^{t-1} \phi(x_h^\tau,a_h^\tau)
\cdot \bigl[ V(x_{h+1}^\tau) - (\mathbb{P}_h V )(x_h^\tau,a_h^\tau)\bigr ]\biggr \| _{(\Lambda^t_h)^{-1}}^2 
\notag \\
& \qquad \leq 2H^2\cdot  \logdet  ( I+ K_h^t / \lambda)  + 2H^2  t  \cdot (\lambda -1)   + 8  t^2  \epsilon^2 / \lambda  \notag \\
& \qquad \qquad  +4  H^2 \cdot \bigl [ \log N_{\infty }(  \epsilon; h+1, B_T) + \log (2  T^2H ^3   ) \bigr ]
\#
holds uniformly for all   $(t,h)\in [T] \times [H]$   with probability at least $1 -     ( 2 T^2 H^2)^{-2}  $, where we utilize the fact that $V_{h+1}^t \in \cV_{\UCB} ( h+1, R_T, B_T)$. 
Note that we set $\lambda =1 +  1/T$.
Then,  setting $\epsilon$ as  $ \epsilon^*  = H / T$, \eqref{eq:apply_uniform} is further  reduced to 
\#\label{eq:apply_uniform2}
& \biggl \| \sum_{\tau = 1}^{t-1} \phi(x_h^\tau,a_h^\tau)
\cdot \bigl[ V^{t}_{h+1}(x_{h+1}^\tau) - (\mathbb{P}_hV^{t}_{h+1})(x_h^\tau,a_h^\tau)\bigr ]\biggr \| _{(\Lambda^t_h)^{-1}}^2  \notag \\
& \qquad \leq 4 H^2 \cdot \Gamma_K(T, \lambda) + 11 H^2   + 4   H^2 \cdot   \log N_{\infty }(  \epsilon^* ; h+1, B_T) + 8 H^2 \cdot \log (  T H   )   . 
\#
Thus, combining  \eqref{eq:new_bonus}, \eqref{eq:optim12}, \eqref{eq:optim13},  \eqref{eq:optim15}, and \eqref{eq:apply_uniform2}, 
we obtain that 
\#\label{eq:new_optim18}
 & \big| \phi(z)^\top ( \hat \theta_h^t - \overline \theta_h^t ) \big |  \notag \\
 &\qquad  \leq H \cdot  \bigl \{        \bigl [  4   \cdot \Gamma_K(T, \lambda)     + 4    \cdot   \log N_{\infty }(  \epsilon^* ; h+1, B_T) + 8   \cdot \log ( T H   )  + 11  ]^{1/2}  + \sqrt{\lambda}  R_{Q}   \big \}    \cdot b_h^t(z) \notag \\
 & \qquad  \leq H \cdot \bigl [ 8  \cdot \Gamma_K(T, \lambda)     + 8    \cdot   \log N_{\infty }(  \epsilon^* ; h+1, B_T) + 16   \cdot \log (T H  )  + 22  + 2 R_Q^2  \lambda \bigr ] ^{1/2 } \cdot b_h^t(z) \notag \\
 & \qquad \leq B_T \cdot b_h^t(z) = \beta \cdot b_h^t(z) 
\#
holds uniformly for all $(t,h) \in [T] \times [H]$ with probability at least $1 - ( 2T^2H^2)^{-1}$, 
where the second  inequality follows from the elementary inequality 
$\sqrt{a} + \sqrt{b}  \leq \sqrt{2(a^2 + b^2 )} $, and the last inequality follows from the assumption on $B_T$ given in \eqref{eq:equation_B_T}.
  
Finally, by \eqref{eq:new_optim18} and the definition  of the temporal-difference error $ \delta_h^t$ in \eqref{eq:td_error}, we have 
\# \label{eq:new_optim19}
 - \delta_h^t  (z)  =  Q_h^t  (z)  -  \phi(z) ^\top \overline \theta_h^t  \leq \phi(z)^\top ( \hat \theta_h^t - \overline \theta_h^t ) +  \beta \cdot  b_h^t (z) \leq 2 \beta \cdot b_h^t(z).
\# 
In addition, since $Q_{h+1}^t(z)  \leq H - h $ for all $z \in \cZ$, we have $( \TT_h^\star Q_{h+1} ^t )  \leq H- h + 1$.
Hence, we have 
\#\label{eq:optim110}
\delta _h^t  (z)  & =  \phi(z)^\top  \overline \theta_h^t  - \min \bigl  \{  \phi(z)^\top  \hat \theta_h^t + \beta \cdot b_h^t (z), H - h+1 \bigr \}^+  \notag \\
&  \leq \max\bigl  \{ \phi(z)^\top \overline \theta_h^t   -  \phi(z)^\top  \hat \theta_h^t - \beta \cdot b_h^t (z), \phi(z)^\top  \overline \theta_h^t  - (H-h+1)\bigr  \} \leq 0. 
\#
Therefore, combining \eqref{eq:new_optim19} and \eqref{eq:optim110}, we conclude the proof of    Lemma \ref{lemma:ucb_kernel1}. 
\end{proof}

\subsection{Proof of Lemma \ref{lemma:bound_mtg}} \label{proof:lemma_bound_mtg}
\begin{proof}
Following \cite{cai2019provably}, we prove this lemma by showing that   $ \{ \zeta_{t,h}^1,  \zeta_{t,h}^2 \}_{(t,h) \in [T] \times [H]}$ can  be written as a bounded martingale difference sequence with respect to a filtration. 
In particular, we construct the filtration explicitly as follows. For any $(t, h )\in[T]\times[H]$, we define $\sigma$-algebras $\cF_{t,h,1}$ and  $\cF_{t,h,2}$ as follows:
\#\label{eq:define_filtration} 
\begin{aligned}
\cF_{t,h,1} & = \sigma \bigl (	\{(x^\tau_i, a^\tau_i)\}_{(\tau, i)\in [t-1] \times [H]}  \cup \{(x^t_i, a^t_i)\}_{i\in [h]}  \bigr),  \\
\cF_{t, h ,2} & =  \sigma \bigl ( \{ (x^\tau_i, a^\tau_i) \}_{(\tau, i)\in [t-1] \times [H]}  \cup \{(x^t_i, a^t_i)\}_{i\in [h]} \cup \{x^t_{h+1}\} \bigr ) ,
\end{aligned}
\#
where $\sigma( \cdot)$ denotes the $\sigma$-algebra generated by a finite set. 
Moreover, for any $t\in [T]$, $h \in [H]$ and $m\in [2]$, we define the timestep index $\tau(t, h, m) $ as 
	\#\label{eq:time_ordering}
\tau(t,h,m)=(t-1)\cdot 2H+(h-1)\cdot2+m,
\#
which offers an partial ordering over  the triplets $(t, h , m)\in [T]\times[H]\times[2]$.  
Moreover, by the definitions in \eqref{eq:define_filtration}, for any $(t, h, m)$ and $(t', h', m')$ satisfying  $\tau(k,h,m) \le \tau (k',h',m')$, it holds that $\cF_{k,h,m}\subseteq \cF_{k',h',m'}$. Thus, the sequence of $\sigma$-algebras $\{\cF_{t,h,m}\}_{(t,h,m)\in[T]\times[H]\times[2]}$ forms a filtration.

Furthermore, for any $(t, h) \in [T] \times [H]$, 
since both  $Q_h^t$ and $V_h^t$ are  obtained based on the trajectories of the  first $(t-1)$ episodes, 
they are both measurable with respect to $\cF_{t,1, 1}$, which is a subset of $\cF_{t,h,m}$ for all $h \in [H]$ and $m \in [2]$. 
Thus, by \eqref{eq:define_filtration},  $\zeta_{t,h}^1$ defined in \eqref{eq:define_mtg1} and $\zeta_{t,h}^2 $ defined in \eqref{eq:define_mtg2} are measurable with respect to $\cF_{t, h , 1}$ and $\cF_{t,h,2}$, respectively. 
 In addition, note that $a_h^t \sim \pi_{h}^t(\given x_h^t)$ and that $x_{h+1}^t \sim \PP_h (\cdot \given x_h^t, a_h^t)$.  
Thus, we have 
 \#\label{eq:mtg_mean_zero}
 \EE\bigl [  \zeta_{t,h}^1 \biggiven \cF_{t, h-1, 2} \bigr ]  = 0 , \qquad \EE \bigl [ \zeta_{t,h}^2 \biggiven  \cF_{t, h, 1}\bigr ]  = 0,
 \#
 where we identify $\cF_{t, 0, 2}$ with $\cF_{t-1, H, 2}$ for all $t\geq 2$ and let $\cF_{1,0, 2}$ be the empty set. 
 Combining \eqref{eq:time_ordering} and \eqref{eq:mtg_mean_zero}, we can define 
a martingale $\{ M _{t, h, m} \} _{(t, h, m) \in [T] \times [ H] \times [2] }$ indexed 
by $\tau(t, k, m )$, defined in \eqref{eq:time_ordering}, as follows. 
 For any $(t,h, m) \in  [T] \times [ H] \times [2]$, we define  
 \#\label{eq:define_martingale}
 M_{t, h, m} = \biggl \{  \sum_{(s, g, \ell )   } \zeta_{s, g} ^{\ell} \colon  
 \tau( s, g, \ell ) \leq \tau (t, h , m) 
 \biggr \};
 \#
 that is, $M_{t, h,m}$ is the sum of all terms of the form $\zeta_{s,g}^\ell$ defined in \eqref{eq:define_mtg1} or \eqref{eq:define_mtg2} such that its timestep index $\tau (s,g, \ell)$ is no greater than $\tau (t, h,m)$. 
By definition, we have
\#\label{eq:mtg111}
M_{K,H, 2} = \sum_{t=1}^T \sum_{h=1}^{H} ( \zeta_{t,h}^1 +\zeta_{t,h}^2 )  .
\#
 Moreover, since $V_h^t$, $Q_h^t$, $V_{h}^{\pi^t}$, and $Q_h^{\pi^t}$ all takes values in $[0, H]$, we have $| \zeta_{t,h}^1 | \leq 2H$ and $| \zeta_{t,h}^2 | \leq 2 H$ for all $(t,h) \in [T]\times [H]$. 
 This means that the martingale $M_{t, h, m}$ defined in \eqref{eq:define_martingale} has uniformly  bounded differences. 
 Thus,  
applying the  Azuma-Hoeffding inequality \citep{azuma1967weighted} to  $M_{T, H, 2}$ in \eqref{eq:mtg111}, we obtain that 
	\#\label{eq:apply_azuma}
	\PP \biggl( \bigg| \sum_{t=1}^T \sum_{h=1}^{H} ( \zeta_{t,h}^1 +\zeta_{t,h}^2 )   \bigg|>t \biggr) \leq 2\exp\biggl( \frac{-t^2}{16T H^3} \biggr)
	\#
	holds for all $t>0$. 
	Finally, we set  the right-hand side of \eqref{eq:apply_azuma} to $\zeta$ for some $\zeta  \in (0,1)$, which yields 
  $t=\sqrt{16 T H^3 \cdot\log(2 / \zeta )}$. Thus, we obtain  that 
  \$
  \bigg| \sum_{t=1}^T \sum_{h=1}^{H} ( \zeta_{t,h}^1 +\zeta_{t,h}^2 )   \bigg| \leq \sqrt{16T H^3\cdot\log(2 / \zeta )},
  \$
  with probability at least $1- \zeta$, which concludes the proof. 
   \end{proof}

\subsection{Proof of Lemma \ref{lemma:ucb_nn1}} \label{sec:proof_ucb_nn1}

\begin{proof}
The proof of this lemma utilizes the connection between overparameterized neural networks
and NTKs.  Recall that we denote $z = (x,a) $ and $\cZ = \cS \times \cA$. 
Also recall that $(b^{(0)}, W^{(0)})$ is the initial value of the network parameters obtained by the symmetric initialization scheme introduced in \S\ref{sec:dnn}.  
Thus, $f(\cdot; W^{(0)})$ is a  zero function. 
For any $(t, h) \in [T] \times [H]$, since 
$\hat W_h^t$ is the global minimizer of  loss function $L_h^t$ defined in \eqref{eq:nn_loss}, 
we have 
\#\label{eq:nn_opt1}
 L_h^t \bigl ( \hat W_h^t \big)&  = \sum_{\tau =1}^{t-1}  \bigl [   r_h ( x_h^\tau , a_h^\tau) + V _{h+1}^t( x_{h+1} ^\tau  ) -  f   ( x_h^{\tau },  a_h^\tau ; \hat W_h^t    )    \bigr ] ^2 + \lambda \cdot    \bigl \| \hat  W_h^t  - W^{(0)}  \bigr \|_2  ^2 \notag \\
&   \leq L_h^t \bigl (W^{(0)} \bigr ) = \sum_{\tau =1} ^{t-1} \bigl [   r_h ( x_h^\tau , a_h^\tau) + V _{h+1}^t( x_{h+1} ^\tau ) \big ]^2 \leq (H - h + 1) ^2 \cdot (t-1) \leq T H^2 ,
\#
where the second-to-last inequality follows from the facts that $V_{h+1}^t$ is bounded by $H-h$ and that $r_h \in [0,1]$. 
Thus, \eqref{eq:nn_opt1} implies that 
\#\label{eq:nn_param_bound}
 \bigl \| \hat  W_h^t  - W^{(0)}  \bigr  \| _2 ^2  \leq T H^2  / \lambda, \qquad \forall (t,h) \in [T] \times [H]. 
\#
That is, each $\hat W_h^t$ belongs to the Euclidean ball $\cB = \{ W \in \RR^{2md} \colon \| W - W^{(0)} \|_2 \leq H \sqrt{T /\lambda} \}$. 
Here the regularization parameter $\lambda $ is does not depend on $m$ and will be determined later.  
Notice that  the radius of $\cB$ does not depend on $m$. When $m$ is sufficiently large, it can be shown that $f(\cdot , W)$ is close to its linearization, $\hat f(\cdot; W) = \la \varphi(\cdot; W^{(0)}), W - W^{(0) } \ra $, for all $W \in \cB$, where $\varphi(\cdot; W) = \nabla_{W} f(\cdot; W)$.

Furthermore, recall that  the temporal-difference error $\delta_h^t$ is defined as 
$$
\delta_h^t   = r_h + \PP_h V_{h+1}^ t - Q_h^t =  \TT_h ^{\star} Q_{h+1} ^t - Q_h^t .
$$
Under Assumption \ref{assume:opt_closure2}, 
we have 
$ \TT_h ^{\star} Q_{h+1} ^t \in \cQ^\star$ for all $(t, h) \in [T]\times [H]$, 
where $\cQ^\star $ is defined in \eqref{eq:random_feature_class}. 
That is, for all $(t, h) \in [T]\times [H]$,  there exists  a function $\alpha_h^t \colon \RR^d \rightarrow \RR^d $ such that 
\#\label{eq:target_rf_rep}
(\TT_h ^{\star} Q_{h+1} ^t )(z) = \int_{\RR^d } \act'(w^\top z) \cdot z^\top \alpha_h^t (w) ~\ud p_0(w), \qquad \forall (t,h)\in [T]\times [H], \forall z \in \cZ. 
\#
Moreover, it holds that $\| \alpha_h^t \|_{2, \infty} = \sup_{w} \| \alpha_h^t(w) \|_2 \leq R_{Q} H / \sqrt{d}$.

Now we are ready to bound the temporal-difference error $\delta_h^t $ defined in \eqref{eq:td_error}. 
Our proof is decomposed into three steps. 

\vspace{5pt}
{\noindent \textbf{Step I.}} In the first step, we show that, with high probability,  $\TT_h ^{\star} Q_{h+1} ^t $ can be well-approximated by 
the class of linear functions of $\varphi (\cdot; W^{(0)}) $ with respect to the $\ell_{\infty}$-norm. 

Specifically, by Proposition C.1 in \cite{gao2019convergence},
with probability at least $1 - m^{-2}$ over the randomness of initialization,
for any $(t, h) \in [T] \times [H]$, 
there exists  a function 
$\tilde Q_{h}^t \colon \cZ \rightarrow \RR $ that can be written as 
\#\label{eq:approximate_rf}
\tilde Q_{h}^t (z) = \frac{1}{\sqrt{m} } \sum_{j=1}^m \act' \big ( \la W_j ^{(0)},  z \ra \big ) \cdot z^\top \alpha_j,
\#
where $ \| \alpha_j \|_2 \leq R_{Q} /\sqrt{d m}$ for all $j \in [m]$ and $\{W_j^{(0)} \}_{j\in [2m]}$ are the random weights generated in the symmetric  initialization scheme. 
Moreover, $\tilde Q_{h}^t$ satisfies that 
\#\label{eq:rf_approx_error}
 \| \tilde Q_{h}^t - \TT_h ^{\star} Q_{h+1} ^t  \|_{\infty} \leq 10 C_{\act}  R_{Q} H  \cdot    \sqrt{ \log (m TH) / m}   .
\#
Also, 
for any $j \in [2m]$, let $W_j^{(0)}$ and $b_j^{(0)}$ be the $j$-th component of $b^{(0)}$ and $W^{(0)}$, respectively.

Now we show that $\tilde Q_h^t$ in \eqref{eq:approximate_rf} can be written as    $\varphi (\cdot ;  W^{(0)})^\top (\tilde W_h^t - W^{(0)})$ for some $\tilde W_h^t \in \RR^{2md}$. 
To this end, 
we define $\tilde W_h^t = (\tilde W_1, \ldots \tilde W_{2m}) \in \RR^{2md}$ as follows. 
For any $ j \in [m]$, we let $\tilde W_j = W_j^{(0)} + b_j^{(0)} \cdot \alpha_j /\sqrt{2}$, and for any $j \in \{m+1, \ldots, 2m\}$, we let 
$\tilde W_j =W_j^{(0)} + b_j ^{(0) } \cdot \alpha_{j-m} /\sqrt{2} $. 
Then,  
by the symmetric initialization scheme, we have 
\#\label{eq:approximate_rf2}
\tilde Q_{h}^t (z) &  = \frac{1}{\sqrt{2m} }  \sum_{j=1}^m \sqrt{2} \cdot   (b_j^{(0)})^2 \cdot \act'\big(\la W_j^{(0)}, z\ra\big) \cdot z^\top \alpha_j \notag  \\
&  = \frac{1}{\sqrt{2m} }  \sum_{j=1}^{m} 1/ \sqrt{2} \cdot  (b_j^{(0)})^2 \cdot \act'  \big (\la W_j^{(0)}, z\ra \big) \cdot z^\top \alpha_j + \frac{1}{\sqrt{2m} }  \sum_{j=1}^{m} 1/ \sqrt{2} \cdot  (b_j^{(0)})^2 \cdot \act' \big(\la W_j^{(0)}, z\ra  \big)  \cdot z^\top \alpha_{j-m}  \notag \\
& = \frac{1}{\sqrt{2m}} \sum_{j=1}^{2m} b_j^{(0)} \cdot \act' \big ( \la W_j^{(0)}, z\ra \big)\cdot z^\top \big   ( \tilde W_j - W_j^{(0)} \big) = \varphi   (z; W^{(0)}   )^\top  \big ( \tilde W_h^t - W^{(0)}  \big ).
\#
Moreover, since $\| \alpha_j \|_2 \leq R_Q H /\sqrt{dm}$, 
we have $\| \tilde W_h^t - W^{(0)} \|_2 \leq 
R_Q H / \sqrt{d}$.

Therefore, for all $(t, h)\in [T] \times [H]$, we have constructed $\tilde Q_h^t$ 
to be linear in $\varphi(\cdot; W^{(0)})$. 
Moreover, with probability at least $1- m^{-2}$ over the randomness of initialization,  
$\tilde Q_h^t$ is close to $\TT_h^\star Q_{h+1} ^t$ in the sense that \eqref{eq:rf_approx_error}
 holds uniformly for all $(t,h) \in [T] \times [H]$. 
 Thus, we conclude the first step.

\vspace{5pt}
{\noindent \textbf{Step II.}} In the second step, we 
show that $  Q_h^t$ used in Algorithm \ref{algo:neural} 
can be well approximated by functions based on the feature mapping $\varphi(\cdot; W^{(0)})$. 

Recall that the bonus  in  $Q_h^t$ utilizes matrix $\Lambda_h^t$ defined in \eqref{eq:lambda_mat}, 
which involves the neural tangent features $\{ \varphi(\cdot; \hat W_h^\tau) \} _{\tau \in [T]}$.
Similar to $\Lambda_h^t$, we define $\overline \Lambda_h^t$ as 
\#\label{eq:define_barLambda}
\overline \Lambda_h^t = \lambda \cdot I_{2md} + \sum_{\tau = 1}^{t-1}  \varphi   ( x_h^\tau , a_h^\tau ; W^{(0)} ) \varphi   ( x_h^\tau , a_h^\tau ; W^{(0)}  ) ^\top,
\#
 which adopts the same feature mapping $\varphi(\cdot; W^{(0)})$. 
 To simplify the notation, hereafter, we use $\varphi(\cdot )$ to denote $\varphi (\cdot; W^{(0)})$ when its meaning is clear from the text. 
Moreover, for any $(t, h) \in [T] \times [H]$, we define 
the response vector $y_h^t \in \RR^{t-1}$ by letting its   entries be  
 \#\label{eq:ntk_response}
 [ y_h^t ]_{\tau} = r_h (x_h^\tau , a_h^\tau ) + V_{h+1} ^t(x_{h+1}^\tau ) , \qquad \forall \tau \in [t-1]. 
 \#   
 We define the feature matrix $\Phi_h^t \in \RR^{  (t-1) \times 2md} $
by 
\#\label{eq:ntk_featuremat}
\Phi_h^t = \big [ \varphi(x_h^1, a_h^1)^\top , \ldots, \varphi (x_h^{t-1}, a_h^{t-1}) ^\top   \bigr ] ^\top .
\#
Hence, by \eqref{eq:define_barLambda} and \eqref{eq:ntk_featuremat}, we have $\overline \Lambda_h^t = \lambda \cdot I_{2md} + 
(\Phi_h^t )^\top \Phi_h^t $.
Similar to the bonus function $b_h^t$ defined in \eqref{eq:neural_bonus}, we define 
\#\label{eq:ntk_bonus}
\overline b_h^t = \bigl [\varphi(x,a)^\top (\overline \Lambda_h^t ) ^{-1} \varphi(x,a )\bigr ] ^{1/2} =   \| \varphi(x,a) \|_{( \overline \Lambda_h^t)^{-1} }.
\#
Similar to $L_h^t$ defined in   \eqref{eq:nn_loss},  we define another least-squares loss function  
$\overline L_h^t \colon \RR^{2md} \rightarrow \RR$ as 
\#\label{eq:ntk_loss}
\overline L_h^t (W ) & =   \sum_{\tau =1}^{t-1}  \bigl [   r_h ( x_h^\tau , a_h^\tau) + V_{h+1}^t ( x_{h+1} ^\tau) - \bigl   \la \varphi ( x_h^{\tau },  a_h^\tau ) , W - W^{(0)} \big \ra   \bigr ] ^2 + \lambda \cdot    \| W - W^{(0)}  \|_2  ^2 
\#
and let $\overline W_h^t$ be its global minimizer. 
By direct computation,  
$\overline W_h^t$ can be written in closed form as 
\#\label{eq:ntk_solution}
\overline W_h^t = W^{(0)} + (\overline \Lambda_h^t) ^{-1} (\Phi_h^t )^\top y_h^t,
\# 
where $\overline \Lambda_h^t$, $\Phi_h^t$, and $y_h^t$ are defined respectively in  \eqref{eq:define_barLambda}, \eqref{eq:ntk_featuremat}, and 
 \eqref{eq:ntk_response}. 
Similar to \eqref{eq:nn_opt1}, 
utilizing the fact that $\overline L_h^t( \overline W_h^t ) \leq \overline L_h^t(W^{(0)})$, 
 we also have $\| \overline W_h^t - W^{(0)} \|_2 \leq H    \sqrt{T  / \lambda}$. 
Then, in a manner similar to the construction of $Q_h^t$ in Algorithm \ref{algo:neural}, 
we combine $\overline b_h^t$ in \eqref{eq:ntk_bonus} and $\overline W_h^t$ in  
\eqref{eq:ntk_solution} to define $\overline Q_h^t \colon \cZ \rightarrow \RR$ as 
\#\label{eq:ntk_Q}
\overline Q_h^t (x,a)=    
 \min \bigl  \{\varphi (x,a)^\top   ( \overline W_h^t - W^{(0)}   ) + \beta \cdot      \overline b_h^t(x,a)  , H - h +1   \bigr \} ^{+}. 
\#
Note that $\overline Q_h^t$ share the same form as $Q$ in \eqref{eq:define_ntk_q_func}. 
Thus, we have $\overline Q_h^t \in \cQ_{\UCB} (h, H\sqrt{T/\lambda}, B)$ for any $B \geq \beta$.
Moreover, we define $\overline V_h^t(\cdot ) = \max_{a \in \cA} \overline Q_h^t(\cdot, a)$.

In the following, we aim to show that $\overline Q_h^t$ is close to $Q_h^t$ when $m $ is sufficiently large. 
When this is true, $\overline V_h^t$ is also close to $V_h^t$. 
To bound $Q_h^t - \overline Q_h^t$,  
since the  truncation operator is non-expansive,
  by the triangle inequality we have 
\#\label{eq:nn_opt2}
\| Q_h^t - \overline Q_h^t \|_{\infty} \leq  \underbrace{\bigl \| f (\cdot ; \hat W_h^t ) - \varphi(\cdot  )^\top ( \overline W_h^t - W^{(0)}) \bigr \|_{\infty}}_{\dr \textrm{(i)}} + \underbrace{\beta\cdot \| b_h^t  -  \overline b_h^t \|_{\infty} }_{\dr \textrm{(ii)}}. 
\#
Recall that we define $\cB  = \{ W \in \RR^{2md} \colon \| W - W^{(0)} \|_2 \leq H  \sqrt{T/\lambda} \}$. 
To bound the two terms on the right-hand side of \eqref{eq:nn_opt2}, we 
  utilize  the following lemma 
  that quantifies the perturbation of $f(\cdot; W)$ and $\varphi(\cdot; W)$ within  $W\in \cB$.

\begin{lemma} \label{lemma:nn_ntk_approx}
When $TH ^2 = \cO( m \cdot  \log ^{-6} m )$, with probability at least $1 - m^{-2}$ with respect to the randomness of  initialization, for any $W \in \cB$ and any $z\in \cZ$, we have 
\$ 
\bigl | f(z,W ) - \varphi  (z, W^{(0)}   )^\top  ( W - W^{(0)} ) \bigr | & \leq \overline C \cdot T^{2/3} \cdot H^{4/3} \cdot m^{-1/6} \cdot \sqrt{ \log  m}, \qquad \\
 \bigl \| \varphi(z, W) -\varphi  (z, W^{(0)}  )  \bigr  \|_2 \leq  \overline C \cdot (TH^2  &/ m)^{1/6}  \cdot \sqrt{\log m} ,  \qquad 
  \| \varphi(z, W)  \| _2 \leq \overline C.  
 \$
\end{lemma}
\begin{proof}
	
	See \cite{allen2018convergence,gao2019convergence,cai2019neural} for a detailed proof. 
	More specifically, this lemma is obtained from Lemmas F.1 and F.2  in \cite{cai2019neural}, which are  further based on results in \cite{allen2018convergence,gao2019convergence}. 
\end{proof}

By Lemma \ref{lemma:nn_ntk_approx} and triangle inequality, 
Term (i) on the right-hand side of \eqref{eq:nn_opt2} is bounded by 
\#\label{eq:nn_opt3}
\textrm{Term (i)} &  \leq \bigl \| f (\cdot ; \hat W_h^t ) - \varphi(\cdot  )^\top ( \hat W_h^t - W^{(0)}) \bigr \|_{\infty}  + \bigl \| \varphi(\cdot  )^\top ( \hat W_h^t -  \overline W_h^t  ) \bigr \|_{\infty} \notag \\
& \leq \overline C \cdot T^{2/3} \cdot H^{4/3}  \cdot m^{-1/6} \cdot \sqrt{ \log  m} + \overline C \cdot  \bigl \|\hat W_h^t -  \overline{W}_h^t  \bigr \|_2 .
\#
To bound $\bigl \|\hat W_h^t -  \overline{W}_h^t  \bigr \|_2$, notice that $\hat W_h^t$ and $\overline W_h^t$ are the global minimizers of  $L_h^t$ in  \eqref{eq:nn_loss}   and $\overline L_h^t$ in \eqref{eq:ntk_loss}, respectively. 
Thus, by the first-order optimality condition, we have 
\#
\lambda \cdot \bigl (   \hat W_h^t - W^{(0)}  ) & = 
\sum _{\tau =1}^{t-1}  \bigl \{  [  y_h^t ] _{\tau } - f  (z_h^\tau ; \hat W_h^t   ) \bigr \} \cdot \varphi   (z_h^\tau ; \hat W_h^t   ) ,
\label{eq:nn_optimality_cond1} \\
\lambda \cdot   ( \overline  W_h^t - W^{(0)} )  & = 
\sum _{\tau =1}^{t-1} \bigl \{  [  y_h^t ] _{\tau } - \big  \la \varphi(z_h^\tau ;W^{(0)} ), \overline W_h^t - W^{(0)} \big  \ra  \bigr \} \cdot \varphi(z_h^\tau ;W^{(0)} ) , 
\label{eq:nn_optimality_cond2}
\#
where $[  y_h^t ] _{\tau } $ is defined in \eqref{eq:ntk_response} and $z_h^\tau = (x_h^\tau, a_h^\tau )$. 
 In addition, 
by the definition of $\overline \Lambda_h^t$ in \eqref{eq:define_barLambda}, 
\eqref{eq:nn_optimality_cond2}
can be equivalently written as   
\#\label{eq:nn_optimality_cond3}
\overline \Lambda_h^t \bigl (   \overline  W_h^t - W^{(0)}   \bigr ) = \sum _{\tau =1}^{t-1}   [  y_h^t ] _{\tau }   \cdot \varphi(z_h^\tau ;W^{(0)} ) . 
\#
Similarly, for \eqref{eq:nn_optimality_cond1}, by direct computation  we have 
\#\label{eq:nn_optimality_cond4}
\overline \Lambda_h^t  \big  (   \hat  W_h^t - W^{(0)}   \bigr ) & = \sum_{\tau=1}^{t-1}  [ y_h^t]_{\tau} \cdot   \varphi \  (z_h^\tau ; \hat W_h^t   )  \\
& \qquad  + \sum _{\tau = 1}^{t-1} \bigl [ \bigl \la \varphi(z_h^\tau ;W^{(0)} ), \hat W_h^t - W^{(0)} \bigr \ra   \cdot \varphi (z_h^\tau ;W^{(0)} ) - f  (z_h^\tau ; \hat W_h^t   ) \cdot \varphi  (z_h^\tau ;  \hat W_h^t  ) \bigr ] .  \notag
\#
For any $\tau \in [t-1]$, 
we have 
\#\label{eq:nn_optimality_cond5}
& \bigl \la \varphi(z_h^\tau ;W^{(0)} ), \hat W_h^t - W^{(0)} \bigr \ra   \cdot \varphi(z_h^\tau ;W^{(0)} ) - f \big (z_h^\tau ; \hat W_h^t \bigr ) \cdot \varphi  (z_h^\tau ;  \hat W_h^t   ) \notag \\
& \qquad=  \bigl \la \varphi(z_h^\tau ;W^{(0)} ), \hat W_h^t - W^{(0)} \bigr \ra \cdot  \bigl [ \varphi(z_h^\tau ;W^{(0)} ) - \varphi   (z_h^\tau ;  \hat W_h^t  ) \bigr ] \notag \\
& \qquad \qquad + 
\bigl [ \bigl \la \varphi(z_h^\tau ;W^{(0)} ), \hat W_h^t - W^{(0)} \bigr \ra - 
 f  (z_h^\tau ; \hat W_h^t   )   \bigr ] \cdot \varphi  (z_h^\tau ; \hat W_h^t )   .
\#
Thus, applying Lemma \ref{lemma:nn_ntk_approx}
 to \eqref{eq:nn_optimality_cond5}, we have 
 \#\label{eq:nn_optimality_cond6}
 & \bigl \| \bigl \la \varphi(z_h^\tau ;W^{(0)} ), \hat W_h^t - W^{(0)} \bigr \ra   \cdot \varphi(z_h^\tau ;W^{(0)} ) - f   (z_h^\tau ; \hat W_h^t   ) \cdot \varphi   (z_h^\tau ;  \hat W_h^t   )  \bigr \|_2 \notag \\
 & \qquad   \leq  \big \|  \varphi(z_h^\tau ;W^{(0)} ) \big   \| _2 \cdot   \big  \| \hat W_h^t - W^{(0)}   \bigr \|_2 \cdot \bigl \|  \varphi(z_h^\tau ;W^{(0)} ) - \varphi   (z_h^\tau ;  \hat W_h^t  ) \bigr \| \notag \\
 &  \qquad \qquad + \bigl | \bigl \la \varphi(z_h^\tau ;W^{(0)} ), \hat W_h^t - W^{(0)} \bigr \ra - 
 f  (z_h^\tau ; \hat W_h^t   )   \bigr | \cdot  \bigl \| \varphi  (z_h^\tau ; \hat W_h^t  )    \bigr \|_2 \notag \\
 & \qquad  \leq 2 \overline C ^2 \cdot T^{2/3} \cdot H^{4/3} \cdot m^{-1/6} \cdot \sqrt{\log m } \cdot \lambda^{-1/2},
 \#
 where we utilize the fact that $\| \hat W_h^t - W^{(0)} \|_2 \leq H  \sqrt{T    /\lambda }  \leq H \sqrt{T} $.
Then, combining   \eqref{eq:nn_optimality_cond3}, \eqref{eq:nn_optimality_cond4}, and  \eqref{eq:nn_optimality_cond6}, we have 
\#\label{eq:nn_optimality_cond7}
&\bigl \| \overline \Lambda_h^t  \bigl (  \hat W_h^t - \overline W_h^t \bigr ) \bigl \|_2 \notag \\
 &\qquad  \leq \bigg\|  \sum_{\tau=1}^{t-1} [ y_h^t]_{\tau} \cdot \bigl [ \varphi \bigl (z_h^\tau ; \hat W_h^t \bigr )  -   \varphi(z_h^\tau ;W^{(0)} )  \bigr ]  \bigg\|_2  +  2 \overline C ^2 \cdot T^{5/3} \cdot H^{4/3} \cdot m^{-1/6} \cdot \sqrt{\log m } \notag  \\
& \qquad  \leq \overline C \cdot T^{7/6} \cdot H^ {4/3} \cdot m^{-1/6} \cdot \sqrt{\log m } +  2 \overline C ^2 \cdot T^{5/3} \cdot H^{4/3} \cdot m^{-1/6}   \cdot \sqrt{\log m },
\#
where in the last inequality we utilize the fact that $ [ y_h^t]_{\tau }\in [0, H]$.
 When $T$ is sufficiently large, the second term in \eqref{eq:nn_optimality_cond7}
dominates. 
Since the eigenvalues of $(\overline \Lambda_h^t)^{-1}$ are all bounded by $1/\lambda$, we have 
\#\label{eq:bound_w_norms}
\bigl \| \hat W_h^t - \overline W_h^t  \bigr \|_2 \leq \bigl \| (\overline \Lambda_h^t)^{-1} \bigr \| _{\oper} \cdot  \bigl \| \overline \Lambda_h^t  \bigl (  \hat W_h^t - \overline W_h^t \bigr ) \bigl \|_2  \leq 1/ \lambda \cdot  \bigl \| \overline \Lambda_h^t  \bigl (  \hat W_h^t - \overline W_h^t \bigr ) \bigl \|_2 .
\#
In the sequel, we set $\lambda$ as 
\#\label{eq:nn_define_lambda}
\lambda = \overline C^2 \cdot( 1 + 1 / T ) \in \bigl [  \overline C^2 , 2 \overline C^2 \bigr  ]. 
\#

Thus, 
combining \eqref{eq:nn_opt3},  \eqref{eq:nn_optimality_cond7},  
\eqref{eq:bound_w_norms}, and \eqref{eq:nn_define_lambda},  we have 
\#\label{eq:nn_optimality_cond8}
\mathrm{Term~(i)} \leq   4  \cdot T^{5/3}\cdot H^{4/3} \cdot m^{-1/6} \cdot \sqrt{\log m}
\# 
where we use the fact that $\overline C^2 / \lambda \leq 1$.

Furthermore, to bound Term (ii), by the definitions of $b_h^t$ and $\overline b_h^t$, for any $z\in \cZ$, we have 
\#\label{eq:nn_opt4}
\bigl | b_h^t(z) - \overline b_h^t(z) \bigr | &= \Bigl | \sqrt{ \varphi(z; \hat W_h^t ) ^\top (\Lambda_h^t)^{-1}\varphi(z; \hat W_h^t )  }  - \sqrt{ \varphi(z;  W^{(0)} ) ^\top (\overline \Lambda_h^t)^{-1}\varphi(z;   W^{(0)} )  }  \Big | \notag \\
& \leq  \sqrt{  \bigl |\varphi(z; \hat W_h^t ) ^\top (\Lambda_h^t)^{-1}\varphi(z; \hat W_h^t ) -  \varphi(z;  W^{(0)} ) ^\top (\overline \Lambda_h^t)^{-1}\varphi(z;   W^{(0)} )  \bigr |}  ,
\#
where the inequality follows from the elementary inequality $| \sqrt{x} - \sqrt{y} | \leq \sqrt{ | x- y|}$. 
By the triangle inequality
\#\label{eq:nn_opt5}
& \bigl |\varphi(z; \hat W_h^t ) ^\top (\Lambda_h^t)^{-1}\varphi(z; \hat W_h^t ) -  \varphi(z;  W^{(0)} ) ^\top (\overline \Lambda_h^t)^{-1}\varphi(z;   W^{(0)} )  \bigr | \notag \\
 & \qquad \leq \bigl | \bigl [ \varphi(z; \hat W_h^t )  - \varphi(z;  W^{(0)} ) \big ]  ^\top (\Lambda_h^t)^{-1}\varphi(z; \hat W_h^t )   \bigr | + 
\bigl |\varphi(z;  W^{(0)} ) ^\top  \bigl [ (\Lambda_h^t)^{-1} - (\overline \Lambda_h^t)^{-1} \bigr ] \varphi(z; \hat W_h^t )   \bigr |  \notag \\
& \qquad \qquad 
+
\bigl |\varphi(z;  W^{(0)} ) ^\top     (\overline \Lambda_h^t)^{-1}  \bigl [ \varphi(z; \hat W_h^t ) -  \varphi(z;  W^{(0)} ) \bigr ]  \bigr |.
\#
Combining  H\"older's inequality and Lemma \ref{lemma:nn_ntk_approx}, we bound the first term on the right-hand side of \eqref{eq:nn_opt5} by 
\#\label{eq:nn_opt6}
  &\bigl | \bigl [ \varphi(z; \hat W_h^t )  - \varphi(z;  W^{(0)} ) \big ]  ^\top (\Lambda_h^t)^{-1}\varphi(z; \hat W_h^t )   \bigr | \\
  &\qquad \leq \bigl \| \varphi(z; \hat W_h^t )  - \varphi(z;  W^{(0)} ) \bigl\|_2 \cdot \big \|  (\Lambda_h^t) ^{-1} \big \|_{\oper} \cdot \bigl \|\varphi(z; \hat W_h^t )   \bigr \|_2 \leq  \overline C^2 \cdot T^{1/6} \cdot H^{1/3} \cdot m^{-1/6} \cdot \lambda^{-1}\cdot \sqrt{\log m}, \notag
\# 
where $\|  (\Lambda_h^t)^{-1} \|_{\oper} $ is the matrix operator norm of $(\Lambda_h^t)^{-1}$, which is bounded by $1/\lambda$. 
Similarly, for the third term, we also have 
\#\label{eq:nn_opt7}
\bigl |\varphi(z;  W^{(0)} ) ^\top     (\overline \Lambda_h^t)^{-1}  \bigl [ \varphi(z; \hat W_h^t ) -  \varphi(z;  W^{(0)} ) \bigr ]  \bigr | \leq  \overline C ^2 \cdot T^{1/6} \cdot H^{1/3} \cdot m^{-1/6} \cdot \lambda^{-1}\cdot \sqrt{\log m}.
\# 
For the second term,
since both $\Lambda_h^t$ and $\overline \Lambda_h^t$ are invertible, 
we have 
\#\label{eq:nn_opt8}
\bigl \| (\Lambda_h^t)^{-1} - (\overline \Lambda_h^t)^{-1} \bigr \|_{\oper} & = \big \| (\Lambda_h^t)^{-1}   ( \Lambda_h^t - \overline \Lambda_h^t  ) (\overline \Lambda_h^t)^{-1} \big\| _{\oper} \notag \\
& \leq \| (\Lambda_h^t)^{-1}  \|_{\oper} \cdot \| ( \overline \Lambda_h^t)^{-1}  \|_{\oper} \cdot \|\Lambda_h^t - \overline \Lambda_h^t \| _{\oper} \leq \lambda ^{-2}  \cdot \|\Lambda_h^t - \overline \Lambda_h^t \| _{\fro}.  
\#
By direct computation, we have 
\$
& \|\Lambda_h^t - \overline \Lambda_h^t \| _{\fro} \\
& \qquad = \biggl \| \sum_{\tau=1}^{t-1} \big [ \varphi(z_h^\tau; \hat W_h^{\tau+1} ) \varphi(z_h^\tau; \hat W_h^{\tau+1} ) ^\top -  \varphi(z_h^\tau; W^{(0)} )  \varphi(z_h^\tau; W^{(0)} ) ^\top \bigr ] \bigg\|_{\fro} \notag \\
& \qquad \leq \sum_{\tau=1}^{t-1} \bigl \| \varphi(z_h^\tau; \hat W_h^{\tau+1} ) \varphi(z_h^\tau; \hat W_h^{\tau+1} ) ^\top -  \varphi(z_h^\tau; W^{(0)} )  \varphi(z_h^\tau; W^{(0)} ) ^\top  \bigr \|_{\fro} \notag \\
& \qquad \leq \sum_{\tau=1}^{t-1} \bigl \| \bigl [ \varphi(z_h^\tau; \hat W_h^{\tau+1} ) -  \varphi(z_h^\tau; W^{(0)} )  \bigr ]  \varphi(z_h^\tau; \hat W_h^{\tau+1} ) ^\top +   \varphi(z_h^\tau; W^{(0)} ) \big[  \varphi(z_h^\tau; \hat W_h^{\tau+1} )  - \varphi(z_h^\tau; W^{(0)} ) \big] ^\top   \bigr \|_{\fro}.
\$
Hence, by  Lemma \ref{lemma:nn_ntk_approx}   we can bound $ \|\Lambda_h^t - \overline \Lambda_h^t \| _{\fro}$ by 
\#\label{eq:nn_opt9}
\|\Lambda_h^t - \overline \Lambda_h^t \| _{\fro}
& \leq 2 (t-1) \cdot \overline C ^2 \cdot T^{1/6} \cdot H^{1/3} \cdot m^{-1/6} \cdot \sqrt{\log m} \notag \\
& \leq 2  \overline C ^2 \cdot T^{7/6} \cdot H^{1/3} \cdot m^{-1/6} \cdot \sqrt{\log m }.
\#
Hence, combining \eqref{eq:nn_opt8} and \eqref{eq:nn_opt9}, 
the  second term on the right-hand side of \eqref{eq:nn_opt5} can be bounded by 
\#\label{eq:nn_opt10}
& \bigl |\varphi(z;  W^{(0)} ) ^\top  \bigl [ (\Lambda_h^t)^{-1} - (\overline \Lambda_h^t)^{-1} \bigr ] \varphi(z; \hat W_h^t )   \bigr |   \notag \\
&\qquad  \leq \bigl \|\varphi(z;  W^{(0)} ) \bigr \|_2 \cdot \bigl \| \varphi(z; \hat W_h^t )  \bigr \|_2 \cdot \bigl \| (\Lambda_h^t)^{-1} - (\overline \Lambda_h^t)^{-1} \bigr \|_{\oper}   \notag \\
& \qquad  \leq 2  \overline C^4 \cdot T^{7/6} \cdot H^{1/3} \cdot m^{-1/6} \cdot \lambda^{-2} \cdot \sqrt{\log m}.
\#
 Notice that $\lambda$ defined in \eqref{eq:nn_define_lambda}
satisfies that $\lambda \geq \overline C^2$. 
Thus, combining \eqref{eq:nn_opt4}-\eqref{eq:nn_opt7}, and \eqref{eq:nn_opt10}, we have 
\#\label{eq:nn_opt11}
\bigl | b_h^t(z) - \overline b_h^t(z) \bigr | \leq   2   \cdot T^{7/12} \cdot H^{1/6} \cdot m^{-1/ 12} \cdot (\log m)^{1/4}, \qquad \forall (t, h) \in [T] \times [H],
\#
which establishes the second inequality in \eqref{eq:nn_optimism}. 
 Finally, combining \eqref{eq:nn_opt2}, \eqref{eq:nn_optimality_cond8}, and \eqref{eq:nn_opt11}, we conclude that 
\$
\| Q_h^t - \overline Q_h^t \|_{\infty} \leq 4  \cdot T^{5/3}\cdot H^{4/3} \cdot m^{-1/6} \cdot \sqrt{ \log  m} + 2     \beta \cdot T^{7/12} \cdot H^{1/6} \cdot m^{-1/ 12} \cdot (\log m)^{1/4}. 
\$
Note that $\beta > 1$. When $m = \Omega ( \beta^{12} \cdot T^{13} \cdot H^{14} \cdot (\log m )^3 )$, the second term in the above inequality is the dominating  term. 
Thus, we have 
\#\label{eq:approx_sup_norm}
\sup_{x\in \cS} \bigl | V_h^t (x) - \overline V_{h}^t (x) \bigr |     \leq \| Q_h^t - \overline Q_h^t \|_{\infty} & \leq 4  \beta \cdot T^{7/12} \cdot H^{1/6} \cdot m^{-1/ 12} \cdot (\log m)^{1/4}.
\#
This concludes the second step. 

\vspace{5pt}
{\noindent \textbf{Step III.}} 
In the last step, we establish optimism by comparing $\varphi(\cdot)^\top (\overline W_h^t - W^{(0)} )$ and the function 
$\tilde Q_h^t $ defined in \eqref{eq:approximate_rf}, where $\varphi(\cdot)$ denotes $\varphi(\cdot; W^{(0)})$.
By the definition of $\overline \Lambda_h^t $ in \eqref{eq:define_barLambda}, we have 
\$
 \tilde W_h^t - W^{(0)} = (\overline \Lambda _h^t) ^{-1} \cdot \bigl [ \lambda \cdot \bigl ( \tilde W_h^t - W^{(0)}  \big ) + (\Phi_h^t)^\top \Phi_h^t \bigl ( \tilde W_h^t - W^{(0)}  \big ) \bigr ] ,
\$
where $\tilde W_h^t$ is given in \eqref{eq:approximate_rf2}. 
Hence, combining \eqref{eq:ntk_solution}, we have 
\# \label{eq:nn_op12}
\overline W_h^t -  \tilde W_h^t = - \lambda \cdot (\overline \Lambda _h^t) ^{-1}  \bigl ( \tilde W_h^t - W^{(0)}  \big ) + (\overline \Lambda _h^t) ^{-1}   (\Phi_h^t)^\top   \bigl [ y_h^t - \Phi_h^t \bigl ( \tilde W_h^t - W^{(0)} \bigr ) \bigr ].
\#
Thus, for any $z \in \cZ$, 
by \eqref{eq:nn_op12} we have 
\#\label{eq:nn_opt13}
& \varphi(z)^\top \bigl (\overline W_h^t -  \tilde W_h^t \bigr ) \notag \\
& \qquad =\underbrace{  - \lambda \cdot \varphi(z)^\top (\overline \Lambda _h^t) ^{-1} \cdot \bigl ( \tilde W_h^t - W^{(0)}  \big )}_{\dr (iii)} + \underbrace{ \varphi(z)^\top (\overline \Lambda _h^t) ^{-1}   (\Phi_h^t)^\top   \bigl [ y_h^t - \Phi_h^t \bigl ( \tilde W_h^t - W^{(0)} \bigr ) \bigr ] }_{\dr (iv)}.
\#
For Term (iii) on the right-hand side of \eqref{eq:nn_opt13}, by the Cauchy-Schwarz inequality, we have 
\#\label{eq:nn_opt14}
& \bigl |\lambda \cdot \varphi(z)^\top (\overline \Lambda _h^t) ^{-1} \cdot \bigl ( \tilde W_h^t - W^{(0)}  \big ) \bigr |  \leq \lambda \cdot \bigl \| \tilde W_h^t - W^{(0)} \bigr \|_2 \cdot \bigl \| (\overline \Lambda _h^t) ^{-1} \varphi(z) \bigr \|_2 \notag \\
& \qquad \leq \lambda \cdot R_{Q }  H / \sqrt{d} \cdot \sqrt{ \varphi(z)^\top (\overline \Lambda _h^t) ^{-1} (\overline \Lambda _h^t) ^{-1} \varphi(z)  } \leq R_{Q} H \cdot  \sqrt{\lambda / d} \cdot \overline b_h^t (z).
\#
For Term (iv) in \eqref{eq:nn_opt13}, recall that $\tilde Q_h^t (z) = \varphi(z)^\top (\tilde W_h^t - W^{(0)})$. 
To simplify the notation, let $q^\star \in \RR^{t-1}$ denote the vector whose $\tau$-th entry is $(\TT_h^\star Q_{h+1}^t ) (x_h^\tau , a_h^\tau ) $ for any $\tau\in [t-1]$. 
Then, by \eqref{eq:rf_approx_error}, for any $\tau \in [t-1]$,   the $\tau$-th entry of 
$\Phi_h^t  ( \tilde W_h^t - W^{(0)}   ) $ satisfies 
\$
\bigl |  [\Phi_h^t \bigl ( \tilde W_h^t - W^{(0)} \bigr ) ]_{\tau } -  [ q^\star ]_{\tau } \big |  & = 
\bigl | \bigl  [\Phi_h^t \bigl ( \tilde W_h^t - W^{(0)} \bigr ) \bigr  ]_{\tau } - (\TT_h^\star Q_{h+1}^t ) (x_h^\tau , a_h^\tau ) \big | \notag \\
& \leq 10 C_{\act} \cdot R_Q H \cdot \sqrt{\log (mTH)/m}. 
\$
Moreover, for any $\tau \in [t-1]$, the $\tau$-th entry of  $( y_h^t - q^\star )$ can be written as 
\#\label{eq:nn_opt15} 
[y_h^t]_{\tau} - [  q^\star ]_{\tau}  &   = r_h (x_h^\tau , a_h^\tau ) + V_{h+1} ^t(x_{h+1}^\tau ) - \varphi(x_h^\tau, a_h^\tau)^\top \overline \theta_h^t = r_h (x_h^\tau , a_h^\tau ) + V_{h+1} ^t(x_{h+1}^\tau ) - (\TT_h^\star Q_{h+1}^t ) (x_h^\tau , a_h^\tau ) \notag \\
& = V_{h+1} ^t(x_{h+1}^\tau ) -  (\PP_h  V_{h+1}^t ) (x_h^\tau , a_h^\tau ).
\#
Then, by the triangle inequality and \eqref{eq:nn_opt15}, we have  
\#\label{eq:nn_opt16}
& \bigl | \varphi(z)^\top (\overline \Lambda _h^t) ^{-1}   (\Phi_h^t)^\top   \bigl [ y_h^t - \Phi_h^t \bigl ( \tilde W_h^t - W^{(0)} \bigr ) \bigr ] \bigl | \notag \\
& \qquad \leq \bigl | \varphi(z)^\top (\overline \Lambda _h^t) ^{-1}   (\Phi_h^t)^\top   \bigl [ y_h^t - q^\star \bigr ] \bigl | + \bigl | \varphi(z)^\top (\overline \Lambda _h^t) ^{-1}   (\Phi_h^t)^\top \bigl [ q^\star -  \Phi_h^t \bigl ( \tilde W_h^t - W^{(0)} \bigr ) \bigr ] \bigr | \notag \\
& \qquad \leq \| \varphi(z) \|_{(\overline \Lambda_h^t)^{-1} } \cdot    \biggl \| \sum_{\tau = 1}^{t-1} \varphi(x_h^\tau,a_h^\tau)
\cdot \bigl[ V^{t}_{h+1}(x_{h+1}^\tau) - (\mathbb{P}_hV^{t}_{h+1})(x_h^\tau,a_h^\tau)\bigr ]\biggr \| _{(\overline \Lambda^t_h)^{-1}}  \notag \\
&\qquad \qquad + 10 C_{\act} \cdot R_Q H  \cdot \sqrt{\log (mTH)/m}  \cdot \| \varphi(z) \|_{(\overline \Lambda_h^t)^{-1} } . 
\#
Recall that we have shown in \textbf{Step  II } that, with probability at least $1-m^2$ with respect to the randomness of initialization, \eqref{eq:approx_sup_norm} holds for all $(t,h) \in [T] \times [H]$. 
To simplify the notation, we denote  
\#\label{eq:define_err}
\texttt{Err} = 4  \beta \cdot T^{7/12} \cdot H^{1/6} \cdot m^{-1/12} \cdot (\log m)^{1/4} . 
\#
Moreover, we define functions $\Delta V _1= V_{h+1}^t - \overline V_{h+1}^t$ and $\Delta V_2 =  \PP_h (V_{h+1}^t - \overline V_{h+1}^t)$. 
Then  \eqref{eq:approx_sup_norm} implies that 
$\sup_{x \in S } | \Delta V_1(x) | \leq \texttt{Err}$ and $\sup_{z \in \cZ } | \Delta V_2 (z) | \leq \texttt{Err}$. 
By the elementary inequality $(a+b)^2 \leq 2 a^2 + 2b^2$, we have 
\#\label{eq:nn_opt17}
 & \biggl \| \sum_{\tau = 1}^{t-1} \varphi(x_h^\tau,a_h^\tau)
	\cdot \bigl[ V^{t}_{h+1}(x_{h+1}^\tau) - (\mathbb{P}_hV^{t}_{h+1})(x_h^\tau,a_h^\tau)\bigr ]\biggr \| _{(\overline \Lambda^t_h)^{-1}}  ^2 \notag \\
	& \qquad   \leq 2 \underbrace{ \biggl \| \sum_{\tau = 1}^{t-1} \varphi(x_h^\tau,a_h^\tau)
	\cdot \bigl[ \overline V^{t}_{h+1}(x_{h+1}^\tau) - (\mathbb{P}_h \overline V^{t}_{h+1})(x_h^\tau,a_h^\tau)\bigr ]\biggr \| _{(\overline \Lambda^t_h)^{-1}}^2 }_{\dr (v)} \notag \\
& \qquad \qquad  + 2 \biggl \| \sum_{\tau = 1}^{t-1} \varphi(x_h^\tau,a_h^\tau)
	\cdot \bigl[ \Delta V_1  (x_{h+1}^\tau) -  \Delta V_2 (x_h^\tau,a_h^\tau)\bigr ]\biggr \| _{(\overline \Lambda^t_h)^{-1}} ^2  \notag \\
	& \qquad \leq 2 \cdot \mathrm{Term~(v)}  + 8 \cdot  \texttt{Err}^2   \cdot T^2,
\#
where the last inequality follows from  
the fact that 
\$
& \biggl \| \sum_{\tau = 1}^{t-1} \varphi(x_h^\tau,a_h^\tau)
\cdot \bigl[ \Delta V_1  (x_{h+1}^\tau) -  \Delta V_2 (x_h^\tau,a_h^\tau)\bigr ]\biggr \| _{(\overline \Lambda^t_h)^{-1}} ^2  \leq 4 \texttt{Err}^2   \cdot   \biggl \| \sum_{\tau = 1}^{t-1} \varphi(x_h^\tau,a_h^\tau)
 \biggr \| _{(\overline \Lambda^t_h)^{-1}} ^2 \notag \\
 &\qquad \leq 4\cdot  \texttt{Err}^2 \cdot (t-1) \cdot  \lambda ^{-1}  \cdot \sum_{\tau =1}^{t-1} \| \varphi(x_h^\tau,a_h^\tau) \|_2 ^2  \leq 4  \cdot  \texttt{Err}^2 \cdot (t-1)^2 \cdot \overline C^2 \cdot \lambda^{-1} \leq 4 \cdot   \texttt{Err}^2 \cdot T^2. 
\$
  Here the second-to-last inequality follows from Lemma \ref{lemma:nn_ntk_approx} and the definition of $\lambda$.

   Recall that we define $\overline b_h^t(z) = \| \varphi(z) \|_{(\overline \Lambda _h^t)^{-1}}$. 
   Combining \eqref{eq:nn_opt15},   \eqref{eq:nn_opt16}, and \eqref{eq:nn_opt17},  we have 
   \#\label{eq:nn_opt17}
   & \bigl | \varphi(z)^\top (\overline \Lambda _h^t) ^{-1}   (\Phi_h^t)^\top   \bigl [ y_h^t - \Phi_h^t \bigl ( \tilde W_h^t - W^{(0)} \bigr ) \bigr ] \bigl | \notag \\
   & \qquad \leq  \overline b_h^t(z) \cdot\bigl [10 C_{\act} \cdot R_Q H \cdot \sqrt{\log (mTH)/m}  + \sqrt{ 2 \cdot \mathrm{Term~(v)}} + 2\sqrt 2 \cdot \texttt{Err} \cdot T \bigr ]  \notag \\
   & \qquad  \leq  \overline b_h^t(z) \cdot\bigl [ R_Q H  + \sqrt{ 2 \cdot \mathrm{Term~(v)}}   \bigr ],
   \#
   where  we apply the elementary inequality $\sqrt{ a+b} \leq \sqrt{a} + \sqrt{b}$. 
   Here in the last inequality we let $m$ be sufficiently large such that 
   $$
   10 C_{\act} \cdot R_Q H \cdot \sqrt{\log (mTH)/m}   + 2\sqrt 2 \cdot \texttt{Err} \cdot T \leq R_Q H .
   $$
   
In the following, we aim to bound Term (v) in \eqref{eq:nn_opt17} 
by combining the concentration of the self-normalized stochastic process  and uniform concentration. 
To characterize the function class that contains each $\overline V_{h}^t$, 
we define 
$\tilde \varphi \colon \cZ \rightarrow \RR$  by $\tilde \varphi (z) = \varphi(z) / \overline C$.
Then, conditioning  on the event where the statements in Lemma \ref{lemma:nn_ntk_approx} are true,
we have $\| \tilde \varphi(z) \|_2 \leq 1 $ for all $z \in \cZ$. 
Furthermore, we define a kernel function $\tilde K \colon \cZ \times \cZ \rightarrow \RR$ by letting 
 $\tilde K (z, z') = \tilde \varphi(z) ^\top \tilde \varphi(z')$ for all $z, z' \in \cZ$.  
 That is, $\tilde K$ is  the normalized version of the empirical NTK $K_m$.
 By construction, $\tilde K$ is a bounded kernel such that $\sup_{z\in \cZ} \tilde K (z,z) \leq 1$. 
 We can also consider the maximal information gain in \eqref{eq:maintext_infogain} for $\tilde K$ and $K_m$. 
 These two quantities are linked via 
 \#\label{eq:link_infogain_nn}
 \Gamma_{\tilde K} (T, \sigma ) = \Gamma_{ K_m }\bigl (T, \overline C^2 \sigma \bigr ), \qquad \forall  \sigma > 0.
 \#

Furthermore, we define $\tilde \lambda =    \lambda / \overline C^{2} $ and $\tilde \Lambda_h^t =    \overline \Lambda_h^t/ / \overline C^{2}$ for all $(t, h ) \in [T] \times [H]$. 
By the definition of $\lambda $ in \eqref{eq:nn_define_lambda}, we have $\tilde \lambda = 1 + 1/T \in [1,2]$. 
Moreover, by \eqref{eq:define_barLambda} we have 
\$
\tilde \Lambda _{h}^t = \tilde \lambda  \cdot I_{2md}+ \sum_{ \tau = 1}^{t-1} \tilde \varphi(x_h^\tau, a_h^\tau ) \tilde \varphi(x_h^\tau, a_h^\tau ) ^\top. 
\$
Since $\tilde \lambda > 1$, $\tilde \Lambda_h^t$ is  an invertible matrix and the eigenvalues of $( \tilde \Lambda _h^t )^{-1}$ are all bounded above by one. 

Using $\tilde \varphi$ and $\tilde \Lambda_h^t$,
we rewrite each $\overline Q_h^t$ as follows. 
For $\overline W_h^t$ defined in \eqref{eq:ntk_solution},
we have 
\#\label{eq:rewrite_function}
 \varphi (x,a)^\top  \bigl ( \overline W_h^t - W^{(0)} ) = \overline C \cdot  \tilde \varphi (x,a)^\top  \bigl ( \overline W_h^t - W^{(0)} \bigr ),
\#
where $\overline C \cdot \| \overline W_h^t - W^{(0)} \|_2 \leq \overline C \cdot   H \sqrt{T    /\lambda } \leq  H \sqrt{T    }$ since $\lambda \ge (\overline C)^2$.
Meanwhile,  we also  have 
\#\label{eq:rewrite_bonus}
\overline b_h^t (z) = \| \varphi(z) \|_{(\overline \Lambda _h^t)^{-1}} = \bigl [ \tilde \varphi (z) ^\top \bigl (   \tilde \Lambda_h^t \bigr ) ^{-1} \tilde \varphi (z) \bigr ] ^{1/2} .
\#
Thus, combining \eqref{eq:rewrite_function} and \eqref{eq:rewrite_bonus},  $\overline Q_h^t$ defined in \eqref{eq:ntk_Q}
can be written equivalently as 
\$
\overline Q_h^t(z) = \min \bigl \{  \tilde \varphi(z)^\top \overline \vartheta_{h}^t + \beta \cdot  \bigl \| \tilde \varphi(z) \bigr \|_{(\tilde \Lambda_h^t)^{-1}}, H - h +1\bigr \}^{+},
\$
where $\overline \theta_h^t = \overline C \cdot ( \overline W_h^t - W^{(0)})$, which 
satisfies 
$\| \overline \vartheta_h^t \| _2 \leq H \sqrt{T}$. 

Let $\cD $ be a finite subset of $\cZ$ with no more than $T$ elements. 
For any fixed $\cD$, 
we define 
\#\label{eq:tilde_lambda_D}
\tilde \Lambda_{\cD} =\tilde \lambda \cdot I_{2dm} +  \sum_{z \in \cD } \tilde \varphi (z) \tilde \varphi(z)^\top  \in \RR^{2md\times 2md}.
\# 
For any $h \in [H]$, $R , B> 0$, we let $\tilde Q_{\UCB} (h, R, B)$  consists of   functions that take the form of 
\$
Q(\cdot   ) =\min \bigl \{  \tilde \varphi (\cdot) ^\top \vartheta  + \beta \cdot\bigl \| \tilde \varphi (\cdot )\bigl\|_{ (\tilde \Lambda_{\cD} ) ^{-1} } ; H - h +1 \bigr \}^+,
\$
for some $\vartheta \in \RR^{2md} $ with $\| \vartheta \|_2 \leq R$ and some $\cD \subseteq \cZ$. 
Then $\tilde Q_{\UCB} (h, R, B)$ corresponds to the function class in \eqref{eq:main_text_func_class} with the kernel being $\tilde K $. 
Moreover, we define $\tilde \cV_{\UCB}(h, R, B)$ as 
\$
\tilde \cV_{\UCB}(h, R, B) = \bigl \{ V \colon V(\cdot ) = \max_{a} Q(\cdot, a)~~\textrm{for some}~~Q \in \tilde \cQ_ {\UCB}(h, R, B)\bigr \}.
\$

By definition, for all $h \in [H]$ and any $R, B> 0$,  we have that 
$
  \cQ_{\UCB} (h, R, B) = \tilde \cQ_{\UCB}(h, \overline C  R, B) $.
Meanwhile,  since $ (\overline C )^2\leq  \lambda \leq 2 (\overline C )^2$, for all $R > 0$,  
we have
\#\label{eq:contain_classes}
\cQ_{\UCB} (h, R, B) \subseteq \tilde \cQ_{\UCB} (h, R \sqrt{\lambda} , B)
 \subseteq 
 \cQ_{\UCB} ( h, \sqrt{2} R, B).
 \#
 Recall that we define $R_T =  H  \sqrt{2T/\lambda}$ and let $N_{\infty} (\epsilon; h, B)$ denote the $\epsilon$-covering number of 
 $\cQ_{\UCB}(h, R_T, B)$ with respect to the $\ell_{\infty}$-norm on $\cZ$. 
 Moreover, hereafter, we denote $\epsilon^* = H/ T$ and set $B = B_T$ which satisfy \eqref{eq:equation_B_T_nn}.
 Since we set $\beta = B_T$ in Algorithm \ref{algo:neural}, 
 it holds  for  all $(t, h) \in [T]\times [H] $ that 
$$
\overline Q_h^t\in \tilde \cQ_{\UCB}( h, H \sqrt{T} , B) \subseteq  \cQ_{\UCB} ( h,  R_T, B), \qquad 
\overline V_h^t \in \tilde \cV_{\UCB} (h, H\sqrt{T}, B).
$$

Now, to  bound Term (v) in \eqref{eq:nn_opt17}, 
similar to the analysis the proof of Lemma \ref{lemma:ucb_kernel1},
 we apply the concentration of self-normalized stochastic process and uniform concentration over $\tilde \cV_{\UCB} (h, H\sqrt{T} , B_T)$. 
 Specifically, similar to \eqref{eq:apply_uniform} and \eqref{eq:apply_uniform2},
 with probability at least $1- (2T^2H^2)^{-1}$, 
 we have 
 \#\label{eq:nn_opt18}
 \mathrm{Term~(v)} & 
 = \biggl \| \sum_{\tau = 1}^{t-1} \varphi(x_h^\tau,a_h^\tau)
 \cdot \bigl[ \overline V^{t}_{h+1}(x_{h+1}^\tau) - (\mathbb{P}_h \overline V^{t}_{h+1})(x_h^\tau,a_h^\tau)\bigr ]\biggr \| _{(\overline \Lambda^t_h)^{-1}}^2  \notag \\
 & =  \biggl \| \sum_{\tau = 1}^{t-1} \tilde \varphi(x_h^\tau,a_h^\tau)
 \cdot \bigl[ \overline V^{t}_{h+1}(x_{h+1}^\tau) - (\mathbb{P}_h \overline V^{t}_{h+1})(x_h^\tau,a_h^\tau)\bigr ]\biggr \| _{(\tilde \Lambda^t_h)^{-1}}^2  \notag \\
 & \leq 4H^2 \cdot \Gamma _{\tilde K} (T, \tilde \lambda ) + 11 H^2 + 4 H^2 \cdot \log N_{\infty} (\epsilon^*, h+1 , B_T)   + 8H^2 \cdot \log (TH).
 \#
 Thus, combining \eqref{eq:nn_opt13}, \eqref{eq:nn_opt14},  \eqref{eq:nn_opt17}, and  \eqref{eq:nn_opt18}, 
 we obtain that 
 \$
 & \bigl | \varphi(z) ^\top  (   \overline W_h^t - \tilde W_h^t) \big |  \notag \\
 &\qquad 
 \leq \bigl  |\mathrm{Term~(iii)} \bigr | + \bigl | \mathrm{Term~(iv)} \bigr  |  \leq   \bigl [  R_Q H  + \sqrt{ 2 \cdot \mathrm{Term~(v)}}   +  R_{Q}  H  \cdot \sqrt{\lambda /d} \bigr ] \cdot \overline b_h^t (z)  \notag \\
 & \qquad \leq H \cdot  \bigl \{ \bigl [  8    \Gamma _{ K_m } (T,  \lambda ) + 22   + 8   \cdot \log N_{\infty} (\epsilon^*, h+1 , B_T)   + 16  \cdot \log (TH) \bigr ]^{1/2}  + R_Q \cdot (1 +   \sqrt{\lambda / d} ) \bigr \} \cdot \overline b_h^t (z) .
 \$
 Using the elementary inequality $a +b  \leq \sqrt{2(a^2+b^2)}$, we have 
\$
 & \bigl | \varphi(z) ^\top  (   \overline W_h^t - \tilde W_h^t) \big |  \notag \\
 & \qquad \leq H \cdot \bigl [ 16  \Gamma _{ K_m } (T,  \lambda )  +  16   \cdot \log N_{\infty} (\epsilon^*, h+1 , B_T)   + 32  \cdot \log (TH)  + 2 R_{Q} ^2   \cdot  (1 + \sqrt{\lambda / d})^2   \bigr  ]^{1/2}\cdot \overline b_h^t (z) \notag \\
 & \qquad  \leq H \cdot \bigl [ 16  \Gamma _{ K_m } (T,  \lambda ) + 16   \cdot \log N_{\infty} (\epsilon^*, h+1 , B_T)   + 32  \cdot \log (TH)  +  4   R_{Q} ^2   \cdot  (1 +  \lambda / d)     \bigr  ]^{1/2}\cdot \overline b_h^t (z).
\$
By the choice of $B_T$ in \eqref{eq:equation_B_T_nn}, we have that
\$
\bigl | \varphi(z) ^\top  (   \overline W_h^t - \tilde W_h^t) \big | =  \bigl | \varphi(z)^\top ( \overline W_h^t   -  W^{(0)  } ) - \tilde Q_h^t (z) \big|   \leq \beta \cdot \overline b_h^t (z) 
\$
holds 
simultaneously 
for all $(t, h) \in [T] \times [H]$ and $z \in \cZ$ with probability at least $1- (2T^2H^2)^{-1}$.

 Thus, combining this with \eqref{eq:approximate_rf} and  \eqref{eq:rf_approx_error}, 
 we have 
 \# \label{eq:nn_opt210}
 \bigl | \varphi(z)^\top ( \overline W_h^t   -  W^{(0)  } ) -   \TT_h^\star Q_{h+1}^t (z) \big|  \leq  \beta \cdot \overline b_h^t(z) + 10 C_{\act}\cdot R_{Q} H \cdot    \sqrt{ \log (m TH) / m} .
 \# 
 By the definition of $\overline Q_h^t $ in \eqref{eq:ntk_Q}, 
 we have 
 \# \label{eq:nn_opt211}
 \overline Q_h^t  (z)  -  \TT_h^\star Q_{h+1}^t (z)    & \leq \varphi(z)^\top ( \overline W_h^t   -  W^{(0)  } ) -   \TT_h^\star Q_{h+1}^t (z)  +\beta \cdot \overline b_h^t(z) \notag \\
 & \leq 2 \beta \cdot  \overline b_h^t(z) + 10 C_{\act}\cdot R_{Q} \cdot    \sqrt{ \log (m TH) / m} .
 \# 
 Moreover, since $\TT_h^\star Q_{h+1} ^t   \leq H- h + 1$, by \eqref{eq:nn_opt210}
 we have 
 \# \label{eq:nn_opt22}
 \TT_h^\star Q_{h+1}^t (z) - \overline Q_h^t  (z) 
 & = \TT_h^\star Q_{h+1}^t (z) - \min \bigl  \{\varphi (x,a)^\top  \bigl ( \overline W_h^t - W^{(0)} \bigr ) + \beta \cdot      \overline b_h^t(x,a)  , H - h +1   \bigr \}^+ \notag \\
 & = \max \bigl \{ \TT_h^\star Q_{h+1}^t (z)- \varphi (z)^\top  \bigl ( \overline W_h^t - W^{(0)} \bigr ) -  \beta \cdot      \overline b_h^t(z), 0 \bigr \}^+ \notag \\
 & \leq  10 C_{\act}\cdot R_{Q} \cdot    \sqrt{ \log (m TH) / m}. 
 \#
  Let $\logt$ denote $T^{7/12} \cdot H^{1/12} \cdot m^{-1/12} \cdot (\log m)^{1/4}$.
 When $m$ is sufficiently large, it holds that 
 \$
 10 C_{\act}\cdot R_{Q} \cdot    \sqrt{ \log (m TH) / m} \leq \logt \leq \beta \cdot \logt.
 \$
 Meanwhile, combining the definition  of the TD error $ \delta_h^t$ in \eqref{eq:td_error} and \eqref{eq:approx_sup_norm}, we have 
 \#\label{eq:nn_opt23}
  &\big| \delta _h^t  (z) - \bigl [\TT_h^\star Q_{h+1}^t (z) - \overline Q_h^t  (z)] \big |    =  \bigl| Q_h^t(z) - \overline Q_h^t  (z) \big|   \notag \\
  & \qquad \leq  4\beta \cdot   T^{7/12} \cdot H^{1/12} \cdot m^{-1/12} \cdot  (\log  m) ^{1/4} .
  \#
 
Finally, combining \eqref{eq:nn_opt211}, \eqref{eq:nn_opt22}, and \eqref{eq:nn_opt23}, we establish
  that, with probability at least $1- (2T^2H^2)^{-1}$, 
  \$
  \delta _h^t  (z) & \leq [\TT_h^\star Q_{h+1}^t (z) - \overline Q_h^t  (z)] + 4 \beta \cdot \logt \leq 5 \beta \cdot \logt \\
    \delta _h^t  (z) & \geq [\TT_h^\star Q_{h+1}^t (z) - \overline Q_h^t  (z)] -  4 \beta\cdot \logt \geq - 2\beta \cdot \overline b_h^t (z) - 5 \beta \cdot \logt 
  \$
  hold for all $(t, h) \in [T] \times [H]$ simultaneously.
  Finally, 
  combining this with 
   \eqref{eq:nn_opt11}, we have 
\$
 -2 \beta\cdot b_h^t - 9 \beta \cdot \logt \leq -2 \beta\cdot \overline b_h^t - 5 \beta \cdot \logt \leq \delta _h^t  (z)  \leq 5 \beta \cdot \logt,
\$ 
which, together with  \eqref{eq:nn_opt11}, 
concludes the proof of    Lemma \ref{lemma:ucb_nn1}. 
\end{proof}

%% file: proof_aux.tex


\section{Covering Number and Effective Dimension} \label{sec:aux_lem}

In this section, we  present results on the covering number of the class 
of value functions that we study and the effective dimension of the 
corresponding RKHS. Both of these results play a key role in establishing 
our regret bounds.
 
\subsection {Covering Number of the Classes  of Value Functions} \label{sec:covering}


 For any $R, B >0$, any $h \in [H]$,  and fixed $\cD$,  we define 
$ \cQ_{\UCB} (h, R, B) $ as the function class that  contains functions on $\cZ$  that  
take  the following form:
 \#\label{eq:some_Q}
Q(z )=  \min  \bigl \{ \theta (z ) + \beta \cdot \lambda^{-1/2}   \bigl [  K(z,z) - k _t(z) ^\top (K_t + \lambda I  )^{-1} k_t(z) \bigr ]  ^{1/2},   H - h  + 1  \bigr \}^{+},
\#
where 
 $\theta  \in \cH$ satisfies $\| \theta  \|_{\cH } \leq R $, $\beta \in [0, B]$, $h \in [H]$, and $\cD =\{ z^\tau = (x^\tau, a^\tau),\}_{\tau \in [t]} $ is a finite subset of $\cZ$ with $t$ elements, where $t \leq T$.
 Here $T$ is the total number of  the episodes. 
  Moreover, $K_t \in \RR^{t\times t}$ and $k_t \colon \cZ \rightarrow \RR^t$ are defined similarly as  in  \eqref{eq:define_gram} based on state-action pairs in $\cD$, that is, 
 \$
  K_t   = [ K(z^\tau , z^{\tau  '}) ] _{\tau , \tau ' \in [t] } \in \RR^{t \times t} , \qquad k_t (z) =  \bigl  [  K(z^1, z) , \ldots K(z^t , z)  \bigr ] ^\top \in \RR^{t}.
 \$
By definition, 
$Q $ in \eqref{eq:some_Q} is determined by $Q_0 \in \cH$ and a bonus term constructed using $\cD$. 
Thus, the function $Q_h^t$  constructed in Algorithm \ref{algo:kucb} belongs to $ \cQ_{\UCB} (h, R, B)$ when $\beta \leq B$ and $\| \alpha _h^t \|_{\cH } \leq R$. 
In the following, for any $\epsilon \in (0, 1)$, we let $\cC( \cQ_{\UCB} (h , R,B), \epsilon)$ be the minimal $\epsilon$-cover of $\cQ_{\UCB}(h, R, B)$ with respect to the $\ell_{\infty}$-norm on $\cZ$. 
That is, for any $
Q \in \cQ_{\UCB}(h, R, B)$, there exists $Q' \in \cC( \cQ_{\UCB} (h , R,B), \epsilon)$ satisfying $\| Q - Q ' \|_{\infty} \leq \epsilon$. Moreover, among all function classes that possess such a property, $\cC(\cQ_{\UCB}(h, R, B), \epsilon)$ has the smallest cardinality. Thus, by definition, $|\cC(\cQ_{\UCB}(h, R, B), \epsilon)| $ is the $\epsilon$-covering number of $\cQ_{\UCB}(h, R, B)$ with respect to  the $\ell_{\infty}$-norm on $\cZ$.

In addition, based on $\cQ_{\UCB} (h, R, B) $, we define the function class 
$\cV_{\UCB}(h, R, B) $ as 
\# \label{eq:func_class_V}
\cV_{\UCB} (h, R, B ) = \big\{ V \colon  V(\cdot ) =    \max_a Q(\cdot,a)  \text{~~for some~~} Q \in \cQ_{\UCB} (h, R, B) \bigr \}.  
\#
For any two value functions $V_1,V_2 \colon \cS \rightarrow \RR$, we denote their supremum norm distance as 
\#\label{eq:distance_V}
\textrm{dist}(V_1, V_2) = \sup_{x\in \cS} \bigl | V_1(x) - V_2(x) \bigr| .
\#
For any $\epsilon \in (0,1)$, we let $\cC(\cV_{\UCB}(h, R, B), \epsilon)$ denote the minimal $\epsilon$-cover of $\cV_{\UCB}(h, R, B)$ with respect to $\textrm{dist}(\cdot, \cdot)$ defined in \eqref{eq:distance_V}.   

The main result of this section is a set of upper bounds on 
the size of $\cC(\cV_{\UCB}(h, R, B), \epsilon)$ under the three eigenvalue decay conditions specified in Assumption \ref{assume:decay}.

\begin{lemma} \label{lemma:covering_number_V}
  Let Assumption \ref{assume:decay} hold and $\lambda $ be bounded in $[c_1, c_2]$, where both $c_1$ and $c_2$ are 
  absolute constants. 
  Then, for any $h \in [H]$, any $R , B > 0$, and any $\epsilon \in (0,1/e)$, 
   there exists a  positive constant  $C_N $   such that
  \#\label{eq:cover_lemma_final}
  & \log \bigl | \cC\bigl (\cV_{\UCB}(h, R, B), \epsilon \bigr ) \bigr |  \leq \log \bigl | \cC\bigl (\cQ_{\UCB}(h, R, B), \epsilon \bigr ) \bigr | \\
  & \qquad  \leq
  \begin{cases} 
  	C _N \cdot \gamma \cdot \bigl [  1+ \log (R/ \epsilon ) \bigr ] + C _N \cdot \gamma^2 \cdot \bigl [  1 + \log ( B / \epsilon)\big]   & \qquad \textrm{case (i)} , \\
    C_N  \cdot   \bigl [ 1+ \log (  R/\epsilon)    \bigr ]^{1+ 1/ \gamma} + C _N   \cdot \bigl [ 1+  \log ( B / \epsilon ) ] ^{1 + 2/ \gamma}& \qquad      \textrm{case (ii)} , 
    \\
    C _N \cdot ( R / \epsilon)^{ 2/ [ \gamma \cdot (1 - 2 \tau) - 1 ] } \cdot [ 1 + \log (R / \epsilon )] + C_N  \cdot ( B / \epsilon)^{ 4/ [ \gamma \cdot (1 - 2 \tau) - 1 ]     } \cdot [ 1 + \log ( B / \epsilon ) ]&  \qquad      \textrm{case (iii)} , \notag
   \end{cases} 
  \#
  where cases (i)--(iii) above correspond  to the three eigenvalue decay conditions specified in Assumption \ref{assume:decay}, respectively. 
Moreover, here $C_N$   in \eqref{eq:cover_lemma_final} is  independent of $T$, $H$, $R$, and $B$, and only depends on $C_{\psi}$, $C_1$, $C_2$, $c_1$, $c_2$,  $\gamma$, and $\tau$. 
\end{lemma}

\begin{proof}
For  any fixed subset $\cD = \{ z^\tau \}_{\tau \in [t]}$ of $\cZ$ with size $t \in [T]$, we define  
$\Phi_{\cD} \colon \cH \rightarrow\RR^{t}$ and $\Lambda_{\cD} \colon \cH \rightarrow \cH  $ respectively as 
\#
\Phi_{\cD} &=  \big [ \phi(z^1)^\top , \phi(z^2)^\top \ldots, \phi (z^{t}) ^\top   \bigr ] ^\top, \notag  \\
 \Lambda _{\cD} &= \sum_{\tau = 1}^{t}  \phi(z^\tau) \phi(z^\tau)^\top + \lambda \cdot I_{\cH} =  \lambda \cdot I_{\cH} +(\Phi _{\cD}  )^\top \Phi_{\cD}    , \label{eq:Lambda_D} 
\#
where $\phi \colon \cZ \rightarrow \cH$ is the feature mapping of $\cH$ and  $\cI_{\cH}$ is the identity mapping on $\cH$. 
Then, we can equivalently write $Q_1 \in \cQ_{\UCB} ( h, R, B ) $   as 
\#\label{eq:some_Q2}
Q_1 (z) = \phi(z)^\top \theta_1 + \beta \cdot \sqrt{   \phi(z)^\top \Lambda_{\cD_1}^{-1} \phi(z) } ,
\#
where $\theta_1  \in \cH$ has an RKHS norm bounded by $R$, $\beta_1 \in [0, B]$, and $\cD_1 $ is a finite subset of $\cZ$ with size  $t_1 \leq T$.  
Let $V_1 (\cdot ) = \max_{a \in \cA} Q_1(\cdot , a)$.
Combining \eqref{eq:func_class_V} and \eqref{eq:some_Q2}, we can write  $V_1 \in \cV_{\UCB} (h, R, B)$ as 
\#\label{eq:two_val_fun}
V_1 (\cdot) &= \min \Bigl \{ \max_{a} \bigl \{ \phi(\cdot, a)^\top \theta_1 + \beta_1 \cdot \sqrt{ \phi(\cdot, a) ^\top\Lambda _{\cD_1} ^{-1} \phi(\cdot, a)  }\bigr \}, H-h +1  \Bigr \}^{+}, 
\#
Similar to $V_1$ in \eqref{eq:two_val_fun}, consider any $V_2 \colon \cS \rightarrow \RR$ that can be written as 
\#\label{eq:val_fun2}
V_2 (\cdot ) =    \min \Bigl \{ \max_{a} \bigl \{ f_1 (\cdot, a) + \beta_2 \cdot f_2(\cdot , a )  \bigr \}, H-h +1  \Bigr \}^{+},
\#
  where $Q_2  = f_1 + \beta_2 \cdot f_2 $ for some  $f_1 , f_2 \colon \cZ \rightarrow \RR$ and $\beta_2 \in [0, B]$. 
Since both $\min\{\cdot, H- h +1 \}^+$ and $\max_a$ are contractive mappings, 
by \eqref{eq:two_val_fun}  and \eqref{eq:val_fun2}
we have 
\$
\mathrm{dist}(V_1, V_2) \leq \sup_{(x, a) \in \cZ} ~\Bigl | \Bigl[  \phi(x, a) ^\top \theta_1 +  \beta_1 \cdot \sqrt{\phi(x, a) ^\top \Lambda_{\cD_1} ^{-1}  \phi(x, a)  } \Bigr ]   -  \Bigl[  f(x,a)  +\beta_2  \cdot f_2 (x,a) \Bigr ] \Bigr | = \| Q_1 - Q_2 \|_{\infty}, 
\$
which implies that 
$$
\log \bigl | \cC\bigl (\cV_{\UCB}(h, R, B), \epsilon \bigr ) \bigr |  \leq \log \bigl | \cC\bigl (\cQ_{\UCB}(h, R, B), \epsilon \bigr ) \bigr |.
$$
Moreover, by the triangle inequality, we have 
\begin{align} \label{eq:cover_upper}
  \| Q_1 - Q_2 \|_{\infty}  &  \leq \sup_{(x,a) \in \cZ }   \bigl |  \phi (x,a) ^\top  \theta_1 - f_2(x,a)  \bigr | + | \beta_1 - \beta_2  | \cdot   \sup_{(x,a ) \in \cZ}   \bigl \| \phi(x,a)  \bigr \|_{\Lambda_{\cD_1} ^{-1} }  \notag \\
& \qquad \qquad +  B\cdot  \sup_{(x,a) \in \cZ } \big | \bigl \| \phi(x,a)  \bigr \|_{\Lambda_{\cD_1} ^{-1} }  -f_2   (x,a)      \bigr | ,
\end{align}
where we denote  $\| \phi(x,a)  \bigr \|_{\Lambda_{\cD_1} ^{-1} } ^2  = \phi(x, a) ^\top \Lambda_{\cD_1} ^{-1}  \phi(x, a) $. 
Notice that by the reproducing property we have 
$ \phi(x,a)^\top \theta  = \la \theta, \phi(x,a) \ra _{\cH} = \theta(x,a)  $ for all $\theta \in \cH$ and $(x,a )\in \cZ$. 
Also note that  
$$\| \phi(x,a)  \bigr \|_{\Lambda_{\cD_1} ^{-1} } ^2 \leq 1/ \lambda \cdot \| \phi (x,a) \|^2 \leq 1/ \lambda \cdot K(z, z) \leq 1/ \lambda. $$ 
Thus,  by \eqref{eq:cover_upper} we have 
\#\label{eq:cover_upper02}
\| Q_1 - Q_2 \|_{\infty} &    \leq \sup_{(x,a) \in \cZ }    \bigl |  \theta_1 (x,a) - f_1(x,a) \bigr |    +   | \beta_1 - \beta_2 |  / \lambda  \notag \\
& \qquad \qquad  + B\cdot  \sup_{(x,a) \in \cZ } \big | \bigl \| \phi(x,a)  \bigr \|_{\Lambda_{\cD_1} ^{-1} }  -f_2   (x,a)      \bigr | .
\#
Thus, by \eqref{eq:cover_upper02}, to get  the covering number of $\cQ_{\UCB} (h, R, B)$ with respect to $\textrm{dist}(\cdot, \cdot)$, it suffices to bound the  covering numbers of  the 
RKHS norm ball $\{ f \in \cH \colon \| f\|_{\cH} \leq R \}$, the interval $[0, B]$, and the set of functions that are of the form of  $ \| \phi(\cdot ) \|_{ \Lambda _{\cD}^{-1}}$, respectively.

Notice that, by the definition in \eqref{eq:Lambda_D}, $\Lambda_{\cD}  \colon \cH \rightarrow \cH$ is a self-adjoint operator on $\cH$ with eigenvalues bounded in $[0, 1/\lambda]$. 
To simplify the notation,    we define  the
 function class $\cF(\lambda)$ as 
\#\label{eq:define_Flambda}
\cF(\lambda) = \bigl \{ \| \phi(\cdot ) \|_{\Upsilon} = \bigl [\phi (\cdot) ^\top \Upsilon  \phi (\cdot)   \bigr ] ^{1/2 } \colon \| \Upsilon \|_{\oper} \leq 1 /\lambda    \bigr \},
\#
where $\Upsilon \colon \cH\rightarrow \cH$ in \eqref{eq:define_Flambda} is a self-adjoint operator on $\cH$ whose eigenvalues are all bounded by $1/\lambda$ in magnitude. Here,  the operator norm of $\Upsilon$  is defined as 
\$
\| \Upsilon \|_{\oper} = \sup \bigl \{ f^\top \Upsilon f \colon f\in \cH, \| f\|_{\cH } = 1  \big \} =   \sup  \big \{ \la f,  \Upsilon f \ra _{\cH } \colon f\in \cH,  \| f\|_{\cH } = 1 \big \} .
\$
Thus, by definition, for any finite subset $\cD$ of $\cZ$, $\| \phi(\cdot) \|_{\Lambda_{\cD}^{-1}} $ belongs to $\cF(\lambda)$, where $\Lambda_{\cD} $ is defined in  \eqref{eq:Lambda_D}. 
 For any $\epsilon \in (0,1)$, we let
 $N_{\infty} (\epsilon, \cF,  \lambda)$ denote the $\epsilon $-covering number of $\cF(\lambda)$   in \eqref{eq:define_Flambda}
 with respect to the $\ell_{\infty}$-norm.
  Moreover, let 
  $N_{\infty} (\epsilon, \cH, R)$ denote the $\epsilon$-covering number of $\{ f \in \cH \colon \| f\|_{\cH} \leq R\}$ with respect to the $\ell_{\infty}$-norm and let $N (\epsilon, B) $ denote the $\epsilon$-covering number of the interval $[ 0, B]$ with respect the Euclidean distance. Then, by \eqref{eq:cover_upper02}
we obtain that 
\#\label{eq:bound_covering_split}
\bigl | \cC\bigl (\cQ_{\UCB}(h, R, B), \epsilon \bigr ) \bigr | \leq N_{\infty} (\epsilon /3 , \cH, R) \cdot N  ( \epsilon \cdot \lambda /3 , B) \cdot N_{\infty} \bigl  ( \epsilon / (3B), \cF, \lambda \bigr ).
\# 
As shown in  
\cite[Corollary 4.2.13]{vershynin2018high}, it holds that 
\#\label{eq:covering_interval}
  N  ( \epsilon \cdot \lambda /3 , B)  \leq 1+  6B / (\epsilon \cdot \lambda) \leq 1 + 6B / \epsilon,
\#
where the last inequality follows from the fact that $\lambda \in [1,2]$. 
 
It remains to 
bound the first and the third terms on the right-hand side of \eqref{eq:bound_covering_split} separately. 
We 
  establish the $\ell_{\infty}$-covering of the  RKHS norm ball and $F(\lambda)$ in the following two lemmas, respectively.

\begin{lemma} [$\ell_\infty$-norm covering number of RKHS ball] \label{lemma:cover_rkhs}
For any $\epsilon \in (0,1)$, we let $N_{\infty} (\epsilon, \cH, R)$ denote the $\epsilon$-covering number of the RKHS norm ball $\{ f \in \cH \colon \| f\|_{\cH} \leq R \}$ with respect to the $\ell_{\infty}$-norm. 
Consider the three eigenvalue decay conditions given in Assumption \ref{assume:decay}. 
Then, under Assumption \ref{assume:decay}, 
there exist absolute constants $C_3$ and $C_4$ such that 
\$
\log N_{\infty} (\epsilon, \cH, R) \leq \begin{cases} 
	C_3 \cdot \gamma \cdot \bigl [ \log (R / \epsilon ) + C_4 \bigr ]  & \qquad \textrm{$\gamma$-finite spectrum}, \\
	 C_3 \cdot \bigl [ \log (R/\epsilon) + C_4\bigr ]^{1+ 1/ \gamma}  &\qquad   \textrm{$\gamma$-exponential decay} , \\
	 C_3 \cdot (R / \epsilon )^{2 / [ \gamma \cdot (1- 2\tau ) - 1 ]}
\cdot  \bigl [   \log (R/\epsilon) + C_4\bigr ]   & \qquad   \gamma\textrm{-polynomial decay} ,  
\end{cases}
\$
where $C_3$ and $C_4$ are independent of $T$, $H$,  $R$, and $\epsilon$, and only depend on absolute constants  $C_{\psi}$, $C_1$, $C_2$, $\gamma$, and $\tau$ specified in Assumption \ref{assume:decay}. 
\end{lemma}
\begin{proof}
  See \S\ref{proof:lemma_cover_rkhs} for a detailed proof. 
\end{proof}

\begin{lemma} 
  \label{lemma:cover_bonus}
  For any $\epsilon \in (0,1/ e )$, let $N_{\infty}(\epsilon, \cF,  \lambda )$ be the $\epsilon$-covering number of function class $\cF (\lambda)$ with respect to the $\ell_{\infty}$-norm, where $\cF(\lambda)$ is defined in \eqref{eq:define_Flambda}. Here we assume that $\lambda$ is bounded in $[c_1, c_2]$, where both $c_1$ and $c_2$ are absolute constants.
  Then, under Assumption \ref{assume:decay}, there exist absolute constants $C_5$ and $C_6$ such that 
  \$
  \log N_{\infty}(\epsilon, \cF,  \lambda ) \leq 
  \begin{cases}
  	C_5 \cdot \gamma^2 \cdot \bigl [ \log (1 /\epsilon ) + C_6\big ]   & \qquad \textrm{$\gamma$-finite spectrum},  \\
   C_5 \cdot \bigl [ \log (1 / \epsilon ) + C_6] ^{1 + 2/ \gamma}  & \qquad   \textrm{$\gamma$-exponential decay} , \\
   C_5 \cdot (1/ \epsilon )^{4/ [ \gamma \cdot (1 - 2\tau )  - 1 ]} \cdot \bigl [ \log (1 /\epsilon ) + C_6\big ] & \qquad      \textrm{$\gamma$-polynomial  decay} ,
   \end{cases} 
  \$
where $C_5$ and $C_6$ only depend on $C_{\psi}$, $C_1$, $C_2$, $\gamma$, $\tau$, $c_1$, and $c_2$, and do not rely on $T$, $H$, or $\epsilon$. 
\end{lemma}

\begin{proof}
  See \S\ref{proof:lemma_cover_bonus} for a detailed proof. 
\end{proof}

 Finally, we conclude the proof by  combining 
 Lemmas \ref{lemma:cover_rkhs} and \ref{lemma:cover_bonus}.
 Specifically, 
 by \eqref{eq:bound_covering_split} and \eqref{eq:covering_interval}, we have 
 \#\label{eq:cover_two_sets}
 \log \bigl | \cC\bigl (\cQ_{\UCB}(h, R, B), \epsilon \bigr ) \bigr | &  \leq \log  N_{\infty} (\epsilon /3 , \cH, R) + \log  N  ( \epsilon \cdot \lambda /3 , B) + \log  N_{\infty} \bigl  ( \epsilon / (3B), \cF, \lambda \bigr )  \\
 & \leq  \log  \big[  1+ 6B / (\epsilon \cdot \lambda) \bigr ]  +  \log  N_{\infty} (\epsilon /3 , \cH, R)  + \log  N_{\infty} \bigl  ( \epsilon / (3B), \cF, \lambda \bigr ). \notag 
 \#
We consider the three eigenvalue decay conditions separately. 
For the $\gamma$-finite spectrum case, 
by Lemmas \ref{lemma:cover_rkhs} and \ref{lemma:cover_bonus} and \eqref{eq:cover_two_sets} we have 
\$
& \log \bigl | \cC\bigl (\cQ_{\UCB}(h, R, B), \epsilon \bigr ) \bigr | \\
& \qquad \leq  \log \big[  1+ 6B / (\epsilon \cdot \lambda) \bigr ] + C_3 \cdot \gamma \cdot \bigl [ \log (3 R / \epsilon )  +C_4 \big ] + C_5\cdot \gamma^2 \cdot \bigl [  \log ( 3B / \epsilon ) + C_6  \bigr ] \\
& \qquad 
\leq C_N \cdot \gamma \cdot \big  [ 1 + \log (  R / \epsilon )    \bigr ] + C_N \cdot \gamma^2 \cdot \big [ 1 + \log (B / \epsilon )     \bigr ],
\$
 where    $C_N$  is an   absolute constant.  
 Similarly, for the case where the eigenvalues satisfy the $\gamma$-exponential decay condition,
 by Lemmas \ref{lemma:cover_rkhs} and \ref{lemma:cover_bonus} 
 we have 
\$
& \log \bigl | \cC\bigl (\cQ_{\UCB}(h, R, B), \epsilon \bigr ) \bigr | \\
& \qquad \leq  \log  \big[  1+ 6B / (\epsilon \cdot \lambda) \bigr ]  
+ C_3 \cdot   \bigl [ \log (3 R/\epsilon) + C_4\bigr ]^{1+ 1/ \gamma} + C_5 \cdot \bigl [ \log (3 B / \epsilon ) + C_6] ^{1 + 2/ \gamma} \notag \\
&\qquad  \leq C_N \cdot   \bigl [ 1 + \log (  R/\epsilon)  \bigr ]^{1+ 1/ \gamma} + C_N  \cdot \bigl [1+  \log (  B / \epsilon ) \big   ] ^{1 + 2/ \gamma}
\$
for some absolute constant  $C_N > 0$. 
Finally, for the case of  $\gamma$-polynomial eigenvalue  decay,
we  have 
\$
& \log \bigl | \cC\bigl (\cQ_{\UCB}(h, R, B), \epsilon \bigr ) \bigr | \notag \\
& \qquad \leq  \log  \big[  1+ 6B / (\epsilon \cdot \lambda) \bigr ]   + C_3 \cdot ( 3 R / \epsilon) ^{2 / [ \gamma \cdot (1 - 2\tau )-1 ]}
 \cdot \big [ \log ( 3R / \epsilon ) + C_4  \big ] \notag \\
 & \qquad \qquad + C_5 \cdot ( 3B / \epsilon ) ^{ 4 / [ \gamma \cdot (1 - 2\tau )-1 ] } \cdot \big [ \log ( 3 B / \epsilon ) + C_6  \big ] \notag \\
 & \qquad \leq  C_N \cdot  (   R / \epsilon) ^{2 / [ \gamma \cdot (1 - 2\tau )-1 ]}
 \cdot \big [ 1 + \log ( R / \epsilon ) \bigr ]  + C _N \cdot (  B / \epsilon ) ^{ 4  / [ \gamma \cdot (1 - 2\tau )-1 ] } \cdot \big [ 1 + \log (   B / \epsilon )  \big ],
 \$
where $C _N > 0$ is an absolute constant. 
Therefore, we conclude the proof.
\end{proof}




 \subsection{Effective Dimension of RKHS} \label{proof:lemma_effective_dim}

\begin{definition}[Maximal information gain] \label{define:max_infogain}
	For any fixed integer $T$ and any $\sigma > 0$, we define the  maximal information gain associated with the RKHS $\cH$ as 
	\#\label{eq:max_infogain}
	\Gamma_K (T, \sigma^2  ) = \sup_{\cD \subseteq  \cZ } \bigl \{ 1/2 \cdot   \logdet (  I + \sigma^{-2} \cdot K_{\cD}     )   \bigr \},
	\#
	where the supremum is taken over all discrete subsets of $\cZ$ with cardinality no more than $T$, and $K_{D} $ is the Gram matrix induced by $\cD \subseteq \cZ$, which is defined similarly as in \eqref{eq:define_gram}.  
	Here the subscript $K$ in  $\Gamma_K(T, \sigma^2)$ denotes the kernel function of $\cH$.
\end{definition} 

The maximal information gain naturally arises in  Gaussian process  regression. Specifically, let $f \sim \textrm{GP}(0, K)$ be   draw from the  Gaussian process with covariance  kernel $K$.
Let $\cD=\{z_1, \ldots , z_{|\cD|}  \}$ be a subset of $ \cZ$ with   $|\cD| \leq T$ elements. 
Suppose that we observe noisy observations of $f$ at points in $\cD$. That is, for any $z_i \in \cD$, we have $y_i = f(z_i) + \epsilon_i$, where $\epsilon_i  \sim N(0, \sigma^2) $ is a    
random Gaussian noise.  We let $y_{\cD}$ denote the vector whose entries are $y_i$. 
Then, the information gain of $y_{\cD}$ is defined as  the mutual information between $f$ and the observations $y_{\cD}$, denoted by $I(f, y_{\cD})$. 
By direct computation, we have 
\$
I(f, y_{\cD}) = 1/2 \cdot   \logdet (  I + \sigma^{-2} \cdot K_{\cD}     ). 
\$
The mutual information $I(f, y_{\cD})$ quantifies the reduction of the uncertainty about $f$ when we observe $y_{\cD}$. Thus, the maximal mutual information $\Gamma_K (T, \sigma^2)$ characterizes the maximal possible reduction of the uncertainty of $f$ when having no more than $T$ observations.

Moreover, we note that, when $\sigma^2$ is a constant,  $\Gamma_K (T, \sigma^2)$ depends on the eigenvalue decay of the RKHS and thus can be viewed as an effective dimension of the RKHS. Specifically, as shown in \cite{srinivas2009gaussian}, 
when the kernel is the $d$-dimensional linear kernel, $\Gamma_K(T, \sigma^2) = \cO( d \log T)$. 
Moreover, for the squared exponential kernel and the Mat\'ern kernel,
which satisfy the  exponential and polynomial eigenvalue decay conditions respectively, the maximal information gains are 
$$ 
\cO  \big ( ( \log T)^{d+1}  \big )   \qquad \text{and}\qquad  \cO\big ( T^{d(d+1) / ( 2\nu + d(d+1))} \cdot \log T \bigr ), $$
respectively, where $\nu$ is the parameter of the Mat\'ern kernel. 
In the following lemma, similar to Theorem 5 in  \cite{srinivas2009gaussian}, we establish    upper bounds on the maximal information gain of the RKHS under the three    eigenvalue decay conditions specified in Assumption \ref{assume:decay}. 
 
\begin{lemma} [Theorem 5 in  \cite{srinivas2009gaussian}] \label{lemma:effective_dim}
Let $\cZ$ be a compact subset of $\RR^d$ and $K \colon \cZ \times \cZ \rightarrow \RR$ be the RKHS kernel of $\cH$. 
We assume that $K$ is a bounded kernel in the sense that $\sup_{z\in \cZ} K(z, z) \leq 1$, and $K$ is continuously differentiable on $\cZ\times \cZ$. 
Moreover, let $T_K$ be the integral operator induced by $K$ and the Lebesgue measure on $\cZ$, whose definition is given in 
\eqref{eq:integral_oper}. Let $\{ \sigma_j\}_{j\geq 1}$ be the eigenvalues of $T_K$ in  the descending order. 
We  assume that $\{ \sigma_j\}_{j\geq 1}$ satisfy either one of the following  three eigenvalue decay conditions:
 \begin{itemize} 
   \item [(i)] $\gamma$-finite spectrum: We have $\sigma_j = 0$ for all $ j \geq  \gamma +1$, where $\gamma $ is a positive integer. 
	 \item [(ii)] $\gamma$-exponential eigenvalue decay: There exist  constants $C_1, C_2>0$ such that $\sigma_j \leq C_1 \exp(-C_2 \cdot j^{\gamma})$ for all $j \geq 1$, where $\gamma > 0$ is positive constant. 
	 \item [(iii)] $\gamma$-polynomial eigenvalue decay: There exists a constant $C_1$ such that $\sigma_j \geq C_1 \cdot j^{-\gamma}$ for all $j \geq 1$, where $\gamma \geq 2 + 1 /d$ is a constant.
 \end{itemize} 
 Let $\sigma $ be bounded in interval $
 [c_1, c_2] $ with $c_1$ and $c_2$ being absolute constants. 
 Then,  
for  conditions (i)--(iii) respectively,
 we have  
\$
 \Gamma _K(T, \sigma^2 )  \leq 
 \begin{cases} 
 C_K \cdot \gamma \cdot \log T   & \qquad \textrm{$\gamma$-finite spectrum,} \\
 C_K \cdot (\log T )^{1 + 1/\gamma}  & \qquad \textrm{$\gamma$-exponential decay},\\
   C_K \cdot T^{(d+1) / (\gamma + d)} \cdot \log T  & \qquad \textrm{$\gamma$-polynomial decay},
   \end{cases}
 \$ 
 where $C_K$ is an absolute constant that depends on $d$, $\gamma$, $C_1$, $C_2$, $C$, $c_1$, and $c_2$. 
\end{lemma}
 
	We note that Lemma \ref{lemma:effective_dim} is a generalization of Theorem 5 in \cite{srinivas2009gaussian},
	which establishes the maximal information gain for the  linear, squared exponential,  and Mat\'ern kernels, respectively. 
	Specifically, the squared exponential kernel satisfies the $\gamma$-exponential eigenvalue decay condition with $\gamma = 1/d$. 
	Lemma \ref{lemma:effective_dim} implies that the $\Gamma_K(T, \sigma^2) = \cO( (\log T)^{d+1} ) $, which matches Theorem 5 in \cite{srinivas2009gaussian}. 
	Furthermore, the Mat\'ern kernel with parameter $\nu$ satisfies the $\gamma$-polynomial eigenvalue decay condition with  $\gamma = (2\nu +d )/ d$. 
	Then, Lemma \ref{lemma:effective_dim} asserts that $\Gamma_K(T, \sigma^2) = \cO( T^{d(d+1)/ (2\nu + d(d+1))}\cdot \log T) $, which also matches  Theorem 5 in \cite{srinivas2009gaussian}.

\begin{proof} 
The proof of this lemma is based on a modification of that of Theorem 5 in \cite{srinivas2009gaussian}. 
To begin with, 
for any $j \in \NN$, we define $B_{K} (j) = \sum _{s > j} \sigma_{s}$, i.e., the sum of eigenvalues with indices larger than $j$. 
Then, we use the following lemma obtained from \cite{srinivas2009gaussian} to bound $\Gamma_K(T, \sigma^2)$ using function $B_K$. 

\begin{lemma} [Theorem 8 in \cite{srinivas2009gaussian}] \label{lemma:apply_theorem_eigen} 
Under the same condition as in Lemma \ref{lemma:effective_dim}, for any fixed $\tau > 0$, we denote 
$C_{\tau} = 2 \mu(\cZ) \cdot (2\tau+1)$ where $\mu(\cZ)$ is the Lebesgue measure of $\cZ$. 
Let $n_T$ denote   $C_{\tau} \cdot T^\tau \cdot \log T$.  Then, 
for  any $T_{\star } \in \{ 1, \ldots, n_T \}$, we have 
\$
\Gamma_K (T, \sigma^2 ) \leq T_{\star } \cdot \log (T \cdot n_T / \sigma^2 ) + C_{\tau} \cdot \sigma^{-2} \cdot \log T \cdot \bigl [  T^{\tau+1} \cdot B_K(T_{\star}) + 1\bigr ] + \cO(T^{1 -\tau /d}).
\$	
\end{lemma}
\begin{proof}
	See \cite{srinivas2009gaussian} for a detailed proof.
\end{proof}
In the following, we choose proper $\tau$ and $T_{\star}$ in Lemma \ref{lemma:apply_theorem_eigen}   for the three eigenvalue decay conditions separately.

\vspace{5pt}
{\noindent \textbf{Case (i): $\gamma$-Finite Spectrum.}} When $\sigma_j = 0$ for all $j \geq \gamma + 1$, 
we set $\tau = d$ and $T_{\star } = \gamma$ in Lemma \ref{lemma:apply_theorem_eigen}. 
Then we have $B_{K} (T_{\star } ) =0$ and $n_T = C_{d }\cdot T^d \cdot \log T$. 
When $T$ is sufficiently large, it holds that $T_{\star } < n_T$. 
Then 
Lemma \ref{lemma:apply_theorem_eigen}
implies that 
\$
\Gamma_{K}   ( T, \sigma^2 ) \leq \gamma \cdot \log \big ( C_{d} \cdot T^{d+1} \cdot  \log T / \sigma^2   \bigr) 
+ C_{d } \cdot \sigma^{-2}  \cdot \log T  + \cO( 1)   \leq C_K \cdot \gamma \cdot \log T,
\$
for some absolute constant $C_K > 0$. Thus, we conclude the proof for the first case. 

\vspace{5pt}
{\noindent \textbf{Case (ii): $\gamma$-Exponential Decay.}} When  $\{\sigma_j \}_{j\geq 1}$ satisfies the $\gamma$-exponential eigenvalue decay condition, for any $T_{\star} \in \NN$, we have 
\#\label{eq:eigen_decay1}
B_{K}(T_{\star}) = \sum_{j > T_{\star}} \sigma_j \leq C_1 \cdot \sum_{j > T_{\star}} \exp(- C_2 \cdot j^{\gamma}) \leq C_1 \cdot \int_{T_{\star }}^{\infty} \exp( -C_2 \cdot u^{\gamma}) ~\ud u.
\#
In a manner similar to the derivation of \eqref{eq:bound_eigen_final}, 
by direct computation we have 
\#\label{eq:eigen_decay2}
\int_{T_{\star}} ^\infty \exp( - C_{2} \cdot u ^\gamma ) ~\ud u  \leq 
\begin{cases}
 C_{2}^{-1}  \cdot \exp( - C_{2} \cdot T_{\star} ^{\gamma}), & \qquad  \text{if~} \gamma \geq  1,\\
 2 \cdot (\gamma \cdot C_{2})^{-1} \cdot  \exp( - C_{2}\cdot T_{\star } ^{\gamma}) \cdot T_{\star } ^{ 1 -  \gamma } , & \qquad \text{if~} \gamma \in (0,1).
\end{cases}
\#
In the following, we set $\tau = d$. Then we have $n_T = C_d \cdot T^d \cdot \log T$ where $C_d = 2 \mu(\cZ) \cdot (2d+1)$. 
Then we have 
\#\label{eq:eigen_decay3}
\log (T \cdot n_T) = \log (C_d) + \log \cdot (T^{d+1} \cdot \log T \bigr ) \leq \log (C_d) + 2(d+1) \cdot \log T,
\#
when $T$ is sufficiently large. 
Moreover, combining  Lemma \ref{lemma:apply_theorem_eigen}  and \eqref{eq:eigen_decay3}, when $\sigma$ is sandwiched by absolute constants $c_1$ and $c_2$,    
we have 
\#\label{eq:eigen_decay4}
\Gamma_K (T, \sigma^2 ) \leq \tilde C_1\cdot T_{\star } \cdot \log T  + \tilde C_2   \cdot \log T \cdot \bigl [  T^{d+1} \cdot B_K(T_{\star}) + 1\bigr ] + \tilde C_3 ,
\#
where $\tilde C_1,$ $\tilde C_2$, and $\tilde C_3$ are absolute constants that depend  on $d$, $\gamma$, $c_1$,  $c_2$, $C_1$, and $C_2$. 
Now we choose $T_{\star}$ 
such that 
\#\label{eq:eigen_decay5}
\exp(C_2 \cdot T_{\star}^\gamma ) \asymp  T\cdot n_T = C_d \cdot T^{d+1} \cdot \log T,
\#
that is, $T_{\star} = \tilde C_4\cdot (\log T)^{1/\gamma} $ 
where $\tilde C_4$ is an absolute constant.  Notice that $T_{\star} < n_T$ when $T $ is sufficiently large.

Thus, combining   \eqref{eq:eigen_decay1},  \eqref{eq:eigen_decay2}, and \eqref{eq:eigen_decay5}, 
for $\gamma \geq 1$, we have 
\# \label{eq:eigen_decay6}
& \log T \cdot  \bigl [  T^{d+1} \cdot B_K(T_{\star}) + 1\bigr ]  \notag \\
& \qquad  \leq C_1 \cdot C_{2}^{-1} \log T \cdot T^{d+1} \cdot  \exp( - C_{2} \cdot T_{\star} ^{\gamma})  + \log T \leq 2 \log T , 
\#
where the last inequality follows from \eqref{eq:eigen_decay5}. 
Similarly, for $\gamma\in (0, 1)$, by  \eqref{eq:eigen_decay1},  \eqref{eq:eigen_decay2}, and \eqref{eq:eigen_decay5},  we have 
\#  \label{eq:eigen_decay7}
& \log T \cdot  \bigl [  T^{d+1} \cdot B_K(T_{\star}) + 1\bigr ] \notag \\
& \qquad  \leq  
	2C_1  \cdot (\gamma \cdot C_{2})^{-1} \cdot  \exp( - C_{2}\cdot T_{\star } ^{\gamma}) \cdot \log T \cdot T^{d+1} \cdot T_{\star } ^{ 1 -  \gamma } + \log T  \asymp (\log T )^{1 / \gamma - 1} + \log T.
\#
Thus, combining \eqref{eq:eigen_decay4}, \eqref{eq:eigen_decay6}, \eqref{eq:eigen_decay7}, we conclude that 
\$
\Gamma_{K} (T, \sigma^2 ) \leq C_{K} \cdot \log (T)^{1 + 1/\gamma} 
\$
for any  $\gamma \geq 0$, where $C_K$ is an absolute constant that depends on $d$, $\gamma$,  $c_1$, $c_2$, $C_1$, and $C_2$.  
Thus, we conclude the proof for the second case.

\vspace{5pt}
{\noindent \textbf{Case (iii): $\gamma$-Polynomial Decay.}} 
Similar to the former case, when $\{ \sigma_j \}_{j\geq 1}$ satisfies the $\gamma$-polynomial eigenvalue decay condition, for any $T_{\star} \in \NN$, we have 
\#\label{eq:eigen_decay8}
B_{K}(T_{\star}) = \sum_{j > T_{\star}} \sigma_j \leq C \cdot \sum_{j > T_{\star}} j^{ - \gamma}  \leq  C \cdot   \int_{T_{\star }}^{\infty}   u^{ - \gamma}  ~\ud u  = \frac{C} {\gamma - 1} \cdot T_{\star} ^{- (\gamma -1)}.
\#
For a fixed $\tau \in (0, \gamma + d)$ to be determined  later, recall that we denote $n_T = T^{\tau } \cdot \log T$. 
Hence, we have $$\log (T\cdot n_T) = (1+ \tau ) \cdot \log T + \log\log T \leq (\gamma+d+2) \cdot \log T.$$
Combining Lemma \ref{lemma:apply_theorem_eigen} and \eqref{eq:eigen_decay8}, we have 
\#\label{eq:eigen_decay9}
\Gamma_K (T, \sigma^2 ) \leq \tilde C_5 \cdot T_{\star } \cdot \log T  + \tilde C_6 \cdot \log T \cdot  \bigl ( T^{\tau +1} \cdot T_{\star}^{-(\gamma -1)} + 1 \bigr )  + \tilde C_7 \cdot T^{1- \tau /d}, 
\# 
where $T_{\star} \in \{ 1, \ldots, n_T \}$, and $\tilde C_5$, $\tilde C_6$, $\tilde C_7$ are absolute constants that depend on $d$, $\gamma$, $c_1$, $c_2$, and $C$. 

Now we choose a $T_{\star} $  that balances the first two terms on the right-hand side of \eqref{eq:eigen_decay9}. 
Specifically, we let $T_{\star } = T^{ (\tau+1) / \gamma}$. Then we have 
$T_{\star } = T^{\tau +1} \cdot T_{\star} ^{- (\gamma - 1)}$. 
Finally, solving the equation  $T_{\star } = T^{1- \tau / d}$, we have 
$
\tau = (\gamma -1 )\cdot d / ( d+ \gamma). 
$
Hence, by \eqref{eq:eigen_decay9} we conclude that 
\$
\Gamma_K(T, \sigma^2 ) \leq C_K \cdot T_{\star } \cdot \log T = C_K \cdot T ^{(d+1) / (\gamma+d)} \cdot \log T,
\$
where $C_K$ is an absolute constant. It remains to verify that $T_{\star } \leq n_T = T^{\tau }\cdot \log T$. 
This is true as long as $T_{\star} \leq T^{\tau}$, i.e., $ \gamma \geq 2+ 1/d $.  
Therefore, we conclude the proof of Lemma \ref{lemma:effective_dim}. 
\end{proof}

\section{Proofs of Auxiliary Results} \label{sec:proof_of_lemmas}

In this section, we provide the proofs of the auxiliary results.

\subsection{Proof of Lemma \ref{lemma:thetahat_estimate}}\label{proof:thetahat_estimate}
\begin{proof} For any function $f \in \cH$, using the feature representation induced by the kernel $K$, we have 
\#\label{eq:norm_bound1}
 \big | \la f, \hat \theta_h^t  \ra _{\cH}  \big |  & =  \big | f^\top \hat \theta_h^t \big| \leq \big |   f^\top (\Lambda_h^t)^{-1} \Phi^\top y _h^t  \big |  \notag  \\
 & =  \bigg |f ^\top (\Lambda^t_h)^{-1} \sum_{\tau=1}^{t-1}  \phi(x_h^\tau, a_h^\tau) \cdot  [r_h (x^{\tau}_h, a^{\tau}_h) + V_{h+1}^t(x^{\tau}_{h+1})] \biggr |, 
 \#
 where we let $\Phi$ denote $\Phi_h^t$ defined in \eqref{eq:define_feature_mat} for simplicity. 
Since $ | r_h (x^{\tau}_h, a^{\tau}_h)   | \leq 1$ 
and $ | V_{h+1}^t (x^{\tau}_{h+1}) | \leq H - h $, we have 
$|[r_h (x^{\tau}_h, a^{\tau}_h) + V_{h+1}^t(x^{\tau}_{h+1})]| \leq H$ for all $h\in [H]$ and $\tau \in [t-1]$. 
Then, 
by \eqref{eq:norm_bound1} and the Cauchy-Schwarz inequality, we have 
\#\label{eq:norm_bound2} 
 \big | \la f, \hat \theta_h^t  \ra _{\cH}  \big | & \leq   H \cdot \sum_{\tau = 1}^{t-1} \bigl | f ^\top  (\Lambda^t_h)^{-1}  \phi(x_h^\tau, a_h^\tau)  \bigr |  
 \notag \\
 & \leq   H  \cdot \biggl [  \sum _{\tau =1}^{t-1} f^\top (\Lambda^t_h)^{-1}  f \biggr ] ^{1/2} \cdot \bigg[  \sum _{\tau =1}^{t-1}    \phi(x_h^\tau, a_h^\tau) ^\top (\Lambda^t_h)^{-1}     \phi(x_h^\tau, a_h^\tau) \biggr ] ^{1/2}  \notag \\
 & \leq   H /  \sqrt{ \lambda }  \cdot \| f \|_{\cH}  \cdot \bigg[  \sum _{\tau =1}^{t-1}    \phi(x_h^\tau, a_h^\tau) ^\top (\Lambda^t_h)^{-1}     \phi(x_h^\tau, a_h^\tau) \biggr ] ^{1/2},
\#
where the last inequality follows from the fact that $(\Lambda_h^t )^{-1} \colon \cH \rightarrow \cH $ is a self-adjoint and positive-definite operator whose eigenvalues are bounded by $1 / \lambda$. 
Furthermore, by Lemma \ref{lemma:telescope}, we have 
\#\label{eq:norm_bound3}  
 \bigg[  \sum _{\tau =1}^{t-1}    \phi(x_h^\tau, a_h^\tau) ^\top (\Lambda^t_h)^{-1}     \phi(x_h^\tau, a_h^\tau) \biggr ]  \leq  2 \logdet ( I+     K_h^t  / \lambda ). 
 \#
 Thus, combining \eqref{eq:norm_bound2},  \eqref{eq:norm_bound3}, and the fact that $\lambda \geq 1$,  we obtain that 
 \$
 \big | \la f, \hat \theta_h^t  \ra _{\cH}  \big | \leq      H \cdot      \| f \|_{\cH} \cdot \sqrt{ 2 /\lambda \cdot \logdet ( I+     K_h^t  / \lambda  ) } \leq H \cdot \| f \|_{\cH} \cdot  \sqrt{ 2 \cdot \logdet  ( I+     K_h^t  / \lambda  ) } .
 \$
 Finally, utilizing the definition of 
 $\Gamma_K(T, \lambda)$ in \eqref{eq:max_infogain}, we 
conclude  the proof of this lemma. 
\end{proof}

\subsection{Proof of Lemma \ref{lemma:cover_rkhs}} \label{proof:lemma_cover_rkhs}

 \begin{proof} 
 	Recall that we have defined the integral operator 
 	 $T_K \colon \cL^2   (\cZ) \rightarrow \cL^2 (\cZ)$ defined  in  
 \eqref{eq:integral_oper}, which 
  has eigenvalues $\{ \sigma_j \}_{j\geq 0}$ and eigenvectors $\{ \psi_j \}_{j\geq 0}$.
 Moreover, 
   $\{ \psi_j \} $ and $\{ \sqrt{\sigma_j} \cdot \psi_j\}_{j\geq 0}$ are  orthonormal bases of $\cL_2 (\cZ)$ and $\cH$, respectively. 
   Then, any $ \in \cH$ with $\| f \|_{\cH} \leq R$ can be written as 
    \# \label{eq:exxpand_f}
   f = \sum_{j= 1}^\infty w_j \cdot \sqrt{\sigma_j }\cdot \psi_j  ,
   \#
   where $\{ w_j \}_{j \geq 0}$ satisfy $\sum _{j =1}^{\infty} w_j^2 = \| f\|_{\cH}^2 \leq R^2$. 
     Let $m$ be any  positive integer and 
     let $\Pi_m \colon \cH \rightarrow \cH $ denote  the projection onto the subspace spanned by $\{ \psi_j \}_{j \in [m]}$, i.e., 
     $
     \Pi_m (f) =  \sum_{j=1}^m w_j \cdot \sqrt{\sigma_j}\cdot \psi_j 
     $
     for any $f \in \cH$ written as in \eqref{eq:exxpand_f}. 
     Then we have 
     \#\label{eq:proj_rkhs}
     \| f - \Pi_m (f) \|_{\infty } =  \sum _{j =m+1}^\infty  |w_j | \cdot \sqrt{ \sigma_j } \cdot  \sup_{z \in \cZ } | \psi_j (z) | .
     \#
In the following, we consider the three eigenvalue decay conditions specified in  Assumption \ref{assume:decay} separately. 

\vspace{5pt}

{\noindent \bf Case (i): $\gamma$-Finite Spectrum.}
Consider the case where $\sigma_{j} = 0$ for all $ j > \gamma$. 
Then, by the definition of $\Pi_{m}$, we have $f = \Pi_{\gamma} (f) $ for all $f \in \cH$. 
That is, \eqref{eq:exxpand_f} is reduced to 
\$
f = \sum_{j=1}^\gamma w_j \cdot \sqrt {\sigma_j} \cdot \psi_j ,
\$
where $\{ w_j\}_{j \in [\gamma]}$ satisfies $\sum_{j=1}^\gamma w_j^2 \leq R^2$. 
Let $\cC_{
  \gamma} (\epsilon, R)$ be the minimal $\epsilon$-cover 
of the $\gamma$-dimensional Euclidean ball $\{w \in \RR^{\gamma} \colon \| w \| _2 \leq R \}$ with respect to the Euclidean norm.  
Then, by construction, there exists $\tilde w \in \RR^{\gamma}$ 
such that $\sum_{j =1}^\gamma ( w_j - \tilde w_j)^2 \leq \epsilon^2$. 
Then, by the Cauchy-Schwarz inequality, we have 
\#\label{eq:linear_rkhs_approx1}
\biggl \| f - \sum_{j=1}^\gamma \tilde w_j \cdot \sqrt{\sigma_j} \cdot \psi_j \biggr \|_{\infty} & = \sup_{z \in \cZ } \bigg | \sum_{j=1}^\gamma (w_j - \tilde w_j)  \cdot \sqrt{\sigma_j}  \cdot \psi_j (z)\bigg |  \\
& = \biggl [ \sum_{j=1}^\gamma (w_j - \tilde w_j )^2 \bigg ]^{1/2} \cdot \sup_{z \in \cZ } \bigg \{  \biggl [ \sum_{j=1 }^\gamma \sigma_j \cdot | \psi_j (z) |^2  \biggr ]^{1/2} \bigg \} \leq \epsilon \cdot \sup_{z} \sqrt{ K(z,z)} \leq \epsilon, \notag
\#
where the last equality follows from the fact that $K(z,z) = \sum_{j=1}^\gamma \sigma_j \cdot | \psi_j(z)|^2$. 
Thus, the $\epsilon$-covering of $\{f \in \cH \colon \| f\|_{\cH} \leq R \}$ is bounded by the cardinality of $\cC_{\gamma }(\epsilon ,R)$. 
As shown in  
\cite[Corollary 4.2.13]{vershynin2018high}, we have 
\#\label{eq:bound_card0}
\bigl| \cC_{\gamma}(\epsilon, R)\bigr | \leq ( 1+ 2R/ \epsilon) ^{\gamma }.
\#
Thus, combining \eqref{eq:linear_rkhs_approx1} and \eqref{eq:bound_card0}, we have 
\$
\log N_{\infty} (\epsilon, \cH, R) \leq \gamma \cdot \log ( 1 + 2 R / \epsilon) \leq C_3 \cdot \gamma \cdot \bigl [ \log (R /\epsilon ) +C_4 \bigr ], 
\$
where both $C_3$ and $C_4$ are absolute constants. 
Thus, we  conclude the proof for the  first case.

\vspace{5pt}

{\noindent \bf Case (ii): $\gamma$-Exponential Decay.}
 In the following, we assume the eigenvalues $\{ \sigma_j \}_{j \geq 1}$ satisfy the $\gamma$-exponential decay condition and 
     $\| \psi_j \|_{\infty} \leq C_{\psi} \cdot \sigma_j ^{-\tau}$ for all $j \geq 1$. Thus, by \eqref{eq:proj_rkhs} we have 
     \#\label{eq:proj_rkhs2}
     \| f - \Pi_m (f) \|_{\infty }&  \leq   \sum _{j =m+1}^\infty  C_{\psi} \cdot  |w_j | \cdot \sigma_j ^{1/2 - \tau } \notag \\
     &  \leq \sum _{j =m+1}^\infty  C_{\psi} \cdot C_1 ^{1/2 - \tau }  \cdot | w_j | \cdot \exp \bigl [  - C_2 \cdot ( 1/2 - \tau) \cdot  j^\gamma \bigr ] .
     \#
     To simplify the notation, we define $C_{1, \tau} = C_{\psi} \cdot C_1^{1/2 - \tau} $ and $C_{2, \tau} = C_2 \cdot (1- 2 \tau )$.
     Then, applying the Cauchy-Schwarz inequality to \eqref{eq:proj_rkhs2}, we have 
     \#\label{eq:proj_rkhs3}
     \| f - \Pi_m (f) \|_{\infty } & \leq C_{1, \tau} \cdot \biggl (  \sum_{j=m+1} ^\infty |w_j |^2 \bigg) ^{1/2} \cdot \bigg [  \sum_{ j=m+1} ^\infty  \exp( - C_{2, \tau} \cdot j^{\gamma} )  \bigg]^{1/2}  \notag \\
     & \leq C_{1, \tau} \cdot  R \cdot \bigg [  \sum_{ j=m+1} ^\infty  \exp( - C_{2, \tau} \cdot j^{\gamma} )  \bigg]^{1/2}  ,
     \#
     where the second inequality follows from the fact that $ \sum_{j \geq 1} w_j^2 \leq R^2$. 
     Since $\gamma > 0$,   $\exp( - u^\gamma )$ is monotonically decreasing in $u$. Thus, we have 
     \#\label{eq:bound_eigens1}
     \sum_{ j=m+1} ^\infty  \exp( - C_{2, \tau} \cdot j^{\gamma} )   \leq \int_{m} ^\infty \exp( - C_{2, \tau} \cdot u ^\gamma ) ~\ud u.
     \#

     In the following, we bound the integral in \eqref{eq:bound_eigens1} by considering the cases where $\gamma \geq 1$ and $\gamma \in (0,1)$ separately. 
     First, when $\gamma \geq 1$, since $d \geq 1$, we have $u^{\gamma -1} \geq 1$ for all $u\geq d$. Hence, we have 
     \#\label{eq:bound_eigens2}
      &\int_{m} ^\infty \exp( - C_{2, \tau} \cdot u ^\gamma ) ~\ud u   \leq \int_{m} ^\infty u^{\gamma - 1} \cdot \exp( - C_{2, \tau} \cdot u ^\gamma )  ~\ud u \notag \\
      & \qquad  \leq    \int_{m^{\gamma} } ^\infty  \exp( - C_{2, \tau} \cdot  v  )  ~\ud v  =     C_{2, \tau}^{-1}  \cdot \exp( - C_{2, \tau} \cdot m^{\gamma}),
     \#
     where the second inequality follows from  the change of variable $v = u^{\gamma}$ and the fact that $\gamma \geq 1$. 
   Second,  when $\gamma < 1$, by letting $v = u^{\gamma}$, we have 
     \#\label{eq:bound_eigens3}
     & \int_{m} ^\infty \exp( - C_{2, \tau} \cdot u ^\gamma ) ~\ud u   = \frac{1}{\gamma } \cdot \int_{m^{\gamma}}^\infty \exp( - C_{2, \tau} \cdot v) \cdot v^{1/\gamma -1} ~\ud v 
      = \frac{1}{\gamma \cdot C_{2, \tau}} \int_{m^{\gamma}} ^\infty v^{1/ \gamma - 1} ~\ud [ - \exp( - C_{2, \tau}\cdot v)]  \notag \\
      & \qquad = \frac{1} {\gamma \cdot C_{2, \tau}} \cdot \exp ( - C_{2, \tau} \cdot m^{\gamma }) \cdot m ^{    1 - \gamma } + \frac{(1-\gamma)}{\gamma^2 \cdot C_{2, \tau} }\int_{m^{\gamma} }^\infty \exp( - C_{2, \tau} \cdot v ) \cdot v^{1/\gamma -2} ~\ud v ,
      \#
      where the last equality follows from integration by parts. Moreover, by direct calculation, we have 
      \#\label{eq:bound_eigens4}
      \frac{1}{\gamma} \int_{m^{\gamma} }^\infty \exp( - C_{2, \tau} \cdot v ) \cdot v^{1/\gamma -2} ~\ud v  & \leq \frac{1}{m^{\gamma}} \cdot \frac{1}{\gamma} \int_{m^{\gamma} }^\infty \exp( - C_{2, \tau} \cdot v ) \cdot v^{1/\gamma -1}  ~\ud v \notag \\
      & = \frac{1} {m^{\gamma} } \int_{m} ^\infty \exp( - C_{2, \tau} \cdot u ^\gamma ) ~\ud u ,
      \#
      where the first inequality follows from the fact that $v \geq m^{\gamma}$ in the integral and the second equality follows from letting $u = v^{1/\gamma}$. 
    Then, combining \eqref{eq:bound_eigens3} and \eqref{eq:bound_eigens4}, we have 
     \#\label{eq:bound_eigens5}
     &  \int_{m} ^\infty \exp( - C_{2, \tau} \cdot u ^\gamma ) ~\ud u \notag \\
     & \qquad \leq \frac{1} {\gamma \cdot C_{2, \tau}} \cdot \exp ( - C_{2, \tau} \cdot m^{\gamma }) \cdot m ^{  1 - \gamma } + \frac{1/\gamma - 1}{  C_{2, \tau} \cdot m^{\gamma}} \cdot   \int_{m} ^\infty \exp( - C_{2, \tau} \cdot u ^\gamma ) ~\ud u.
     \#
     Thus, when $m$ is sufficiently large such that $m^{\gamma} \cdot C_{2, \tau} > 2/\gamma -2$, 
     by \eqref{eq:bound_eigens5} we have 
     \#\label{eq:bound_eigens6}
      &\int_{m} ^\infty \exp( - C_{2, \tau} \cdot u ^\gamma ) ~\ud u  \leq \biggl ( 1 -  \frac{1/\gamma - 1}{  C_{2, \tau} m^{\gamma}}  \biggr)^{-1} \cdot  \frac{1   }{\gamma \cdot C_{2, \tau}} \exp( - C_{2, \tau}\cdot m^{\gamma}) \cdot m^{ 1 - \gamma   } \notag \\
      & \qquad \leq \frac{2   }{\gamma \cdot C_{2, \tau}} \exp( - C_{2, \tau}\cdot m^{\gamma}) \cdot m^{ 1 - \gamma   }. 
     \#
     Therefore, combining \eqref{eq:bound_eigens1}, \eqref{eq:bound_eigens2}, and \eqref{eq:bound_eigens6}, we obtain that 
     \#\label{eq:bound_eigen_final}
     \int_{m} ^\infty \exp( - C_{2, \tau} \cdot u ^\gamma ) ~\ud u  \leq 
     \begin{cases}
      C_{2, \tau}^{-1}  \cdot \exp( - C_{2, \tau} \cdot m^{\gamma}), & \qquad  \text{if~} \gamma \geq  1,\\
      2 \cdot (\gamma \cdot C_{2, \tau})^{-1} \cdot  \exp( - C_{2, \tau}\cdot m^{\gamma}) \cdot m^{ 1 - \gamma   }, & \qquad \text{if~} \gamma \in (0,1).
     \end{cases}
     \#

In the sequel, we let $m^*$   be the smallest integer such that 
\#\label{eq:define_dstar}
\int_{m} ^\infty \exp( - C_{2, \tau} \cdot u ^\gamma ) ~\ud u  \leq \biggl ( \frac{\epsilon}{ 2 C_{1, \tau} \cdot R}\bigg) ^{2}, \qquad \forall m \geq m^*. 
\#
Hence, combining \eqref{eq:proj_rkhs3}, \eqref{eq:bound_eigens1}, and \eqref{eq:define_dstar}, we have 
  $ \| f - \Pi_{m^*} (f) \|_{\infty} \leq \epsilon/2$ 
     for any $f \in \cH$ with $\| f\|_{\cH} \leq R$.  
 Note, moreover, that $C_{1, \tau}$, $C_{2,\tau}$,  and $\gamma$ are all absolute constants.
By \eqref{eq:bound_eigen_final} and \eqref{eq:define_dstar}, 
there exist absolute constants $C_{1, m}$ and $C_{2,m}$ such that  
\#\label{eq:bound_dstar}
m^* \leq C_{1,m} \cdot \bigl[\log (R / \epsilon ) + C_{2,m} \bigr ] ^{1/ \gamma}. 
\#

Finally, it remains to approximate $\Pi_{m^*} (f)$ up to error $\epsilon /2$ for $m^*$ specified in \eqref{eq:define_dstar}. By the expansion of $f$ in \eqref{eq:exxpand_f}, we have $\Pi_{m^*} (f) = \sum_{j=1}^{m^*} w_j \cdot \sqrt{\sigma_j} \cdot \psi_j $. For any $m^*$ real numbers 
$\{ \tilde w_j\}_{j\in [m^*]}$, by the Cauchy-Schwarz inequality, we have 
 \# \label{eq:approx_pi_f}
&  \biggl | \bigl [  \Pi_{m^*}( f) \bigr ]  (z)  - \sum_{j=1}^{m^*} \tilde w_j \cdot \sqrt{ \sigma_j} \cdot \psi_j (z) \biggr  |  =  \biggl |    \sum_{j=1}^{m^*} ( w _j - \tilde w_j ) \cdot  \sqrt{ \sigma_j} \cdot \psi_j (z) \biggr  |   \notag \\
 & \qquad \leq \biggl [\sum _{j=1}^{m^*} ( w_j - \tilde w_j)^2   \bigg ]  ^{1/2} \cdot   \biggl \{ \sum _{j=1}^{m^*}  \sigma_j \cdot [\psi_j (z) ]^2  \bigg \} ^{1/2}  \leq  \sqrt{  K(z,z)}  \cdot  \biggl [\sum _{j=1}^{m^*} ( w_j - \tilde w_j)^2   \bigg ]  ^{1/2},
 \# 
where the last inequality follows from the fact that $K(z,z) = \sum_{j=1}^\infty \sigma_j \cdot [ \psi_j(z)]^2$. 
Under Assumption~\ref{assume:decay}, we have $\sup_{z\in \cZ} K(z, z) \leq 1$. 
Notice that $\sum_{j=1}^{m^*} \omega_j^2 \leq \| f \|_{\cH} ^2 \leq R^2$.
 Let $\cC_{m^*}(\epsilon / 2, R)$ be the minimal $\epsilon/2$-cover of $\{  w\in \RR^{m^*} \colon \| w \|_2 \leq R \} $ with respect to the Euclidean norm. 
By definition, for any 
$f \in \cH $ with $\|f\|_{\cH}\leq R$,
there exist 
$ \tilde w \in \cC_{m^*}(\epsilon /2 , R)$ such that 
$
 \sum _{j=1}^{m^*} ( w_j - \tilde w_j)^2 \leq \epsilon^2 / 4.
$
Therefore, by \eqref{eq:approx_pi_f}  we have 
\#\label{eq:cover_rkhs_exp11}
\bigg \| f - \sum_{j=1}^{m^*} \tilde w_j \cdot \sqrt{ \sigma_j} \cdot \psi_j  \bigg \| _{\infty} \leq  \| f - \Pi_{m^*} (f) \|_{\infty } + \bigg\|\Pi_{m^*} (f) - \sum_{j=1}^{m^*} \tilde w_j \cdot \sqrt{ \sigma_j} \cdot \psi_j    \bigg\| _{\infty}\leq \epsilon,
\#
which implies that the $\epsilon$-covering number of the RKHS norm ball $\{ f \in \cH \colon \| f\|_{\cH} \leq R\}$ is bounded by the cardinality of $\cC_{m^*}(\epsilon / 2 , R)$, i.e., 
$N_{\infty}(\epsilon, \cH, R) \leq | \cC_{m^*}(\epsilon / 2 , R)\bigr |$.
As shown in  
\cite[Corollary 4.2.13]{vershynin2018high}, we have 
\#\label{eq:bound_card}
\bigl| \cC_{m^*}(\epsilon/ 2 , R)\bigr | \leq ( 1+ 4R/ \epsilon) ^{m^*}.
\#
Therefore, combining \eqref{eq:bound_dstar} and \eqref{eq:bound_card}, 
we have 
\$
\log N_{\infty}(\epsilon, \cH, R) & \leq m^* \cdot \log (1 + 4R / \epsilon) \leq C_{1,m} \cdot \bigl[\log (R / \epsilon ) + C_{2,m} \bigr ] ^{1/ \gamma} \cdot  [ \log (1+ 4R/ \epsilon) ] \notag \\
& \leq C_3 \cdot \bigl [ \log (R/ \epsilon ) + C_4]^{1 + 1/ \gamma}, 
\$
where $C_3$ and $C_4$ are absolute constants that only depend on $C_{\Psi}$, $C_1$, $C_2$, $\gamma$, and $\tau$, which are specified in Assumption \ref{assume:decay}. Thus we conclude the proof for the second case.

\vspace{5pt} 

{\noindent \bf Case (iii): $\gamma$-Polynomial Decay.} Finally, we consider the last case where the eigenvalues $\{ \sigma_j \}_{j\geq 1}$  satisfy the $\gamma$-polynomial decay condition. 
The proof is similar to that of {\textbf{Case (ii)}}. 
Specifically, under Assumption \ref{assume:decay} and by \eqref{eq:proj_rkhs} we have 
\#\label{eq:cover_rkhs_poly1}
\| f - \Pi_m (f) \|_{\infty } &  \leq    \sum _{j =m+1}^\infty  | w_j | \cdot \| \psi_j \|_{\infty}  \leq    \sum _{j =m+1}^\infty  C_{\psi} \cdot  |w_j | \cdot \sigma_j ^{1/2 - \tau } \notag \\
&  \leq \sum _{j =m+1}^\infty  C_{\psi} \cdot C_1 ^{1/2 - \tau }  \cdot | w_j | \cdot    j^{- \gamma\cdot (1/2-\tau)},
\#
for any $m \geq 1$. 
To simplify the notation, we define $C_{1, \tau} = C_{\psi} \cdot C_1 ^{1/2 - \tau}$.  
Applying the Cauchy-Schwarz inequality to \eqref{eq:cover_rkhs_poly1} and using the fact that $\sum_{j\geq 1} | w_j|^2 \leq R^2$, we have 
\#\label{eq:cover_rkhs_poly2}
\| f - \Pi_m (f) \|_{\infty } &   \leq C_{1, \tau } \cdot \biggl (  \sum_{ j= m+1}^\infty   |w_j|^2 \bigg)^{1/2 }  \cdot \biggl ( \sum_{j=m+1}^\infty j^{-\gamma\cdot (1- 2\tau)} \bigg) ^{1/2}  \\
& \leq  C_{1, \tau } \cdot R \cdot \biggl ( \int_{m}^\infty   u ^{- \gamma \cdot (1- 2\tau ) }~\ud u    \bigg )^{1/2} = \frac{C_{1, \tau} \cdot R} { \sqrt{ \gamma \cdot (1- 2\tau ) - 1  } }   \cdot m^{  - [ \gamma \cdot ( 1- 2\tau ) -1 ] /2 }.\notag
\#
We define $m^*$ as the smallest integer such that the right-hand side of 
 \eqref{eq:cover_rkhs_poly2} is bounded by $\epsilon / 2$.
 Notice that $C_{1, \tau}$, $\gamma$, and $\tau$ are absolute constants. 
 Then, there exists an absolute constant $C_{m} > 0$ such that 
 \#\label{eq:cover_rkhs_poly3}
 m^* \leq C_{m } \cdot  (R / \epsilon) ^{2 /[ \gamma \cdot (1 - 2\tau ) -1 ]} . 
 \#
Furthermore, it remains to approximate $\Pi_{m^*} (f)$ up to error $\epsilon / 2$. 
Similar to the previous case, we let $\cC_{m^*}(\epsilon /2 , R)$ be the minimal $\epsilon/2$-cover of the $\{ w \in \RR^{m^*}\colon \| w \|_2 \leq R\}$ with respect to the Euclidean norm. 
Then by definition, for any $f \in \cH$ with $ \| f \|_{\cH} \leq R $, there exists $\tilde w \in \cC_{m^*} (\epsilon /2 , R)$ such that \eqref{eq:cover_rkhs_exp11} holds,
 which implies that 
\$
\log N_{\infty}(\epsilon, \cH, R)&  \leq \log | \cC_{m^*}(\epsilon / 2 , R)\bigr | \leq m^* \cdot \log ( 1+ 4 R / \epsilon )  \notag \\
& \leq C_{3} \cdot  (R / \epsilon) ^{2 /[ \gamma \cdot (1 - 2\tau ) -1 ]} \cdot \bigl [  \log (    R / \epsilon )  + C_4 \bigr ],
\$
where $C_3$ and $C_4$ are absolute constants. Here the second inequality follows from  Corollary 4.2.13 in  \cite{vershynin2018high} and the third inequality follows from \eqref{eq:cover_rkhs_poly3}.  Therefore, we conclude the proof of this lemma.   
	\end{proof}

\subsection{Proof of Lemma \ref{lemma:cover_bonus}} \label{proof:lemma_cover_bonus}

\begin{proof}
  As shown in \S\ref{sec:rkhs}, 
  the feature mapping 
  $\phi \colon \cZ \rightarrow \cH$ satisfies 
  \#\label{eq:expand_feature}
  \phi(z) = \sum_{j=1}^\infty \sigma_{j} \cdot \psi_j (z) \cdot \psi_j = \sum_{ j = 1}^\infty \sqrt{\sigma_j} \cdot \psi_j (z) \cdot ( \sqrt{\sigma_j } \cdot \psi_j ). 
  \#
  That is, when expanding $\phi(z) \in \cH$ in the basis $\{\sqrt{\sigma_j}\cdot \psi_j \}_{j\geq 0}$ as in \eqref{eq:exxpand_f}, the $j$-th coefficient is equal to 
  $\sqrt{\sigma_j} \cdot \psi_j(z)$ for all $j \geq 1$. 
  Similar to the proof of Lemma \ref{lemma:cover_rkhs},
  in the following, we consider the three eigenvalue decay conditions separately.

  \vspace{5pt} 
  
  {\noindent \bf Case (i): $\gamma$-Finite Spectrum.} 
  When $\cH$ has only $\gamma$ nonzero eigenvalues, 
  for any $z \in \cZ$, we define a vector  
  $w_z \in \RR^{\gamma}$ by letting its $j$-th entry be $\sqrt{ \sigma_j}\cdot \psi_j(z)$ for all $j \in [\gamma]$.
  Moreover, 
 for any self-adjoint operator $\Upsilon \colon \cH \rightarrow \cH$ satisfying $\| \Upsilon \|_{\oper} \leq 1 / \lambda$,
  we define a matrix $A_{\Upsilon} \in \RR^{\gamma \times \gamma} $ as follows. 
  For any $j, k \in [\gamma]$, we define the $(j, k )$-th entry of $A_{\Upsilon} $ as 
  \$
  [ A_{\Upsilon }]_{j, k } = \bigl \la \sqrt{ \sigma_j} \cdot \psi_j , \sqrt{\sigma_k}\cdot \Upsilon \psi_k \bigr \ra _{\cH }. 
  \$
  By \eqref{eq:expand_feature} and the  definition of $A_{\Upsilon}$, 
  we have 
  \#\label{eq:cover_bonus_linear1}
  \| \phi(z) \|_{\Upsilon }^2  = \sum_{j, k = 1}^{\gamma } \sqrt{ \sigma_j }\cdot \psi_j(z) \cdot \sqrt{ \sigma_k} \cdot \psi_k (z)  \cdot [ A_{\Upsilon }]  _{j,k} = w_z ^\top A_{\Upsilon } w_z. 
  \#
With a slight abuse of notation, we define  $\cC_{\gamma }(\epsilon, \lambda )$ denote the minimal $\epsilon^2  $-cover 
  of $$ \big \{ A \in \RR^{\gamma \times \gamma  } \colon \| A \|_{\fro } \leq \sqrt{\gamma  } / \lambda  \big \}$$ with respect to the   Frobenius norm. 
  Then by definition, there exists 
  $\tilde A_{\Upsilon} \in \cC_{\gamma }(\epsilon, \lambda)$ such that 
  $\| A_{\Upsilon} - \tilde A_{\Upsilon} \|_{\fro} \leq \epsilon^2 $,
  which implies that  
 \#\label{eq:cover_bonus_linear2}
  \bigl | w_z ^\top A_\Upsilon w_z - w_z ^\top \tilde A_{\Upsilon} w_z \bigr | \leq \| w_z \|_2^2 \cdot \| A _{\Upsilon} - \tilde A_{\Upsilon } \|_{\oper} \leq  \| A _{\Upsilon} - \tilde A_{\Upsilon } \|_{\fro} \leq \epsilon ^2  , 
\# 
where we use the fact that 
$$
\| w _z \|_2 ^2 = \sum _{j = 1}^\gamma |  w_j |^2 = \sum _{j = 1}^\gamma \sigma_j \cdot |  \psi_j (z) |^2 = K(z,z ) \leq 1. 
$$
Thus, combining  \eqref{eq:cover_bonus_linear1} and \eqref{eq:cover_bonus_linear2}, and utilizing  Corollary 4.2.13 in \cite{vershynin2018high}, we have 
\$
\log N_{\infty} (\epsilon, \cF, \lambda) \leq \log \bigl | \cC_{\gamma } (\epsilon, \lambda)\bigr | \leq \gamma ^2 \cdot \log \bigl [ 1 + 8 \sqrt{ \gamma } / (\lambda \cdot \epsilon^2 )\bigr ] \leq C_5 \cdot \gamma ^2 \cdot \bigl [  \log (1 / \epsilon) + C_6\bigr ],
\$
where $C_5$ and $C_6$ are absolute constants that depend solely on $\lambda$ and $\gamma$. 
Thus, we conclude the proof for the first case.

 \vspace{5pt}
 
 {\noindent \bf Case (ii): $\gamma$-Exponential Decay.}
 In the following, we focus on the second case where the eigenvalues satisfy the $\gamma$-exponential decay condition. 
 For any $m \in \NN$, 
  we define $\Pi_{m} \colon \cH \rightarrow \cH $ as the projection operator onto the subspace spanned by $\{\psi_j  \}_{j \in [m]}$.
  Then,  by the Cauchy-Schwarz inequality 
  and Assumption \ref{assume:decay},
  for any $z \in \cZ$,  
  by \eqref{eq:expand_feature} 
  we have 
 \#\label{eq:proj_bonus0}
 \big \| \phi(z)  - \Pi_m \bigl [  \phi(z)  \bigr ]  \big \|_{\cH }
 &  =   
  \bigg \| \sum _{j =m+1}^\infty     \sqrt{\sigma_j} \cdot \psi_j (z)     \cdot \sqrt{\sigma_j } \cdot  \psi_j \bigg\|_{\cH }  = \biggl \{   \sum _{j =m+1}^\infty  \sigma_j \cdot  [ \psi_j (z) ] ^2  \bigg \} ^{1/2} \notag \\
 & \leq   \biggl (   \sum _{j =m+1}^\infty  \sigma_j \cdot   \| \psi_j \|_{\infty}^2     \bigg )^{1/2}   \leq     C_{\psi } \cdot   \bigg ( \sum_{j = m+1} ^{\infty}\sigma_j ^{1 - 2\tau }\bigg)^{1/2} , 
 \# 
 where  the second equality follows from the fact that $\{ \sqrt{\sigma_j} \cdot \psi_j\}_{j \geq 0}$ form an  orthonormal basis of $\cH$, the first inequality follows from taking a supremum over $z \in \cZ$, and the last inequality follows from the assumption that $\| \psi_j \|_{\infty} \leq C_{\psi} \cdot \sigma_j^{-\tau}$. 
Then, for any self-adjoint operator $\Upsilon \colon \cH \rightarrow \cH$ satisfying $\| \Upsilon \|_{\oper} \leq 1/ \lambda $ and any $z \in \cZ$,  by \eqref{eq:proj_bonus0} and the triangle inequality  we have 
\#\label{eq:proj_bonus1}
 & \Bigl | \| \phi (z)  \|_{\Upsilon } - \bigl \| \Pi_{m} \big [ \phi (z) \big ]  \bigr \|_{\Upsilon} \Bigr |   \leq \bigl  \| \phi (z) -  \Pi_{m} \big [ \phi (z) \big ]  \big  \|_{\Upsilon }  \leq 
  C_{\psi } / \sqrt{\lambda}  \cdot   \bigg ( \sum_{j = m+1} ^{\infty}\sigma_j ^{1 - 2\tau }\bigg) ^{1/2} .
\#
Note that the eigenvalues $\{ \sigma_j \}_{j\geq 0}$ 
admit $\gamma$-exponential decay 
under 
 Assumption \ref{assume:decay}.
 We now upper bound the right-hand side of \eqref{eq:proj_bonus1} by 
\#\label{eq:proj_bonus4}
\sup_{z \in \cZ} \Bigl | \| \phi (z)  \|_{\Upsilon } - \bigl \| \Pi_{m} \big [ \phi (z) \big ]  \bigr \|_{\Upsilon} \Bigr |  \leq   C_{\psi}  / \sqrt{\lambda}   \cdot \biggl \{  \sum_{ j=m+1} ^\infty 
C_1 ^{1 - 2\tau } \cdot \exp\bigl [     - C_2 \cdot (1 - 2\tau)  \cdot j^\gamma \big] 
\bigg \} ^{1/2}. 
\#
To simplify the notation, we define 
$C_{3, \tau} =   C_{\psi}  \cdot C_1^{1/2 -\tau}  / \sqrt{\lambda}    $ and 
$C_{4 , \tau } = C_2 \cdot (1 - 2\tau)$,
which are both absolute constants. 
Then, by \eqref{eq:proj_bonus4} and the monotonicity of $\exp(-u^\gamma)$, we further obtain 
\#\label{eq:proj_bonus5}
\sup_{z\in \cZ} \Bigl | \| \phi (z)  \|_{\Upsilon } - \bigl \| \Pi_{m} \big [ \phi (z) \big ]  \bigr \|_{\Upsilon} \Bigr |  \leq C_{3, \tau} \cdot \bigg [ \int_{m}^\infty \exp( - C_{4, \tau} \cdot u^\gamma )~\ud u \biggr ]^{1/2}.
\#
Here we can take the supremum over $\cZ$ because the right-hand side of \eqref{eq:proj_bonus4} does not depend on $z$. 
Note that we have shown in \eqref{eq:bound_eigen_final} that 
\# \label{eq:apply_integral_bound}
\int_{m}^\infty \exp( - C_{4, \tau} \cdot u^\gamma )~\ud u  \leq 
\begin{cases}
  C_{4, \tau}^{-1}  \cdot \exp( - C_{4, \tau} \cdot m^{\gamma}), & \qquad  \text{if~} \gamma \geq  1,\\
  2 \cdot (\gamma \cdot C_{4, \tau})^{-1} \cdot  \exp( - C_{4, \tau}\cdot m^{\gamma}) \cdot m^{ 1/ \gamma  - 1 }, & \qquad \text{if~} \gamma \in (0,1),
 \end{cases}
 \#
 where for the case of $\gamma \in (0,1)$,  \eqref{eq:apply_integral_bound} holds for sufficient large $m$ such that  
 $m^{\gamma} \cdot C_{4, \tau} > 2/\gamma -2$. 

 We now define $m^*$ as the smallest integer such that 
\#\label{eq:define_mstar}
\int_{m^*}^\infty \exp( - C_{4, \tau} \cdot u^\gamma )~\ud u   \leq \bigl [ \epsilon / ( 2 C_{3, \tau} )  \bigr ] ^{2}.
\#
By \eqref{eq:apply_integral_bound}, since both $C_{3, \tau}$, $C_{4, \tau}$ and $\gamma$ are absolute constants, 
there exist absolute constants $C_{3, m}$ and $C_{4,m}$ such that 
\#\label{eq:upper_bound_trunc}
m^*  \leq C_{3,m} \cdot \bigl [  \log ( 1/ \epsilon ) + C_{4,m}\big] ^{1/ \gamma}. 
\#
It is worth noting that the choice of $m^*$ in \eqref{eq:upper_bound_trunc} is uniform over all  $z \in \cZ$. 
Moreover, by \eqref{eq:proj_bonus5}, for such an $m^*$, it holds that   
\#\label{eq:proj_bonus6}
\sup_{z \in \cZ }\Bigl | \| \phi (z)  \|_{\Upsilon } - \bigl \| \Pi_{m^*} \big [ \phi (z) \big ]  \bigr \|_{\Upsilon} \Bigr | \leq \epsilon /2. 
\#

Thus, it remains to approximate $ \| \Pi_{m^*}  [ \phi (z)  ] \|_{\Upsilon}$ up to accuracy $\epsilon / 2$. 
Note that the subspace spanned by 
$\{ \psi_{j}\}_{j \in [m^*]}$ 
is $m^*$-dimensional. 
When restricted to such a subspace, $\Upsilon$ 
can be expressed using a matrix $A_{\Upsilon}  \in \RR^{ m^*\times m^*}$.   
Specifically,   for any $j ,k \in [m^*]$, we define the 
$(j, k)$-th entry of $A_{\Upsilon} $~as 
\#\label{eq:define_AUp}
[ A_{\Upsilon}]_{j,k} = \bigl \la \sqrt{ \sigma_j} \cdot \psi_j , \sqrt{\sigma_k}\cdot \Upsilon \psi_k \bigr \ra _{\cH }. 
\#
Moreover, let $w_z \in \RR^{m^*}$ be a vector whose $j$-th entry is given by 
$\sqrt{\sigma_j} \cdot \psi_j(z)$, $\forall j \in [m^*]$. 
Then, by \eqref{eq:define_AUp} it holds that 
\#\label{eq:vec_prod}
\bigl \| \Pi_{m^*}  \bigl [ \phi (z)  \bigr ] \bigr \|_{\Upsilon}^2  = \big \la  \Pi_{m^*}  \bigl [ \phi (z)  \bigr ], \Upsilon  \Pi_{m^*}  \bigl [ \phi (z)  \bigr ] \bigr \ra_{\cH}
 = w_z ^\top A_{\Upsilon} w_z. 
\#
Also, since $\| \Upsilon \|_{\oper} \leq 1/ \lambda$, the matrix operator norm of $A_{\Upsilon}$ is bounded by $1 / \lambda$; i.e.,  $\| A_{\Upsilon} \|_{\oper} \leq 1 / \lambda$. This means that the Frobenius norm of $A_{\Upsilon}$ is bounded by $\sqrt{m^*} / \lambda $. 
Let $\cC_{m^*}(\epsilon/ 2, \lambda )$ denote the minimal $\epsilon^2 /4$-cover 
of $\{ A \in \RR^{m^* \times m^* } \colon \| A \|_{\fro } \leq \sqrt{m^* } / \lambda  \}$ with respect to the   Frobenius norm. 
By definition, there exists $\tilde A_{\Upsilon} \in \cC_{m^*}(\epsilon /2 , \lambda)$ such that 
$\| A_{\Upsilon} - \tilde A_{\Upsilon} \|_{\fro} \leq \epsilon^2 / 4$. 
Hence, we have 
\#\label{eq:proj_bonus7}
\bigl | w_z ^\top A_\Upsilon w_z - w_z ^\top \tilde A_{\Upsilon} w_z \bigr | \leq \| w_z \|_2^2 \cdot \| A _{\Upsilon} - \tilde A_{\Upsilon } \|_{\oper} \leq  \| A _{\Upsilon} - \tilde A_{\Upsilon } \|_{\fro} \leq \epsilon ^2 /4.  
\#
Finally, for any $z \in \cZ$, 
we define 
\#\label{eq:proj_bonus_final}
 f_{\Upsilon} (z) = w_z^\top \tilde  A_{\Upsilon} w_z = \sum_{j,k=1}^{m^*} \sqrt{ \sigma_j \cdot \sigma_k} \cdot \psi_j(z)\cdot \psi_k(z) \cdot \bigl [\tilde A_{\Upsilon}\bigr ]_{jk} ,
\#
where $[\tilde A_{\Upsilon} ]_{jk} $ 
is the $(j,k)$-th entry of $\tilde A_{\Upsilon}$ and $m^*$ is specified in \eqref{eq:define_mstar}. 
We remark that $ f_{\Upsilon} \colon \cZ \rightarrow \RR$ is well defined 
since $m^*$ does not depend on $z$.

Finally, 
combining \eqref{eq:proj_bonus6}, \eqref{eq:vec_prod}, 
\eqref{eq:proj_bonus7}, 
and \eqref{eq:proj_bonus_final}, we obtain 
\$
 \bigl \| \| \phi(z) \|_{\Upsilon} - f_{\Upsilon} \bigr \|_{\infty} 
 & = \sup_{z \in \cZ} \bigl | \| \phi(z) \|_{\Upsilon} -  f_{\Upsilon} (z) \bigr |  \notag \\
&   \leq \sup_{z \in \cZ} \Bigl | \| \phi (z)  \|_{\Upsilon } - \bigl \| \Pi_{m^*} \big [ \phi (z) \big ]  \bigr \|_{\Upsilon} \Bigr |  + 
\sup_{z \in \cZ}  \Bigl |  \bigl \| \Pi_{m^*} \big [ \phi (z) \big ]  \bigr \|_{\Upsilon} - f_{\Upsilon } (z)  \Bigr | \notag \\
& \leq \epsilon /2 + \sup_{z \in \cZ } \Big | \sqrt{ w_z^\top A_{\Upsilon} w_z  }
 - \sqrt{w_z ^\top \tilde A_{\Upsilon } w_z } \Bigr | 
 \leq \epsilon /2 + \sup_{z \in \cZ} \sqrt{ \bigl |  w_z^\top A_{\Upsilon} w_z  - w_z ^\top \tilde A_{\Upsilon } w_z \bigr| }  \leq \epsilon.
 \$
This implies that $\{ f_{\Upsilon} \colon  \Upsilon \in \cC_{m^*} (\epsilon, \lambda)\}$ forms an $\epsilon$-cover of $\cF(\lambda)$ in \eqref{eq:define_Flambda}.
Hence, we have that 
\#\label{eq:cover_bonus_final1}
N_{\infty} (\epsilon, \cF, \lambda) \leq \bigl | \cC_{m^*} (\epsilon/ 2 , \lambda)\bigr |. 
\#
Furthermore, using 
Corollary 4.2.13 in \cite{vershynin2018high},
we have 
\#\label{eq:bound_card2}
\bigl| \cC_{m^*}(\epsilon /2 , \lambda) \bigr | \leq \bigl [  1+ 8 \sqrt{m^*} / (\lambda  \cdot \epsilon^2 ) \bigl ]  ^{{m^*}^2} .
\#
Combining \eqref{eq:upper_bound_trunc},
\eqref{eq:cover_bonus_final1}, and  \eqref{eq:bound_card2}, 
we finally have 
\$
\log N_{\infty} (\epsilon, \cF, \lambda)  & \leq {m^*}^2 \cdot \log  \bigl [  1+ 8 \sqrt{m^*} / (\lambda \cdot  \epsilon^2 ) \bigl ]  \\  
& \leq C_{3,m}^2 \cdot \bigl [ \log (1/ \epsilon) + C_{4,m} \bigr ]^{2/\gamma} \cdot \log   \Big\{   1+ 8 C_{3,m}^{1/2}  \cdot \bigl [ \log (1/ \epsilon) + C_{4,m} \bigr ]^{1/(2\gamma)}  / (\lambda \cdot  \epsilon^2 )   \Bigr \}  \\
& \leq C_5 \cdot  [ \log (1/\epsilon) + C_6]^{1+ 2/ \gamma},  
\$
where $C_5$ and $C_6$ are absolute constants that depend on $C_\psi$, $C_1$, $C_2$, $\tau$, $\gamma$,  and $\lambda$, but are independent of $T$, $H$,  and $\epsilon$. 
Here in the last inequality we use the fact that $\log (1/ \epsilon) \leq 1/ \epsilon$, which holds when $\epsilon \leq 1/ e$. 
Therefore, we conclude the proof for the second case.

\vspace{5pt}
{\noindent \bf Case (iii): $\gamma$-Polynomial Decay.} 
Finally, we consider the last case where the eigenvalues satisfy the $\gamma$-polynomial eigenvalue condition. Our proof is similar to that for the second case.
Specifically,  for any $m \geq 1$, by the assumption that $\sigma_j \leq C_1 \cdot j^{-\gamma}$ for all $j \geq 1$, we have 
\#\label{eq:cover_bonus_poly1}
& \sup_{z \in \cZ} \Bigl |    
\| \phi(z) \| _{\Upsilon } - \bigl \| \Pi_{m} \big [ \phi (z) \big ]  \bigr \|_{\Upsilon} 
\Bigr |  \notag \\
& \qquad  \leq  C_{\psi}/ \sqrt{\lambda} \cdot   \bigg (   \sum_{j = m+1} ^{\infty}\sigma_j ^{1 - 2\tau }\bigg) ^{1/2}   \leq C_{\psi} \cdot C_{1}^{1/2 - \tau } / \sqrt{\lambda}\cdot \biggl ( \sum_{j =m+1} ^\infty j^{ - \gamma \cdot ( 1 - 2\tau  )}\bigg)^{1/2} \notag \\
& \qquad  \leq C_{\psi}   \cdot C_{1}^{1/2 - \tau }/ \sqrt{\lambda}  \cdot \biggl ( \int_{m}^\infty   u ^{- \gamma \cdot (1- 2\tau ) }~\ud u     \biggr )  ^{1/2} = \frac{C_{\psi }  \cdot C_{1}^{1/2 - \tau } / \sqrt{\lambda} } { \sqrt{ \gamma \cdot (1- 2\tau ) - 1  } }   \cdot m^{ -  [ \gamma \cdot ( 1- 2\tau ) -1 ] /2 }. 
\#
Notice that $C_{\psi}$, $C_1$, $\gamma$, and $\tau$ are all absolute constants. 
Similar to the previous case, we let $m^*$ be the smallest integer $m$ such that the  
\$
\sup_{z \in \cZ} \Bigl |    
\| \phi(z) \| _{\Upsilon } - \bigl \| \Pi_{m} \big [ \phi (z) \big ]  \bigr \|_{\Upsilon} 
\Bigr | \leq \epsilon /2 .
\$
 Then by \eqref{eq:cover_bonus_poly1}, there exists an absolute constant $C_{ 3, m}  > 0$ such that 
 \# \label{eq:cover_bonus_poly3}
 m ^* \leq C_{3, m} \cdot (1/ \epsilon )^{2/ [ \gamma \cdot (1 - 2\tau )  - 1 ]} .
 \#
 Recall that we let $C_{m^*} (\epsilon /2 , \lambda)$ be the minimal $\epsilon^2 / 4$-cover of   
  $\{ A \in \RR^{m^* \times m^* } \colon \| A \|_{\fro } \leq \sqrt{m^* } / \lambda  \}$ with respect to the   Frobenius norm. 
  As shown in \eqref{eq:define_AUp} -- \eqref{eq:bound_card2}, we similarly have 
  \$
  \log N_{\infty} (\epsilon, \cF, \lambda ) & \leq \log \bigl | \cC_{m^*} (\epsilon/ 2 , \lambda)\bigr | \leq {m^*}^2 \cdot \log \bigl [ 1 + 8 \sqrt{m^*} / (\lambda \cdot \epsilon^2 )    \bigr ] \notag \\
  & \leq C_{3, m} ^2  \cdot (1/ \epsilon )^{4/ [ \gamma \cdot (1 - 2\tau )  - 1 ]} \cdot \log \bigl [    1 + 8 \sqrt{m^*} / (\lambda \cdot \epsilon^2  \bigr ] \notag \\
  & \leq C_5 \cdot (1/ \epsilon )^{4/ [ \gamma \cdot (1 - 2\tau )  - 1 ]} \cdot \bigl [ \log (1 /\epsilon ) + C_6\big ],
  \$
  where $C_5$ and $C_6$ are absolute constants that only depend on $C_1$, $C_{\psi}$, $C_1$, $\tau$, $\gamma$, and $\lambda$. Here in the last inequality we utilize \eqref{eq:cover_bonus_poly3}, which implies that 
 $$
 \sqrt{ m ^*} / (\lambda \cdot \epsilon^2 ) \leq \sqrt{C_{3, m} } / \lambda \cdot (1/ \epsilon )^{   \frac{2 \gamma \cdot (1 - 2\tau ) - 1} {\gamma \cdot (1 - 2\tau )  - 1 }}.
 $$
 Thus, we conclude the proof of the last case and therefore conclude the proof of the lemma. 
\end{proof}

\subsection{Technical Lemmas} 

Next, we present a few concentration inequalities. The first one provides 
concentration for standard self-normalized processes.

\begin{lemma} [Concentration of Self-Normalized Processes in RKHS \citep{chowdhury2017kernelized}] \label{thm:self_norm}
Let 
$\cH$ be an RKHS defined over $\cX\subseteq \RR^d$ with kernel function $ K(\cdot, \cdot) \colon \cX\times \cX \rightarrow \RR$. 
 Let $\{  x_\tau \}_{\tau=1}^{\infty} \subseteq \cX $ be a  discrete time stochastic process that is adapted to the filtration $\lbrace \cF_t \rbrace_{t=0}^{\infty}$.
 That is, $x_\tau $ is $\cF_{\tau-1}$ measurable for all $\tau \geq 1$. 
  Let $\lbrace\epsilon_t\rbrace_{\tau =1}^{\infty}$ be a real-valued stochastic process such that (i) 
  $\epsilon_\tau \in \cF_\tau$ and (ii) $\epsilon_\tau $ is   zero-mean and $\sigma$-sub-Gaussian    conditioning on $\cF_{\tau-1}$: 
  \$
  \EE [\epsilon_\tau | \cF_{\tau -1}] = 0 ,  
 \qquad \EE [e^{\lambda \epsilon_\tau  } |\cF_{\tau -1}] \le e^{\lambda^2 \sigma^2/2}, \qquad \forall \lambda \in \RR. 
  \$
  Moreover, for any $t \geq 2$, let $E _{t} = (\epsilon_1, \ldots, \epsilon_{t-1}) ^\top \in \RR^{t-1}$ and  
   $K_t \in \RR^{(t-1) \times (t-1)}$ be the Gram matrix of $\{ x_{\tau} \} _{\tau \in [t-1]} $. 
   Then, for any $\eta > 0$  and  any $ \delta \in (0,1)$, with probability at least $1- \delta$,   simultaneously for all $t\geq 1$, we have 
   \#\label{eq:self_norm_ineq}
 E_t^\top \bigl [ (K_t+ \eta \cdot  I)^{-1} + I \bigr ] ^{-1}E_t  \le   \sigma ^2 \cdot \logdet \bigl [ (1+\eta) \cdot I+K_t \big] +  2 \sigma^2 \cdot \log( 1 /   \delta )   .
   \#  
Moreover, if $K_t$ is positive definite for all $t\geq 2$  with probability one, then the inequality in \eqref{eq:self_norm_ineq} also  holds with $\eta = 0$. 
\end{lemma}

\begin{proof} 
See Theorem 1 in \cite{chowdhury2017kernelized} for a detailed proof. 
\end{proof} 

\begin{lemma} [Lemma D.4 of \cite{jin2018q}]  \label{lem:self_norm_covering}
Let $\{x_\tau\}_{\tau = 1}^\infty   $ and $\{\phi_\tau\}_{\tau = 1}^\infty$ be $\cS$-valued and $\cH$-valued stochastic processes  adapted to filtration $\{\cF_\tau\}_{\tau = 0}^\infty$, respectively,  where we assume that $\| \phi_{\tau} \|_{\cH} \leq 1$ for all $\tau \geq 1$. 
Moreover, for any $t\geq 1$, we  let $K_t \in \RR^{t\times t} $ be the Gram matrix of $\{\phi_\tau \}_{\tau \in [t]}$ and   define 
an operator $\Lambda_t \colon \cH \rightarrow \cH$ as  $ \Lambda_t =  \lambda \cdot  I_{\cH} + \sum_{\tau=1}^t \phi_\tau \phi_\tau^\top$ with $\lambda > 1$.
Let $\cV \subseteq \{ V \colon \cS\rightarrow [0, H]\}$ be a class of bounded functions on $\cS$. 
Then for  any $\delta \in (0, 1)$, with probability at least 
$1 -\delta$, 
we have simultaneously for all $t\geq 1$ that 
 \#\label{eq:unif_cov_upper}
& \sup_{V\in \cV}  \biggl \| \sum_{\tau = 1}^t \phi_\tau  \bigl \{ V(x_\tau) - \EE [V(x_\tau)|\cF_{\tau-1}] \bigr \} \biggr \| ^2_{\Lambda_t^{-1}}  \\
 & \qquad \leq 2H^2 \cdot  \logdet  ( I+ K_t / \lambda)  + 2H^2 t (\lambda -1)  + 4 H^2  \log ( \cN_{\epsilon} / \delta) + 8  t^2  \epsilon^2 / \lambda    \notag,
\#
where $\mathcal{N}_{\epsilon}$ is the $\epsilon$-covering  number of $\mathcal{V}$ with respect to the distance $\mathrm{dist}(\cdot, \cdot ) $.
\end{lemma}

\begin{proof}
	Let $\cV_{\epsilon} \subseteq \{ V\colon \cS\rightarrow [0, H] \}$ be the minimal $\epsilon$-cover of $\cV$ such that $N_{\epsilon} = | \cV_{\epsilon}| $. 
Then for any 
 $V \in \mathcal{V}$,  there  exists a value function $V' \colon \cS \rightarrow \RR$ in  $\cN_{\epsilon}$ such that $\mathrm{dist}(V,  V') \leq \epsilon$. Let $\Delta_V = V - V'$.  By the inequality $(a+b)^2 \leq 2a^2 + 2b^2$, we have 
\#\label{eq:unif1}
&\bigg\| \sum_{\tau = 1}^t \phi_\tau \bigl \{ V(x_\tau) - \EE[V(x_\tau)|\cF_{\tau-1}]\bigr \}  \bigg \|^2_{\Lambda_t^{-1}}\\
&\qquad \le  2  \cdot   \bigg\| \sum_{\tau = 1}^t \phi_\tau \big\{ V' (x_\tau) - \EE[V' (x_\tau)|\cF_{\tau-1}] \bigr \} \bigg\|^2_{\Lambda_t^{-1}}
+ 2 \cdot \bigg\| \sum_{\tau = 1}^t \phi_\tau \bigl \{ \Delta_V(x_\tau) - \EE[\Delta_V(x_\tau)|\cF_{\tau-1}]\bigr \}  \bigg\|^2_{\Lambda_t^{-1}}.  \notag 
\#
To bound the first term on the right-hand side of \eqref{eq:unif1}, 
we apply Lemma  \ref{thm:self_norm} to $V'$ and take a union bound over  $V' \in \cV_{\epsilon}$.  
While for the second term, since $\sup_{x\in \cS} |  \Delta_V (x) | \leq \epsilon$, we have 
\#\label{eq:unif2}
\bigg\| \sum_{\tau = 1}^t \phi_\tau \bigl \{ \Delta_V(x_\tau) - \EE[\Delta_V(x_\tau)|\cF_{\tau-1}]\bigr \}  \bigg\|^2_{\Lambda_t^{-1}} \leq t^2 \cdot (2\epsilon)^2 / \lambda = 4 t^2  \epsilon^2   / \lambda.
\#
Thus, combining \eqref{eq:unif1} and \eqref{eq:unif2}, we have 
\#\label{eq:unif3}
& \sup_{V \in \cV} \bigg\| \sum_{\tau = 1}^t \phi_\tau \bigl \{ V(x_\tau) - \EE[V(x_\tau)|\cF_{\tau-1}]\bigr \}  \bigg \|^2_{\Lambda_t^{-1}} \notag \\
& \qquad \leq \sup_{V' \in \cV_{\epsilon} }  2  \cdot  \bigg\| \sum_{\tau = 1}^t \phi_\tau \big\{ V' (x_\tau) - \EE[V' (x_\tau)|\cF_{\tau-1}] \bigr \} \bigg\|^2_{\Lambda_t^{-1}} + 8  t^2  \epsilon^2 / \lambda. 
\#

Now we fix $V' \in \cV_{\epsilon}$ and define  $  \varepsilon_t  \in \RR^t$ by letting $[ \varepsilon_t ]_{\tau} = V' (x_\tau) - \EE[V' (x_\tau)|\cF_{\tau-1}] $ for any $\tau \geq 1$. We define an operator $\Phi \colon \cH\rightarrow \RR^t$ as  $\Phi = \big [ \phi_1^\top , \ldots, \phi_t ^\top   \bigr ] ^\top$ and let $K_t = \Phi_t \Phi_t^\top \in \RR^{t\times t}$. 
Using this notation, we have $\Lambda_t = \lambda \cdot I_{\cH} + \Phi_t^\top \Phi_t$ and 
\#\label{eq:unif4}
& \bigg\| \sum_{\tau = 1}^t \phi_\tau \big\{ V' (x_\tau) - \EE[V' (x_\tau)|\cF_{\tau-1}] \bigr \} \bigg\|^2_{\Lambda_t^{-1}}     =  \| \Phi_t ^\top \varepsilon_t \| _{  \Lambda_t^{-1} } ^2      =   \varepsilon_t ^\top \Phi_t  \Lambda_t^{-1} \Phi_t ^\top \varepsilon_t  \notag \\
& \qquad =  \varepsilon_t ^\top\Phi_t  \Phi_t^\top ( K_t + \lambda \cdot I)^{-1} \varepsilon_t = \varepsilon_t ^\top K_t  ( K_t + \lambda \cdot I)^{-1} \varepsilon_t ,
\#
where the third inequality follows from  \eqref{eq:dim_change}. 
Setting $\lambda =  1+ \eta$ for some $\eta > 0$, 
we have 
\$
(K_t+\eta \cdot  I) \big [ K_t+(1+\eta) \cdot I \bigr ] ^{-1}  = (K_t+\eta \cdot  I)  \bigl [ I + (K_t+\eta \cdot  I)  \bigr ]^{-1} = \bigl [ (K_t+\eta \cdot I)^{-1}+I \bigr ] ^{-1}  ,
\$
which implies that 
\#\label{eq:unif5}
\varepsilon_t ^\top K_t  ( K_t + \lambda \cdot I)^{-1} \varepsilon_t &  \le  \varepsilon_t^\top (K_t+\eta \cdot  I)\bigl [ I + (K_t+\eta \cdot  I)  \bigr ]^{-1}  \varepsilon_t \notag \\
& = \varepsilon_t ^\top  \bigl [ (K_t+\eta \cdot I)^{-1}+I \bigr ] ^{-1}  \varepsilon_t. 
\#
Notice that each entry of $\varepsilon_t$ is bounded by $H$ in absolute value since $V'$ is bounded in $[0, H]$. 
By combining \eqref{eq:unif3}, \eqref{eq:unif4}, \eqref{eq:unif5},    Lemma  \ref{thm:self_norm},  and taking a union bound over $\cV_{\epsilon}$,  for any $\delta \in (0,1)$,  we obtain that, with probability at least $1-\delta$,
\#\label{eq:unif6}
 & \sup_{V' \in \cV_{\epsilon} } \bigg\| \sum_{\tau = 1}^t \phi_\tau \big\{ V' (x_\tau) - \EE[V' (x_\tau)|\cF_{\tau-1}] \bigr \} \bigg\|^2_{\Lambda_t^{-1}}  \notag \\
 & \qquad \leq   H^2 \cdot \logdet [ (1+\eta) \cdot I+K_t] +  2H^2 \cdot \log (   \cN_{\epsilon}  / \delta)      
 \#
 holds simultaneously 
 for all $t\geq 1$. 
 Moreover, notice that $(1+\eta) \cdot I+K_t = [ I+(1+\eta)^{-1} \cdot K_t  ] \cdot [ (1+\eta) \cdot I ]$, which implies that 
\#\label{eq:unif7}
\logdet \bigl [  (1+\eta) \cdot I+K_t \bigr ] & =\logdet \bigl [ I+(1+\eta)^{-1} \cdot K_t \bigr ] +t\ln(1+\eta)  \notag \\
& \le  \logdet \bigl [ I+(1+\eta)^{-1} \cdot K_t \bigr ]  +   \eta t. 
\#
Finally, combining \eqref{eq:unif3},  \eqref{eq:unif6}, and \eqref{eq:unif7}, we conclude that, simultaneously for all $t\geq 1$, 
 \eqref{eq:unif_cov_upper} holds with probability at least $1 -\delta$, which concludes the proof. 
\end{proof}

\begin{lemma} [\cite{abbasi2011improved}]\label{lemma:telescope}
	Let $\{\phi_t \}_{t\geq 1}$ be a   sequence in the RKHS $\cH$. Let $\Lambda_0 \colon \cH \rightarrow \cH $ 
	be defined as  $\lambda \cdot \cI_{\cH}$ where $\lambda \geq 1$ and $\cI_{\cH}$ is the identity mapping on $\cH$.  
	For any $t\geq 1$, we define a  self-adjoint  and positive-definite operator $\Lambda_t $ by letting  $
	\Lambda_t = \Lambda_0 + \sum_{j = 1}^t \phi_ j   \phi_j ^\top $.  
Then, for any $t\geq 1$, we have 
\$
	\sum_{j = 1}^t  \min \bigl \{ 1,  \phi_j^\top \Lambda_{j-1} ^{-1} \phi_j \bigr  \} \leq    2 \logdet( I + K_t / \lambda ), 
\$
	where $K_t \in \RR^{t\times t}$ is the Gram matrix obtained from $\{ \phi_j \}_{j\in [t]}$, i.e.,  for any $j, j' \in [t]$, the $(j, j')$-th entry of $K_t$ is $\la \phi_j, \phi_{j'} \ra_{\cH}$. 
	Moreover, if we further have   $\sup_{t\geq 0} \{ \| \phi_t \|_{\cH}  \}  \leq 1 $, then  it holds that 
\$
	\logdet( I + K_t / \lambda )   \leq \sum_{j=1}^{t} \phi_j^\top \Lambda_{j-1} ^{-1} \phi_j    \leq 2 \logdet( I + K_t / \lambda ) .
\$
	
\end{lemma}
\begin{proof}
Note that we have $\log (1+x) \le x \le 2 \log (1+x)$ for all $x \in [0,1]$.
Since $\Lambda_t^{-1}$ is a self-adjoint and  positive-definite operator, this  implies that 
\#\label{eq:proof_for_epl1}
 \sum_{j=1}^t \min \bigl \{1 ,  \phi_j^\top \Lambda_{j-1} ^{-1} \phi_j  \bigr \} \leq \sum_{j=1}^t 2 \log \bigl ( \min \bigl \{2 ,  1+  \phi_j^\top \Lambda_{j-1} ^{-1} \phi_j  \bigr \} \bigr ) \leq 2 \sum_{j=1}^t \log \bigl ( 1 + \phi_j^\top \Lambda_{j-1} ^{-1} \phi_j  \bigr ).
\# 
Moreover, when   additionally  it is the case that 
   $\sup_{j\geq 1}  \| \phi_j  \| _{\cH}\leq 1$ for all $j\ge 0$, 
 we  have 
\#\label{eq:final011}
\phi_j ^\top  \Lambda_{j-1}^{-1}  \phi_j  = \big \la \phi_j , \Lambda_{j-1} ^{-1} \phi_j \ra_{\cH} \leq \| \phi_j \|_{\cH } \cdot \bigl \| \Lambda_{j-1} ^{-1} \phi_j \|_{\cH}  \leq [\lambda_{\min} (\Lambda_0) ]^{-1}  \cdot  \| \phi_j  \| _{\cH} ^2   \leq 1 . 
\#
Hence, applying the basic inequality $\log (1+x) \le x \le 2 \log (1+x)$  to \eqref{eq:final011}, we have 
	\#\label{eq:final012}
	\sum_{j=1}^{t} \log  \bigl (  1 + \phi_j ^\top  \Lambda_{j-1}^{-1}   \phi_j  \bigr ) \le \sum_{j=1}^{t} \phi_j ^\top  \Lambda_{j-1} ^{-1}  \phi_j    \leq  2\sum_{j=1}^{t}  \log  \bigl (  1 + \phi_j ^\top  \Lambda_{j-1}^{-1}   \phi_j  \bigr ). 
	\# 
	For any $j \geq 1$, let $\Lambda_{j-1}^{1/2} \colon \cH \rightarrow \cH$ be the self-adjoint and positive-definite operator that is the square-root operator of $\Lambda_{j-1}$. 
	Specifically, let $\{ \sigma_\ell\}_{\ell \geq 1}$  be the eigenvalues of $\Lambda_{j-1}$ and let $\{ v_{\ell}\} _{\ell \geq 1}$ be the corresponding eigenfunctions. 
	Then $\Lambda_{j-1}^{1/2} = \sum_{\ell\geq 1} \sigma_{\ell}^{1/2} \cdot v_{\ell } v_{\ell}^\top$. 
	Using this notation,  for any $j \geq 1 $, by the definition of $\Lambda_j $, we have 
	$$
	 \Lambda_j  =  \Lambda_{j -1}   +  \phi_{j}  \phi_j ^\top    =   \Lambda_{j -1}^{1/2}  \bigl  ( \cI_{\cH} +  \Lambda_{j -1} ^{-1/2}  \phi_{j }  \phi_j ^\top \Lambda_{j-1}^{-1/2} \bigr  ) \Lambda_{j -1}^{1/2}  ,
	$$
	which implies that 
	\#
	\label{eq:final_013}
	\logdet ( \Lambda_j) & = \logdet ( \Lambda_{j-1}) + \logdet  \bigl  ( \cI_{\cH} +  \Lambda_{j -1} ^{-1/2}  \phi_{j }  \phi_j ^\top \Lambda_{j-1}^{-1/2} \bigr  ) \notag \\
	& = \logdet \big  ( \Lambda_{j-1})  + \logdet\bigl (1 + \phi_j ^\top \Lambda_{j-1}^{-1} \phi_j \bigr )
	\#
	Moreover, by direct computation, for any $t \geq 1$, we have 
	\#\label{eq:final_014}
	\det (\Lambda_t \Lambda _0^{-1} ) =  \det  ( I + K_t  / \lambda ). 
	\# 
	Hence, combining  \eqref{eq:final_013}, and \eqref{eq:final_014}, we obtain  that 
	\#\label{eq:final015}
	\sum_{j=1}^{t}  \log  \bigl (  1 + \phi_j ^\top  \Lambda_{j-1}^{-1}   \phi_j  \bigr ) = \logdet (\Lambda_t \Lambda_0^{-1} ) = \logdet (I + K_t /\lambda ).
	\#
	Finally, combining \eqref{eq:proof_for_epl1}, 
	\eqref{eq:final012} and \eqref{eq:final015}, we conclude the proof of this lemma. 
\end{proof}